\documentclass[11pt]{article} 
\def\shownotes{1}

\usepackage[backend=biber,
            style=numeric-comp,
            bibstyle=authoryear,
            natbib=true,
            maxbibnames=99,
            minalphanames=3,
            uniquename=false,
            maxcitenames=2,
            sorting=ynt,
            giveninits=true,                        
            labeldateparts=true,
            dashed=false,
            sortcites=true]{biblatex}
\makeatletter
\input{numeric.bbx}
\makeatother

\newcommand{\citeay}[1]{\citeauthor{#1}~(\citeyear{#1})}

\usepackage[left=1.in, top=1in, bottom=1in, right=1.in]{geometry}
\usepackage{macro}
\usepackage[us,12hr]{datetime} %
\usepackage{graphicx, subcaption}
\usepackage{multirow, makecell}
\usepackage{algorithm}
\usepackage{algorithmic}
\usepackage[T1]{fontenc}
\usepackage{microtype}
\usepackage{sidecap}

\usepackage{amsmath,amsfonts,amssymb}
\usepackage{graphicx}
\usepackage{booktabs} %
\usepackage{enumitem}
\usepackage{hyperref}
\usepackage{cleveref}
\usepackage{wrapfig}
\usepackage{breakcites}

\usepackage{relsize}
\usepackage{authblk}
\usepackage[symbol]{footmisc}

\makeatletter
\def\@fnsymbol#1{\ensuremath{\ifcase#1\or *\or \dagger\or \ddagger\or
   \mathsection\or \mathparagraph\or \|\or **\or \dagger\dagger
   \or \ddagger\ddagger \else\@ctrerr\fi}}
\makeatother

\begin{document}

\title{Generalization in Graph Neural Networks: Improved PAC-Bayesian Bounds on Graph Diffusion}

\author[$*\dag$]{Haotian Ju}
\author[$*\dag$]{Dongyue Li}
\author[$*\ddag$]{Aneesh Sharma}
\author[$\dag$]{Hongyang R. Zhang\footnote{The paper is presented at Artificial Intelligence and Statistics (AISTATS), 2023. Email correspondence: \texttt{\{ju.h, li.dongyu, ho.zhang\}@northeastern.edu} and \texttt{aneesh@google.com}.}}
\affil[$\dag$]{\itshape Northeastern University, Boston}
\affil[$\ddag$]{\itshape Google, Mountain View}

\date{}
\maketitle

\begin{abstract}
Graph neural networks are widely used tools for graph prediction tasks. Motivated by their empirical performance, prior works have developed generalization bounds for graph neural networks, which scale with graph structures in terms of the maximum degree. In this paper, we present generalization bounds that instead scale with the largest singular value of the graph neural network's feature diffusion matrix. These bounds are numerically much smaller than prior bounds for real-world graphs. We also construct a lower bound of the generalization gap that matches our upper bound asymptotically. To achieve these results, we analyze a unified model that includes prior works' settings (i.e., convolutional and message-passing networks) and new settings (i.e., graph isomorphism networks). Our key idea is to measure the stability of graph neural networks against noise perturbations using Hessians. Empirically, we find that Hessian-based measurements correlate with the observed generalization gaps of graph neural networks accurately.  Optimizing noise stability properties for fine-tuning pretrained graph neural networks also improves test performance on several graph-level classification tasks.
\end{abstract}

\section{Introduction}

A central measure of success for a machine learning model is the ability to generalize well from training data to test data.
For linear and shallow models, the generalization gap between their training performance and test performance can be quantified via complexity notions such as the Vapnik–Chervonenkis dimension and Rademacher complexity.
However, formally explaining the empirical generalization performance of deep models remains a challenging problem and an active research area (see, e.g., recent textbook by \citet{hardt2021patterns}).
There are by now many studies for fully-connected and convolutional neural networks that provide an explanation for their superior empirical performance \cite{bartlett2017spectrally,neyshabur2018towards}.
Our work seeks to formally understand generalization in graph neural networks (GNN) \cite{scarselli2008graph}, which are commonly used for learning on graphs \cite{hamilton2017representation}.

As a concrete example for motivating the study of generalization performance, we consider the fine-tuning of pretrained graph neural networks \cite{hu2019strategies}.
Given a pretrained GNN learned on a diverse range of graphs, fine-tuning the pretrained GNN on a specific prediction task is a common approach for transfer learning on graphs.
An empirical problem with fine-tuning is that, on one hand, pretrained GNNs use lots of parameters to ensure representational power.
On the other hand, fine-tuning a large GNN would overfit the training data and suffer poor test performance without proper algorithmic intervention.
Thus, a better understanding of generalization in graph neural networks can help us identify the cause of overfitting and, consequently, inspire designing robust fine-tuning methods for graph neural networks.

A naive application of the generalization bounds from fully-connected feedforward networks \cite{bartlett2017spectrally,neyshabur2017pac} to GNNs would imply an extra term in the generalization bound that scales with $n^{l-1}$, where $n$ is the number of nodes in the graph, hence rendering the error bounds vacuous.
Besides, \citet{scarselli2018vapnik} shows that the VC dimension of GNN scales with $n$. Thus, although the VC dimension is a classical notion for deriving learning bounds, it is oblivious to the graph structure.
Recent works have taken a step towards addressing this issue with better error analysis. %
\citeay{verma2019stability} find that one-layer graph neural networks satisfy uniform stability properties \cite{verma2019stability}, following the work of \citet{hardt2016train}.
The generalization bound of \citet{verma2019stability} scales with the largest singular value of the graph diffusion matrix of the model.
However, their analysis only applies to a single layer and node prediction.
\citeay{garg2020generalization} analyze an $l$ layer message-passing neural network --- with $l-1$ graph diffusion layers and $1$ pooling layer --- for graph prediction tasks \cite{garg2020generalization}.
Their result scales with $d^{l-1}$, where $d$ is the maximum degree of the graph.
Subsequently, \citeay{liao2020pac} develop a tighter bound but still scales with $d^{l-1}$ \cite{liao2020pac}.
For both results, the graph's maximum degree is used to quantify the complexity of node aggregation in each diffusion step.

\begin{figure*}[t!]
    \centering
     \begin{subfigure}[b]{0.49\textwidth}
 		\centering
 		\includegraphics[width=0.9\textwidth]{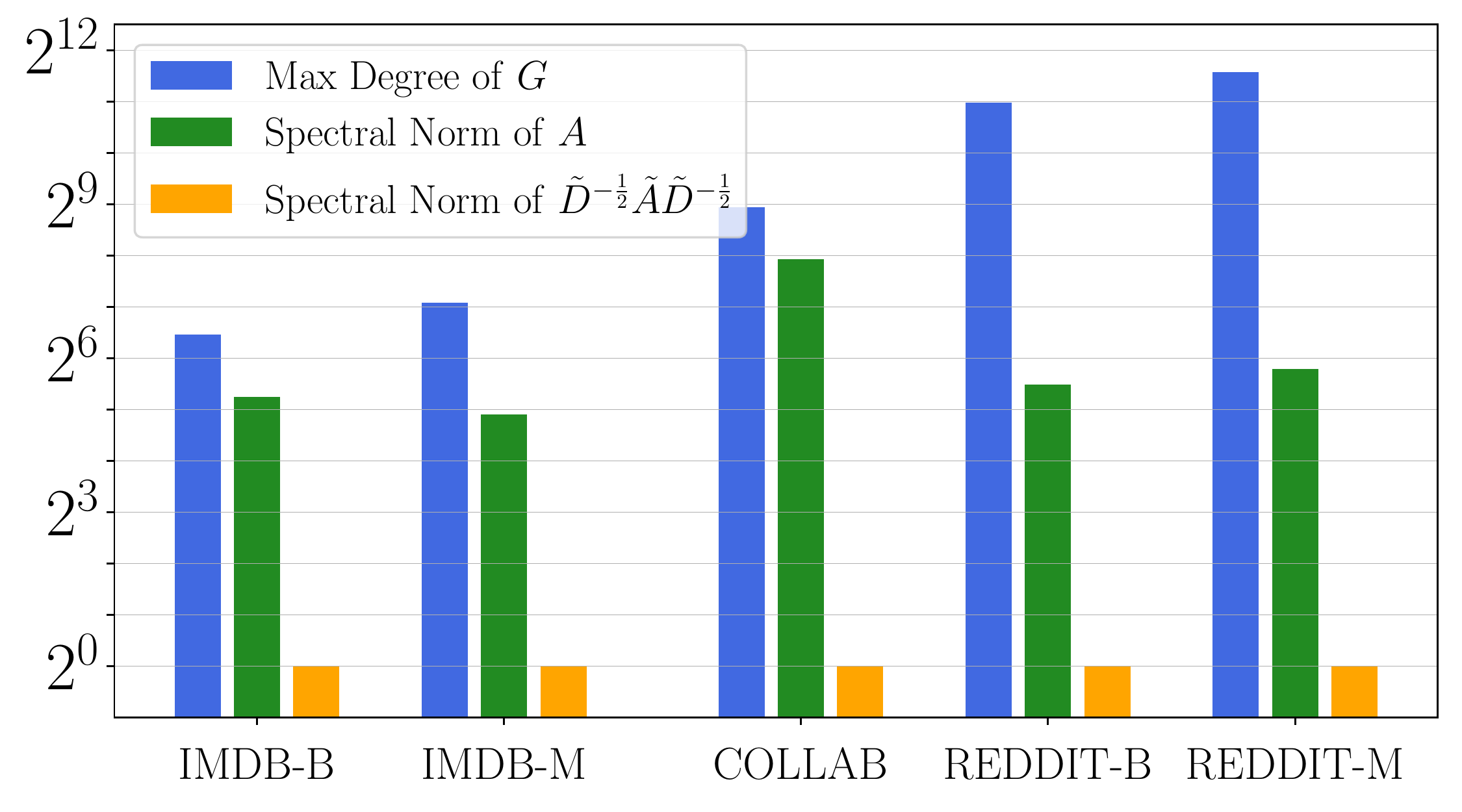}
        \caption{Spectral norm vs. max degree for five graphs}\label{fig_intro_gcn}
 	\end{subfigure}\hfill
    \begin{subfigure}[b]{0.49\textwidth}
        \centering
 		\includegraphics[width=0.9\textwidth]{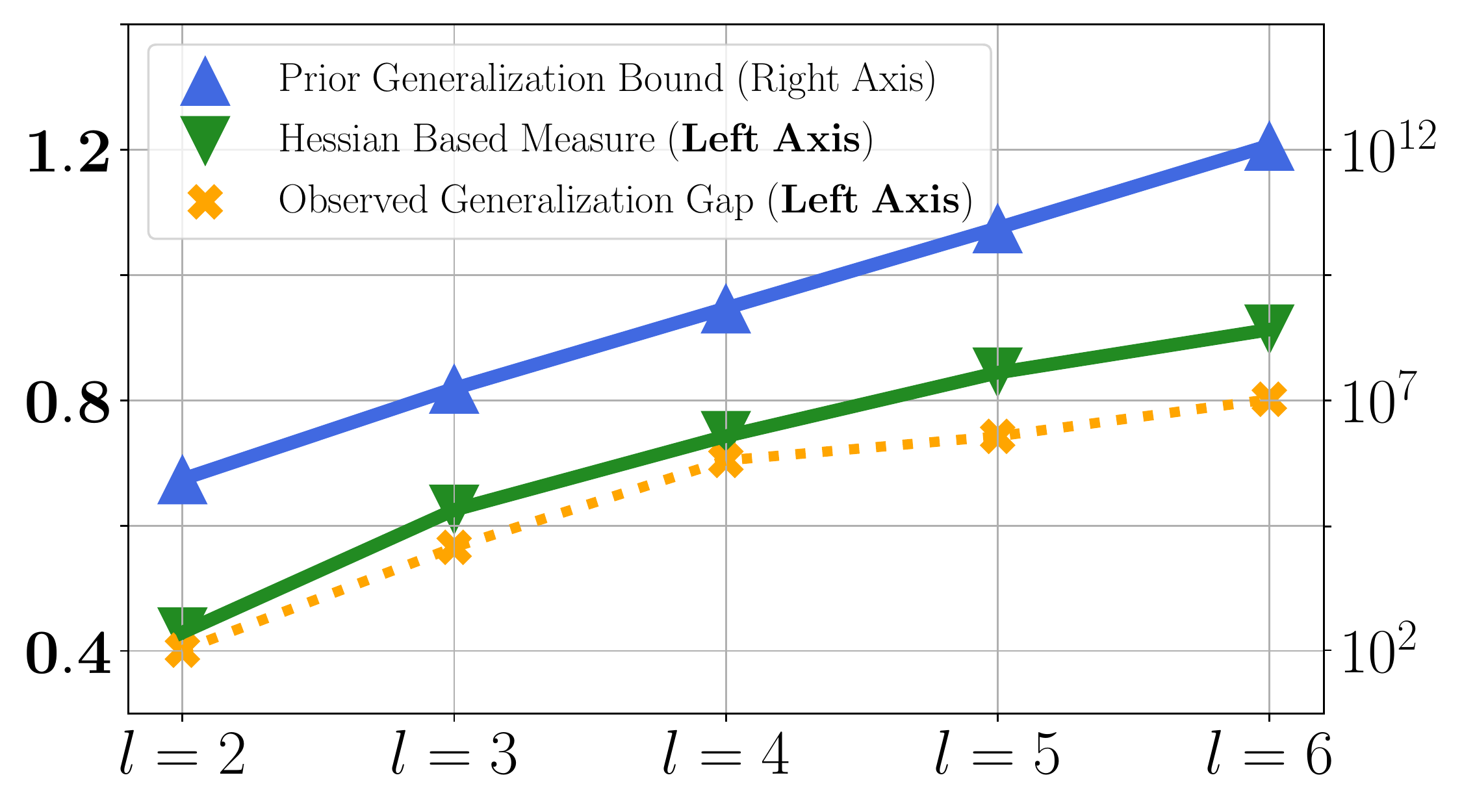}
        \caption{Comparing generalization bounds}\label{fig_intro_graph}
    \end{subfigure}
     \caption{The spectral norm bounds for graph diffusion matrices are orders of magnitude smaller than maximum degree bounds on real-world graphs; In Figure \ref{fig_intro_gcn}, we measure the spectral norm and the max degree for five graph datasets.
     In Figure \ref{fig_intro_graph}, the Hessian-based generalization measure (plotted in green, scaled according to the \emph{left axis}) matches the empirically observed generalization gaps of graph neural networks (plotted in yellow, scaled according to the \emph{left axis}). The blue line shows a uniform-convergence bound (scaled according to the \emph{right axis}) that is orders of magnitude larger than the observed gaps.}
     \label{fig_intro}
\end{figure*}

Our main contribution is to show generalization bounds for graph neural networks by reducing the max degree to the spectral norm of the graph diffusion matrix.
We achieve this by analyzing the stability of a graph neural network against noise injections.
To illustrate, denote an $l$-layer GNN by $f$.
By PAC-Bayesian analysis \cite{mcallester2013pac}, the generalization gap of $f$ will be small if $f$ remains stable against noise injections; otherwise, the generalization gap of $f$ will be large.
Empirically, quantifying the noise stability of $f$ via Lipschitz-continuity properties of its activation functions leads to nonvacuous bounds for feedforward networks that correlate with their observed generalization gaps \cite{arora2018stronger,ju2022robust,dziugaite2017computing,jiang2019fantastic}. %
Our theoretical analysis rigorously formalizes this with refined stability analysis of graph neural networks through Hessians, leading to tight generalization bounds on the graph diffusion matrix.

\medskip
\noindent\textbf{Our Contribution.}
The goal of this work is to improve the theoretical understanding of generalization in graph neural networks, and in that vein, we highlight two results below:

\begin{itemize}[leftmargin=15.0pt,topsep=2.5pt,itemsep=2.5pt] %
    \item First, we prove sharp generalization bounds for message-passing neural networks \cite{dai2016discriminative,gilmer2017neural,jin2018junction}, graph convolutional networks \cite{kipf2016semi}, and graph isomorphism networks \cite{xu2018powerful}.
    Our bounds scale with the spectral norm of $P_{_G}^{l-1}$ for an $l$-layer network, where $P_{_G}$ denotes a diffusion matrix on a graph $G$ and varies between different models (see Theorem \ref{thm_mpgnn} for the full statement).
    We then show a matching lower bound instance where the generalization gap scales with the spectral norm of $P_{_G}^{l-1}$ (see Theorem \ref{prop_lb}).

    \item Second, our stability analysis of graph neural networks provides a practical tool for measuring generalization. 
    Namely, we show that the trace of the loss Hessian matrix can measure the noise stability of GNN.
    The formal statement is given in Lemma \ref{lemma_trace_hess}, and our techniques, which include a uniform convergence of the Hessian matrix, may be of independent interest.
    We note that the proof applies to twice-differentiable and Lipschitz-continuous activations (e.g., tanh and sigmoid).
\end{itemize}
Taken together, these two results provide a sharp understanding of generalization in terms of the graph diffusion matrix for graph neural networks.
We note that the numerical value of our bounds in their dependence on the graph is much smaller than prior results \cite{garg2020generalization,liao2020pac}, as is clear from Figure \ref{fig_intro_gcn}.
Moreover, the same trend holds even after taking weight norms into account (see Figure \ref{fig_bound_measurement}, Section \ref{sec_compare}).
Further, the Hessian-based bounds (see Lemma \ref{lemma_trace_hess}, Section \ref{sec_proff_sketch}) are non-vacuous, matching the scale of empirically observed generalization gaps in Figure \ref{fig_intro_graph}.

Finally, motivated by the above analysis, we also present an algorithm that performs gradient updates on perturbed weight matrices of a graph neural network.
The key insight is that minimizing the average loss of multiple perturbed models  with independent noise injections is equivalent to regularizing $f$'s Hessian in expectation. We conduct experiments on several graph classification tasks with Molecular graphs that show the benefit of this algorithm in the fine-tuning setting.

\section{Related Work}\label{sec_related}

\textbf{Generalization Bounds:}
An article by \citeay{zhang2016understanding} finds that deep nets have enough parameters to memorize real images with random labels, yet they still generalize well if trained with true labels.
This article highlights the overparametrized nature of modern deep nets (see also a recent article by \citet{arora2021technical}), motivating the need for complexity measures beyond classical notions. 
In the case of two-layer ReLU networks, \citeay{neyshabur2018towards} show that (path) norm bounds better capture the ``effective number of parameters'' than VC dimension---which is the number of parameters for piecewise linear activations \cite{bartlett2019nearly}.

For multilayer networks, subsequent works have developed norm, and margin bounds, either via Rademacher complexities \cite{bartlett2017spectrally,golowich2018size,long2020generalization}, or PAC-Bayesian bounds \cite{neyshabur2017pac,arora2018stronger,li2021improved,ju2022robust}.
All of these bounds apply to the fine-tuning setting following the distance from the initialization perspective.
Our analysis approach builds on the work of \citeay{arora2018stronger} and \citeay{ju2022robust}. The latter work connects perturbed losses and Hessians for feedforward neural networks, with one limitation Hessians do not show any explicit dependence on the data.
This is a critical issue for GNN as we need to incorporate the graph structure in the generalization bound.
Our result instead shows an explicit dependence on the graph and applies to message-passing layers that involve additional nonlinear mappings.
We will compare our analysis approach and prior analysis in more detail when we present the proofs in Section \ref{sec_proff_sketch} (see Remark \ref{remark_tech}).

\medskip
\noindent\textbf{Graph Representation Learning:} Most contemporary studies of learning on graphs consider either node-level or graph-level prediction tasks.
Our result applies to graph prediction while permitting an extension to node prediction: see Remark \ref{remark_node} in Section \ref{sec_proff_sketch}.
Most graph neural networks follow an information diffusion mechanism on graphs \cite{scarselli2008graph}.
Early work takes inspiration from ConvNets and designs local convolution on graphs, e.g., spectral networks \cite{bruna2013spectral}, GCN \cite{kipf2016semi}, and GraphSAGE \cite{hamilton2017inductive} (among others).
Subsequent works have designed new architectures with graph attention \cite{velivckovic2017graph} and isomorphism testing \cite{xu2018powerful}.
\citeay{gilmer2017neural} synthesize several models into a framework called message-passing neural networks.
Besides, one could also leverage graph structure in the pooling layer (e.g., differentiable pooling and hierarchical pooling \cite{zhang2018end,ying2018hierarchical}). It is conceivable that one can incorporate the model complexity of these approaches into our analysis.
Recent work applies pretraining to large-scale graph datasets for learning graph representations \cite{hu2019strategies}.
Despite being an effective transfer learning approach, few works have examined the generalization of graph neural nets in the fine-tuning step.

Besides learning on graphs, GNNs are also used for combinatorial optimization \cite{selsam2018learning} and causal reasoning \cite{xu2019can}.
There is another approach for graph prediction using graph kernels \cite{vishwanathan2010graph} such as the graph neural tangent kernel \cite{du2019graph}.
Lastly, we remark that graph diffusion processes have been studied in earlier literature on social and information networks \cite{goel2014connectivity,goel2015note,zhang2019pruning}.
For further references about different applications of GNN, see review articles \cite{hamilton2017representation,errica2019fair,wu2020comprehensive,chami2020machine}.

\medskip
\noindent\textbf{Generalization in Graph Neural Networks:}
Recent work explores generalization by formalizing the role of the algorithm and the alignment between networks and tasks \cite{xu2020neural}.
\citet{esser2021learning} finds that transductive Rademacher complexity-based bound provides insights into the behavior of GNNs in the stochastic block model.
Besides, there are works about size generalization, which refer to performance degradation when models extrapolate to graphs of different sizes from the input \cite{selsam2018learning,yehudai2021local}.
It is conceivable that the new tools we have developed may be useful for studying extrapolation. %

\medskip
\noindent\textbf{Expressivity of Graph Neural Networks:}
The expressivity of GNN for graph classification can be related to graph isomorphism tests and has connections to one-dimensional Weisfeiler-Lehman testing of graph isomorphism \cite{morris2019weisfeiler,xu2018powerful}.
This implies limitations of GNN for expressing tasks such as counting cycles \cite{sato2019approximation,chen2020can,azizian2020expressive}.
The expressiveness view seems orthogonal to generalization, which instead concerns the sample efficiency of learning.
For further discussions and references, see a recent survey by \citet{jegelka2022theory}. %

\section{Sharp Generalization Bounds for Graph Neural Networks}\label{sec_theory}

We first introduce the problem setup for analyzing graph neural networks.
Then, we state our generalization bounds for graph neural networks and compare them with the prior art.
Lastly, we construct an example to argue that our bounds are tight.

\subsection{Problem setup}

Consider a graph-level prediction task.
Suppose we have $N$ examples in the training set; each example is an independent sample from a distribution denoted as $\cD$, which is jointly  supported on the feature space $\cX$ times the label space $\cY$.
In each example, we have an undirected graph denoted as $G = (V, E)$, which describes the connection between $n$ entities, represented by nodes in $V$.
For example, a node could represent a molecule, and an edge between two nodes is a bond between two molecules.
Each node also has a list of $d$ features. Denote all node features as an $n$ by $d$ matrix $X$.
For graph-level prediction tasks, the goal is to predict a graph label $y$ for every example. 
We will describe a few examples of such tasks later in Section \ref{sec_empirical}.

\medskip
\noindent\textbf{Message-passing neural networks (MPNN).} %
We study a model based on several prior works for graph-level prediction tasks \cite{dai2016discriminative,gilmer2017neural,garg2020generalization,liao2020pac}.
Let $l$ be the number of layers: the first $l-1$ layers are diffusion layers, and the last layer is a pooling layer.
Let $d_t$ denote the width of each layer for $t$ from $1$ up to $l$.
There are several nonlinear mappings in layer $t$, denoted as $\phi_{t}, \rho_{t}$, and $\psi_{t}$; further, they are all centered at zero.
There is a weight matrix $W^{(t)}$ of dimension $d_{t-1}$ by $d_{t}$ for transforming neighboring features, and another weight matrix $U^{(t)}$ of dimension $d$ by $d_t$ for transforming the anchor node feature.

For the first $l-1$ layers, we recursively compute the node embedding from the input features $H^{(0)} = X$:
{\begin{align}
    H^{(t)} = \phi_t\Bigbrace{X U^{(t)} + \rho_t\bigbrace{P_{_G} \psi_t(H^{(t-1)})} W^{(t)}}, \text{ for } t = 1,2,\dots,l. \label{eq_matrix_mpgnn}
\end{align}}%
For the last layer $l$, we aggregate the embedding of all nodes: let $\bm{1}_n$ be a vector with $n$ values of one:
{\begin{align}
    H^{(l)} &= \frac 1 n \bm{1}_n^{\top} H^{(l-1)} W^{(l)}. \label{eq_mpgnn_readout}
\end{align}}%
Note that this setting subsumes many existing GNNs.
Several common designs for the graph diffusion matrix $P_{_G}$ would be the adjacency matrix of the graph (denoted as $A$).
$P_{_G}$ can also be the normalized adjacency matrix, $D^{-1}A$, with $D$ being the degree-diagonal matrix.
Adding an identity matrix in $A$ is equivalent to adding self-loops in $G$.
For GCN, we set $U^{(t)}$ as zero, $\rho_t$ and $\psi_t$ as identity mappings. %

\medskip
\noindent\textbf{Notations.}
For any matrix $X$, let $\bignorms{X}$ denote the largest singular value (or spectral norm) of $X$.
Let $\bignormFro{X}$ denote the Frobenius norm of $X$.
We use the notation $f(N) \lesssim g(N)$ to indicate that there exists a fixed constant $c$ that does not grow with $N$ such that $f(N) \le c\cdot g(N)$ for large enough values of $N$.
Let $\cW$ and $\cU$ denote the union of the $W$ and $U$ matrices in a model $f$, respectively.

\subsection{Main results}

Given a message-passing neural network denoted as $f$, what can we say about its generalization performance, i.e., the gap between population and empirical risks? %
Let $f(X, G)$ denote the output of $f$, given input with graph $G$, node feature matrix $X$, and label $y$.
The loss of $f$ for this input example is denoted as $\ell(f(X, G), y)$.
Let $\hat{\cL}(f)$ denote the empirical loss of $f$ over the training set.
Let $\cL(f)$ denote the expected loss of $f$ over a random drawn of $\cD$.
We are interested in the generalization gap of $f$, i.e., $\cL(f) - \hat{\cL}(f)$.
How would the graph diffusion matrix $P_{_G}$ affect the generalization gap of graph neural networks?

To motivate our result, we examine the effect of incorporating graph diffusion in a one-layer linear neural network.
That is, we consider $f(X, G)$ to be $\frac 1 n \bm 1_n^{\top} P_{_G} X W^{(1)}$, which does not involve any nonlinear mapping for simplicity of our discussion.
In this case, by standard spectral norm inequalities for matrices, the Euclidean norm of $f$ (which is a vector) satisfies:
{\begin{align}
    \bignorm{f(X, G)} = & \bignorm{\frac 1 n \bm 1_n^{\top} P_{_G} X W^{(1)}} \nonumber \\
    \le & \bignorms{\frac 1 n \bm 1_n^{\top}} \cdot \bignorms{P_{_G}} \cdot \bignorms{X} \cdot \bignorm{W^{(1)}} \label{eq_one_layer_gcn}
\end{align}}%
Thus, provided that the loss function $\ell(\cdot, y)$ is Lipschitz-continuous, standard arguments imply that the generalization gap of $f$ scales with the spectral norm of ${P_{_G}}$ (divided by ${\sqrt N}$) \cite{mohri2018foundations}.
Let us compare this statement with a fully-connected neural net that averages the node features, i.e., the graph diffusion matrix $P_{_G}$ is the identity matrix.
The spectral norm of $P_{_G}$ becomes one.
Together, we conclude that the graph structure affects the generalization bound of a single layer GNN by adding the spectral norm of ${P_{_G}}$.

Our main result is that incorporating the spectral norm of the \emph{graph diffusion matrix} $P_{_G}^{l-1}$ is sufficient for any $l$ layer MPNN.
We note that the dependence is a power of $l-1$ because there are $l-1$ graph diffusion layers: see equation \eqref{eq_matrix_mpgnn}.
Let $f$ be an $l$-layer network whose weights $\cW, \cU$ are defined within a hypothesis set $\cH$: For every layer $i$ from $1$ up to $l$, we have that
{\begin{align}
     \bignorms{W^{(i)}} \le& s_i, ~~\bignormFro{W^{(i)}} \le s_i r_i, \nonumber\\
     \bignorms{U^{(i)}} \le& s_i, ~~\bignormFro{U^{(i)}} \,\,\le s_i r_i,\label{eq_cH}
\end{align}}%
where  $s_1, s_2, \dots, s_l$ and $r_1, r_2, \dots, r_l$ are bounds on the spectral norm and stable rank and are all greater than or equal to one, without loss of generality.
We now present the full statement.

\begin{theorem}\label{thm_mpgnn}
    Suppose all of the nonlinear activations in $\set{\phi_t, \rho_t, \psi_t: \forall\, t}$ and the loss function $\ell(\cdot, y)$ (for any fixed label $y \in \cY$) are twice-differentiable, Lipschitz-continuous and their first-order and second-order derivatives are both Lipschitz-continuous.
    
    With probability at least $1 - \delta$ over the randomness of $N$ independent samples from $\cD$, for any $\delta > 0$, and any $\epsilon > 0$ close to zero, any model $f$ with weight matrices in the set $\cH$ satisfies:
    {\begin{align}
        \cL(f) \leq (1 + \epsilon) \hat{\cL}(f) %
        + {\sum_{i=1}^l \sqrt{\frac{ C B d_i \Bigbrace{\max\limits_{(X, G, y) \sim \cD}\bignorms{X}^2 \bignorms{P_G}^{2(l-1)}} \Bigbrace{r_i^2 \prod\limits_{j=1}^l s_j^2} } {N}} }
        + \bigo{\frac{\log(\delta^{-1})}{N^{3/4}}} \label{eq_thm_mpgnn}, %
    \end{align}}%
    where $B$ is an upper bound on the value of the loss function $\ell(x, y)$ for any $(x, y) \sim \cD$, $C$ is a fixed constant depending on the activation, and the loss function (see Eq. \eqref{eq_const}, Appendix \ref{proof_theorem}).
\end{theorem}

As a remark, prior works by \citet{garg2020generalization} and \citet{liao2020pac} consider an MPNN with $W^{(t)}$ and $U^{(t)}$ being the same for $t$ from $1$ up to $l$, motivated by practical designs \cite{gilmer2017neural,jin2018junction}.
Thus, their analysis is conducted separately for GCN and MPNN with weight tying.
By contrast, our result allows $W^{(t)}$ and $U^{(t)}$ to be arbitrarily different across different layers.
This unifies GCN and MPNN without weight tying in the same framework so that we can unify their analysis.
We defer the proof sketch of our result and a discussion to Section \ref{sec_proff_sketch}.

\subsection{Comparison with prior art}\label{sec_compare}

\begin{table*}[t!]
    \caption{How does the generalization gap of graph neural networks scale with graph properties? In this work, we show spectrally-normalized bounds on $P_{_G}$ and compare our results with prior results in the following table.
    We let $A$ denote the adjacency matrix, $D$ be the degree-diagonal matrix of $A$, and $l$ be the depth of the GNN.
    Previous generalization bounds scale with the graph's maximum degree denoted as $d$.
    Our result instead scales with the spectral norm of $P_{_G}$ and applies to  graph isomorphism networks (GIN) \cite{xu2018powerful} and GraphSAGE with mean aggregation \cite{hamilton2017inductive}.
    }\label{table_theory}
    \centering
    {\begin{tabular}{@{} | c | c | c | c | c |@{}}
    \toprule
         Graph Dependence & {GCN} & {MPNN} & {GIN} & {GraphSAGE-Mean} \\
    \midrule
        \citeay{garg2020generalization}  & $d^{l-1}$ & $d^{l-1}$ & - & - \\
        \citeay{liao2020pac}  & $d^{\frac{l-1} 2}$ & $d^{l-1}$ & - & -\\
        \textbf{Ours (Theorems \ref{thm_mpgnn} and \ref{thm_gin})} & $1$ & $\bignorms{A}^{l-1}$ & $\sum_{i=1}^{l-1} \frac{\bignorms{A}^i}{l-1}$ & $\bignorms{D^{-1} A}^{l-1}$ \\
    \bottomrule
    \end{tabular}}%
\end{table*}

In Table \ref{table_theory}, we compare our result with prior results.
We first illustrate the effects of graph properties on the generalization bounds.
Then we will also show a numerical comparison to incorporate the other components of the bounds. 
\begin{itemize}[leftmargin=15pt]
    \item Suppose $P_{_G}$ is the adjacency matrix of $G$.
    Then, one can show that for any undirected graph $G$, the spectral norm of ${P_{_G}}$ is less than the maximum degree $d$ (cf. Fact \ref{fact_graph}, Appendix \ref{app_proof} for a proof).
    This explains why our result is strictly less than prior results for MPNN in Table \ref{table_theory}.
    
    \item Suppose $P_{_G}$ is the normalized and symmetric adjacency matrix of $G$: $P_{_G} = \tilde D^{-1/2} \tilde A \tilde D^{-1/2}$,  where $\tilde A$ is $A + \id$ and $\tilde D$ is the degree-diagonal matrix of $\tilde A$.
    Then, the spectral norm of $P_{_G}$ is at most one (cf. Fact \ref{fact_graph}, Appendix \ref{app_proof} for a proof).
    This fact explains why the graph dependence of our result for GCN is $1$ in Table \ref{table_theory}.
    Thus, we can see that this provides an exponential improvement compared to the prior results.
\end{itemize}
Thus, for the above diffusion matrices, we conclude that the spectral norm of ${P_{_G}}$ is strictly smaller than the maximum degree of graph $G$ (across all graphs in the distribution $\cD$).

Next, we conduct an empirical analysis to compare our results and prior results numerically.
Following the setting of prior works, we use two types of models that share their weight matrices across different layers, including GCN \cite{kipf2016semi} and the MPNN specified in \citet{liao2020pac}. 
For both models, we evaluate the generalization bounds by varying the network depth $l$ between $2, 4$, and $6$.

We consider graph prediction tasks on three collaboration networks,  including IMDB-B, IMDB-M, and COLLAB \cite{yanardag2015deep}.
IMDB-B includes a collection of movie collaboration graphs. In each graph, a node represents an actor or an actress, and an edge denotes a collaboration in the same movie. The task is to classify each graph into the movie genre as Action or Romance.
The IMDB-M is a multi-class extension with the movie graph label Comedy, Romance, or Sci-Fi.
COLLAB includes a list of ego-networks of scientific researchers. Each graph includes a researcher and her collaborators as nodes. An edge in the graph indicates a collaboration between two researchers. The task is to classify each ego-network into the field of the researcher, including High Energy, Condensed Matter, and Astro Physics.

We report the numerical comparison in Figure \ref{fig_bound_measurement}, averaged over three random seeds.
Our results are consistently smaller than previous results.
As explained in Table \ref{table_theory}, the improvement comes from the spectral norm bounds on graphs compared with the max degree bounds.

\begin{figure*}[t!]
	\begin{subfigure}[b]{0.33\textwidth}
		\centering
		\includegraphics[width=0.8\textwidth]{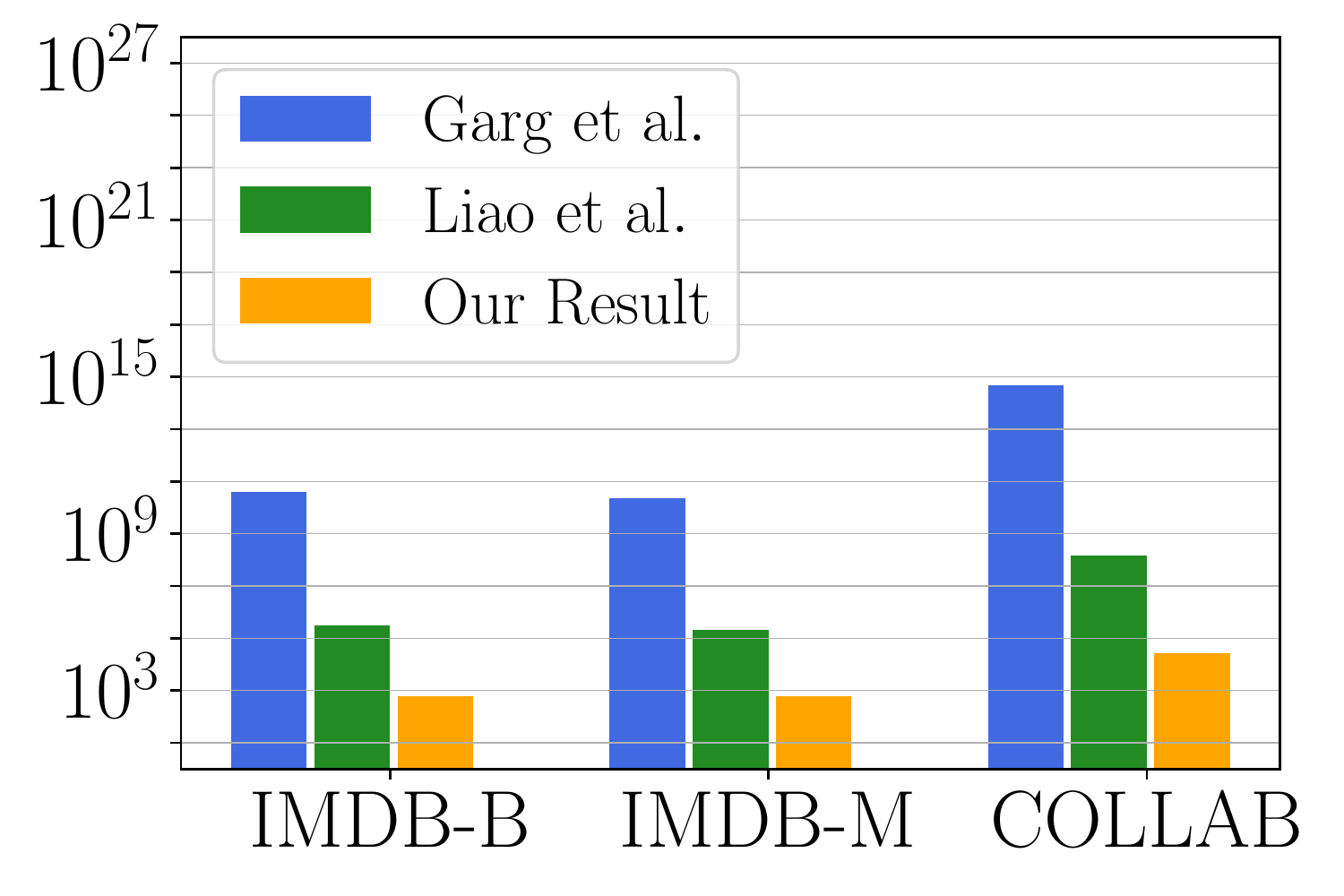}
		\caption{Two-layer GCN}
	\end{subfigure}\hfill%
	\begin{subfigure}[b]{0.33\textwidth}
		\centering
		\includegraphics[width=0.8\textwidth]{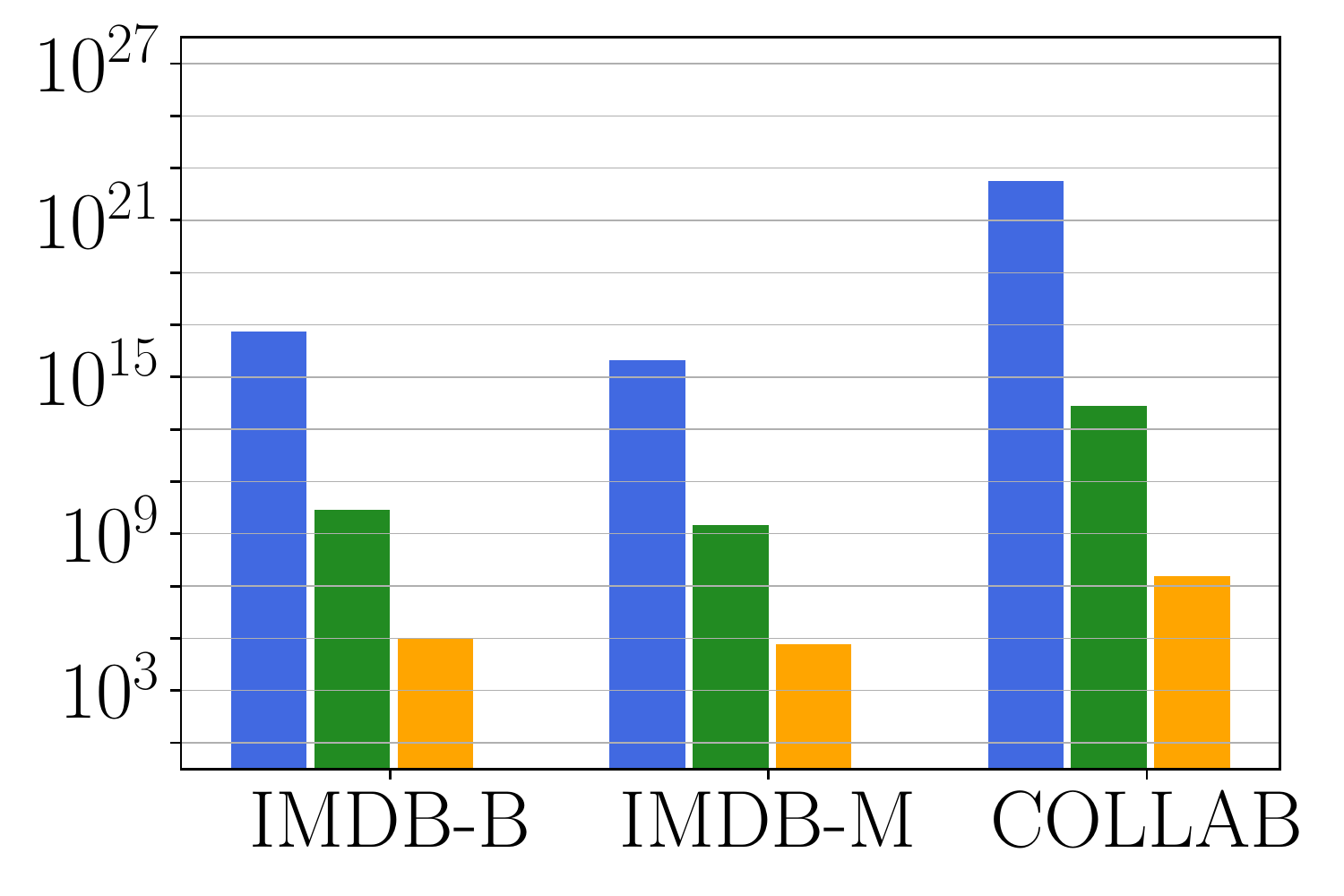}
		\caption{Four-layer GCN}
	\end{subfigure}\hfill%
	\begin{subfigure}[b]{0.33\textwidth}
		\centering
		\includegraphics[width=0.8\textwidth]{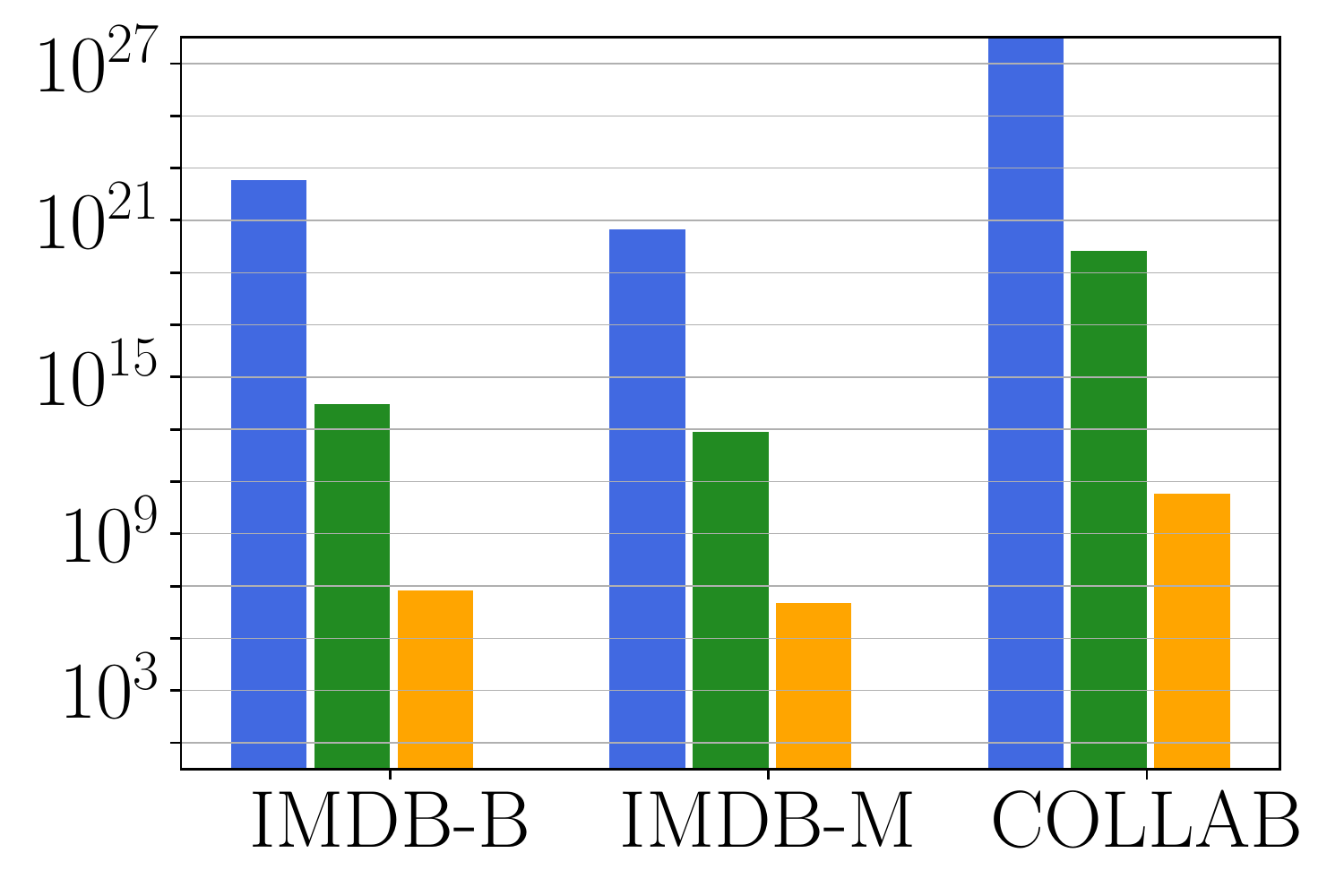}
		\caption{Six-layer GCN}
	\end{subfigure}\hfill%
    \begin{subfigure}[b]{0.33\textwidth}
		\centering
		\includegraphics[width=0.8\textwidth]{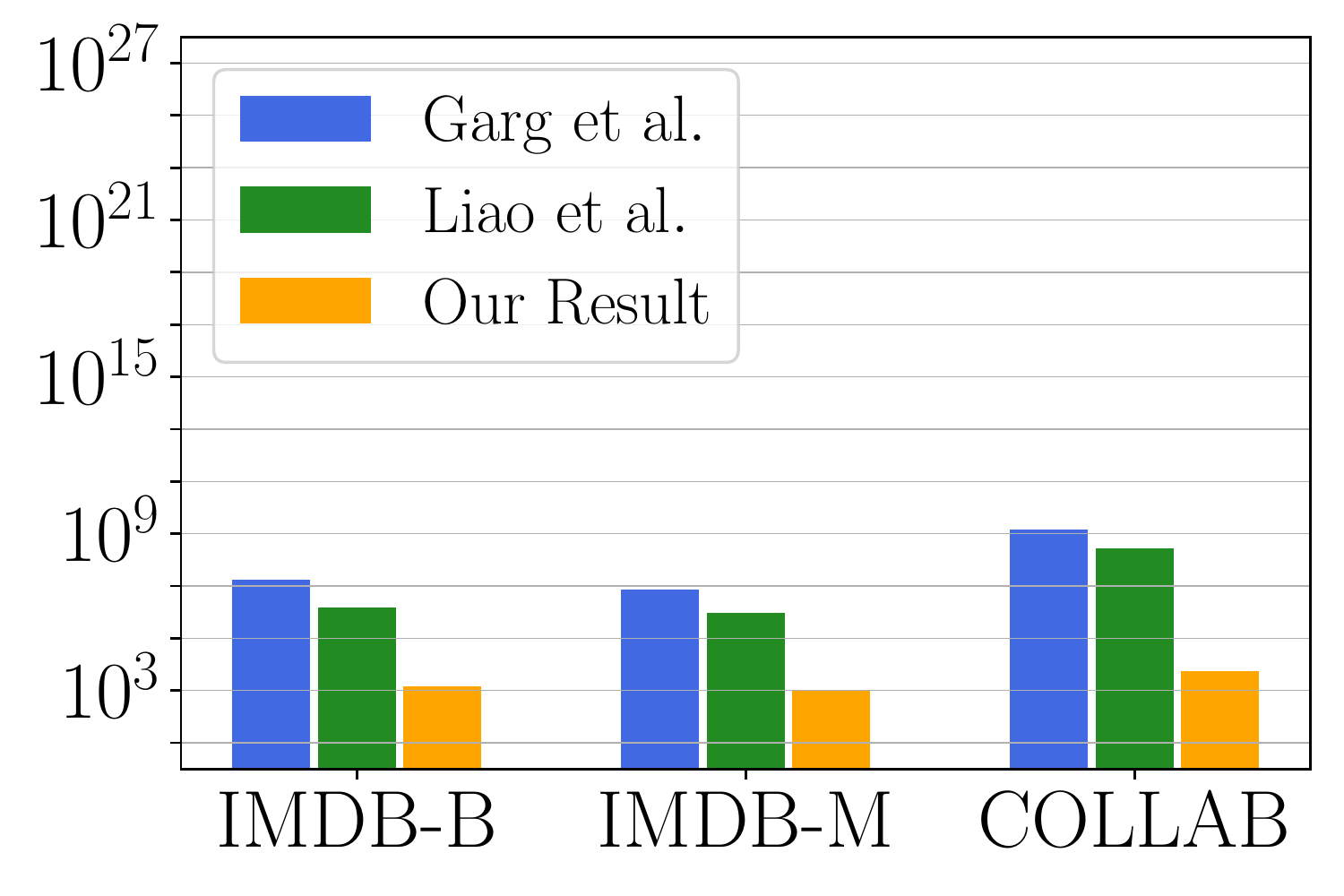}
		\caption{Two-layer MPNN}
	\end{subfigure}\hfill%
	\begin{subfigure}[b]{0.33\textwidth}
		\centering
		\includegraphics[width=0.8\textwidth]{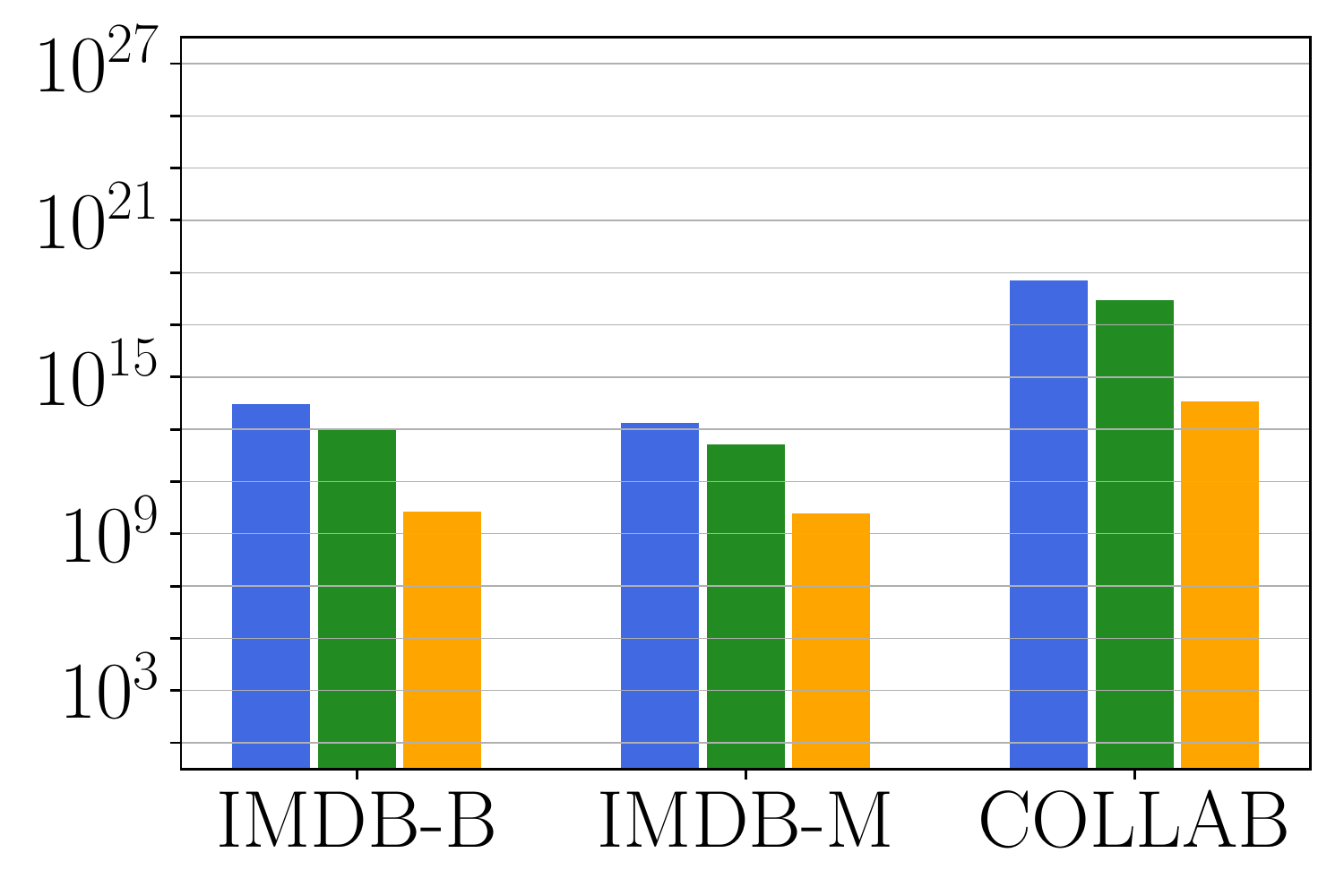}
		\caption{Four-layer MPNN}
	\end{subfigure}\hfill%
	\begin{subfigure}[b]{0.33\textwidth}
		\centering
		\includegraphics[width=0.8\textwidth]{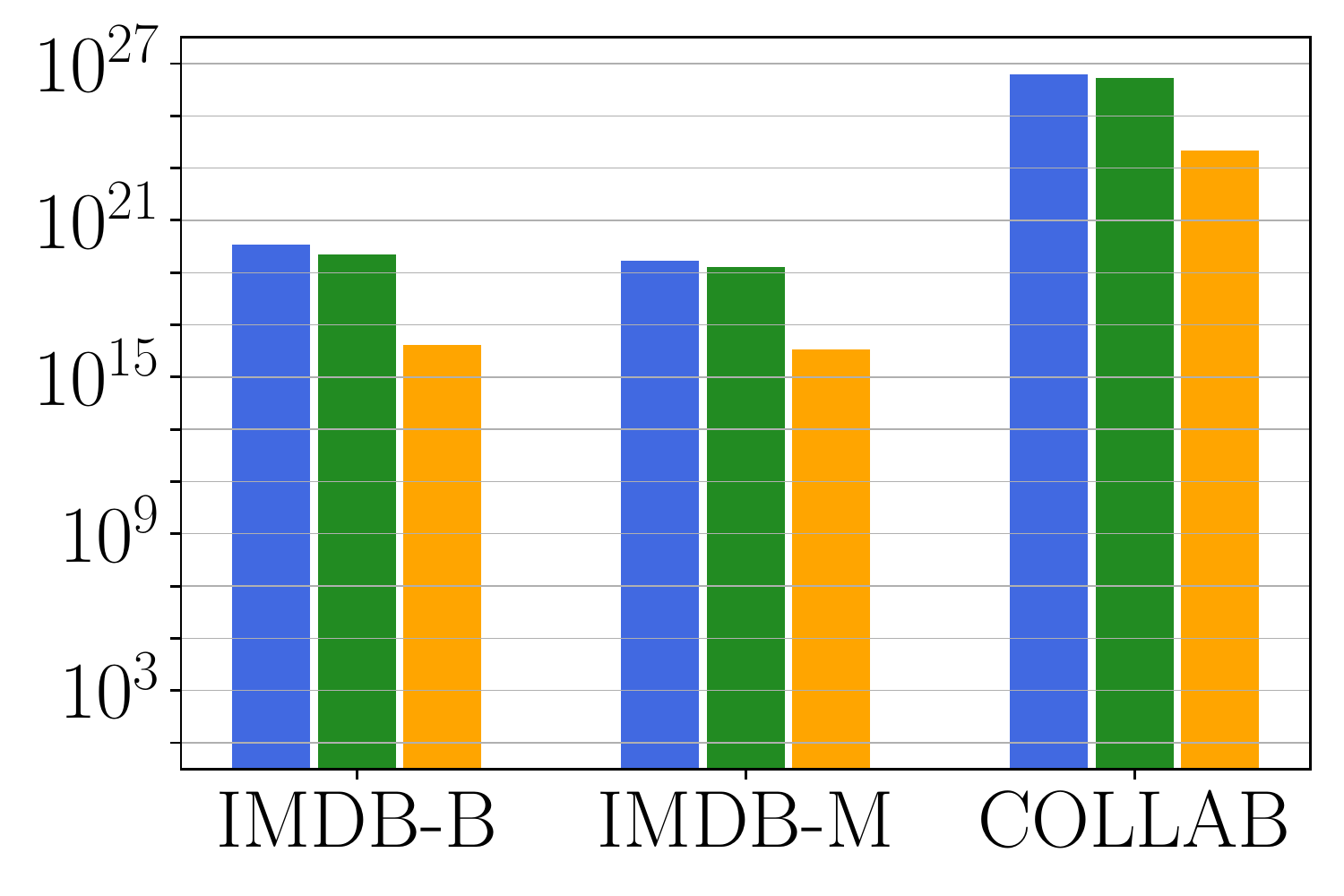}
		\caption{Six-layer MPNN}
	\end{subfigure}
    \caption{Comparing our result and prior results \cite{garg2020generalization, liao2020pac} on three graph classification tasks conducted on GCNs and MPNNs, respectively.}
    \label{fig_bound_measurement}
\end{figure*}

\subsection{A matching lower bound}

Next, we show an instance with the same dependence on the graph diffusion matrix as our upper bound.
In our example:
\begin{itemize}[leftmargin=15pt] 
    \item The graph is the complete graph with self-loops inserted in each node.
    Thus, the adjacency matrix of $G$ is a square matrix with all ones.
    We will set $P_{_G}$ as the adjacency matrix of $G$.

    \item In the first $l-1$ graph diffusion layers, the activation functions $\phi, \rho, \psi$ are all linear functions.
    Further, we fix all the parameters of $\cU$ as zero.
    
    \item The loss function $\ell$ is the logistic loss.
\end{itemize}
Then, we demonstrate a data distribution such that there always exists some weight matrices within $\cH$ whose generalization gap must increase in proportion to the spectral norm of $P_{_G}^{l-1}$ and the product of the spectral norm of every layer $s_1, s_2, \dots, s_l$.

\begin{theorem}\label{prop_lb}
    Let $N_0$ be a sufficiently large value.
    For any norms $s_1, s_2, \dots, s_n$, there exists a data distribution $\cD$ on which with probability at least $0.1$ over the randomness of $N$ independent samples from $\cD$, for any $N \ge N_0$, the generalization gap of $f$ is greater than the following:
    {\begin{align}
        \bigabs{\cL(f) - \hat\cL(f)}  
        \gtrsim\sqrt \frac{\Bigbrace{\max\limits_{(X, G, y) \sim \cD}\bignorms{P_{_G}}^{2(l-1)}} \Bigbrace{\prod\limits_{i=1}^l s_i^2} }{N}. \label{eq_lb}
    \end{align}}%
\end{theorem}
Notice that the lower bound in \eqref{eq_lb} exhibits the same scaling in terms of $G$---$\bignorms{P_{_G}}^{l-1}$---as our upper bound from equation \eqref{eq_thm_mpgnn}.
Therefore, we conclude that our spectral norm bound is tight for multilayer MPNN.
The proof of the lower bound can be found in Appendix \ref{app_lb}.

\begin{remark}\normalfont
    Our results from Theorem \ref{thm_mpgnn} and \ref{prop_lb} together suggest the generalization error bound scales linearly in $l$.
    To verify whether this is the case, we conducted an empirical study on three architectures (GCN, GIN-Mean, and GIN-Sum) that measured the growth of generalization errors as the network depth $l$ varies.
    We find that the generalization error grows sublinearly with $l$ to $\bignorms{P_{_G}}$. We also note that this sublinear growth trend has been captured by our Hessian-based generalization bound (cf. Figure \ref{fig_intro_gcn}). It would be interesting to understand better why the sublinear trend happens and further provide insight into the behavior of GNN.
\end{remark}

\begin{remark}\normalfont
    Theorem \ref{prop_lb} suggests that in the worst case, the generalization bound would have to scale with the spectral norms of the graph and the weight matrices. 
    Although this is vacuous for large $l$, later in Lemma \ref{lemma_gen_error}, we show a data-dependent bound using the trace of the Hessians, which is non-vacuous. As shown in Figure \ref{fig_intro_gcn}, Hessian-based measurements match the scale of actual generalization errors: the green line, calculated based on the trace of the loss Hessian matrix (cf. equation \eqref{eq_main_1}), matches the scale of actual generalization gaps plotted in the yellow line.
\end{remark}

\section{Proof Techniques and Extensions}\label{sec_proff_sketch}

Our analysis for dealing with the graph structure seems fundamentally different from the existing analysis.
In the margin analysis of \citet{liao2020pac}, the authors also incorporate the graph structure in the perturbation error.
For bounding the perturbation error, the authors use a triangle inequality that results in a $(1, \infty)$ norm of the matrix ${P_{_G}}$ (see Lemma 3.1 of \citet{liao2020pac} for GCN).
We note that this norm can be larger than the spectral norm by a factor of $\sqrt{n}$, where $n$ is the number of nodes in $G$: in the case of a star graph, this norm for the graph diffusion matrix of GCN is $\sqrt n$.
By comparison, the spectral norm of the same matrix is less than one (see Fact \ref{fact_graph}, Appendix \ref{app_proof}).

How can we tighten the perturbation error analysis and the dependence on $P_{_G}$ in the generalization bounds, then?
Our proof involves two parts:
\begin{itemize}[leftmargin=15pt]
    \item {\bf Part I:} By expanding the perturbed loss of a GNN, we prove a bound on the generalization gap using the trace of the Hessian matrix associated with the loss.
    \item {\bf Part II:} Then, we explicitly bound the trace of the Hessian matrix with the spectral norm of the graph using the Lipschitzness of the activation functions.
\end{itemize}

\medskip
\noindent\textbf{Part I:} {\itshape Measuring noise stability using the Hessian.} We first state an implicit generalization bound that measures the trace of the Hessian matrix.
Let $\bH^{(i)}$ denote the Hessian matrix of the loss $\ell(f(X, G), y)$ with respect to layer $i$'s parameters, for each $i$ from $1$ up to $l$. Particularly, $\bH^{(i)}$ is a square matrix whose dimension depends on the number of variables within layer $i$.
Let $\bH$ denote the Hessian matrix of the loss $\ell(f(X, G), y)$ over all parameters of $f$.

\begin{lemma}\label{lemma_gen_error}
    In the setting of Theorem \ref{thm_mpgnn}, with probability at least $1 - \delta$ over the randomness of the $N$ training examples, for any $\delta > 0$ and $\epsilon$ close to $0$, we get:
    {\small\begin{align}
        \cL(f) \leq (1 + \epsilon) \hat{\cL}(f) 
        + (1 + \epsilon) \sum_{i=1}^l\sqrt{\frac{B\cdot \left(\max\limits_{(X, G, y)\sim\cD}\tr\big[\bH^{(i)}[\ell(f(X, G), y)]\big]\right) s_i^2 r_i^2}{N}}
        + \bigo{\frac{\log(\delta^{-1})}{N^{3/4}}}.\label{eq_main_1} 
    \end{align}}%
\end{lemma}

\paragraph{Proof Sketch.} At a high level, the above result follows from Taylor's expansion of the perturbed loss.
Suppose each parameter of $f$ is perturbed by an independent noise drawn from a Gaussian distribution with mean zero and variance $\sigma^2$.
Let $\tilde \ell(f(X, G), y)$ be the perturbed loss value of an input example $X, G$ with label $y$.
Let $\cE$ denote the noise injections organized in a vector.
Using Taylor's expansion of the perturbed loss $\tilde \ell$, we get:
{ 
\begin{align}
         \tilde \ell(f(X, G), y) - \ell(f(X, G), y) \label{eq_perturb} 
    =   \cE^{\top} \nabla \ell(f(X, G), y) + \frac 1 2 {\cE}^{\top} \bH\big[\ell(f(X, G), y)\big] {\cE} + \order{\sigma^3}. 
\end{align}}%
Notice that the expectation of the first-order expansion term above equals zero.
The expectation of the second-order expansion term becomes $\sigma^2$ times the trace of the loss Hessian.
To derive equation \eqref{eq_main_1}, we use a PAC-Bayes bound of \citet[Theorem 2]{mcallester2013pac}.
There are two parts to this PAC-Bayes bound:
\begin{itemize}[leftmargin=15pt]
    \item The expectation of the noise perturbation in equation \eqref{eq_perturb}, taken over the injected noise $\cE$;
    \item The KL divergence between the prior and the posterior, which is at most $s_i^2 r_i^2$ for layer $i$, for $i$ from $1$ up to $l$, within the hypothesis set $\cH$.
\end{itemize}
Thus, one can balance the two parts by adjusting the noise variance at each layer---this leads to the layerwise Hessian decomposition in equation \eqref{eq_main_1}.

A critical step is showing the uniform convergence of the Hessian matrix.
We achieve this based on the Lipschitz-continuity of the first and twice derivatives of the nonlinear activation mappings.
With these conditions, we prove the uniform convergence with a standard $\epsilon$-cover argument.
The complete proof can be found in Appendix \ref{proof_trace}.

\begin{remark}\label{remark_node}\normalfont
Our argument in Lemma \ref{lemma_gen_error} applies to graph-level prediction tasks, which assume an unknown distribution of graphs.
A natural question is whether the analysis applies to node-level prediction tasks, which are often treated as semi-supervised learning problems.
The issue with directly applying our analysis to semi-supervised learning is that the size of a graph is only finite.
Instead, a natural extension would be to think about our graph as a random sample from some population and then argue about generalization in expectation of the random sample.
It is conceivable that one can prove a similar spectral norm bound for node prediction in this extension. %
This would be an interesting question for future work.
\end{remark}

\medskip
\noindent\textbf{Part II:} {\itshape Spectral norm bounds of the trace of the Hessian.} Next, we explicitly analyze the trace of the Hessian at each layer.
We bound the trace of the Hessian using the spectral norm of the weight matrices and the graph based on the Lipschitz-continuity conditions from Theorem \ref{thm_mpgnn}.
Notice that the last layer is a linear pooling layer, which can be deduced from layer $l-1$. Hence, we consider the first $l-1$ layers below.

\begin{lemma}\label{lemma_trace_hess}
    In the setting of Theorem \ref{thm_mpgnn}, 
    the trace of the loss Hessian matrix $\bH^{(i)}$ taken over $W^{(i)}$ and $U^{(i)}$ satisfies the following, for any $i = 1,2,\cdots,l-1$,
    {\begin{align}
         & \bigabs{\bigtr{\bH^{(i)}\bigbracket{\ell\bigbrace{f(X, G), y}}}} \nonumber \\ 
    \lesssim ~ & s_l^2 \Bigg(\sum_{p=1}^{d_{i-1}}\sum_{q=1}^{d_i} \bignormFro{\frac{\partial^2 H^{(l-1)}}{\partial \big(W^{(i)}_{p,q}\big)^2}} + \sum_{p=1}^{d_0}\sum_{q=1}^{d_i} \bignormFro{\frac{\partial^2 H^{(l-1)}}{\partial\bigbrace{U_{p,q}^{(i)}}^2}}   
        + \bignormFro{\frac{\partial H^{(l-1)}}{\partial W^{(i)}}}^2 + \bignormFro{\frac{\partial H^{(l-1)}}{\partial U^{(i)}}}^2 \Bigg) \label{eq_loss_3} \\
    \lesssim ~ & {\bignorms{X}^2\bignorms{P_{_G}}^{2(l-1)}}  {\prod\nolimits_{j=1:\,j\neq i}^{l} s_j^2}. \label{eq_loss_4}
    \end{align}}%
\end{lemma}

\paragraph{Proof Sketch.} Equation \eqref{eq_loss_3} uses the chain rule to expand out the trace of the Hessian and then applies the Lipschitzness of the loss function. Based on this result, equation \eqref{eq_loss_4} then bounds the first and second derivatives of $H^{(l-1)}$.
This step is achieved via an induction of $\partial H^{(j)}$ and $\partial^2 H^{(j)}$ over $W^{(i)}$ and $U^{(i)}$, for $j = 1,\dots,l-1$ and $i = 1,\dots, j$.
The induction relies on the feedforward architecture and the Lipschitzness of the first and second derivatives.
We leave out a few details, such as the constants in equations \eqref{eq_loss_3} and \eqref{eq_loss_4} that can be found in Appendix \ref{proof_first} and \ref{proof_second}.
Combining both parts together, we get equation \eqref{eq_matrix_mpgnn}.

\smallskip
\begin{remark}\label{remark_tech}\normalfont
We compare our analysis with the approach of \citet{liao2020pac}.
Both our analysis and \citet{liao2020pac} follow the PAC-Bayesian framework.
But additionally, we explore Lipschitz-continuity properties of the first and second derivatives of the activation functions (e.g., examples of such activations include tanh and sigmoid).
This allows us to measure the perturbation loss with Hessians, which captures data-dependent properties much more accurately than the margin analysis of \citet{liao2020pac}. 
It would be interesting to understand if one could still achieve spectral norm bounds on graphs under weaker smoothness conditions. This is left for future work.
\end{remark}

\subsection{Extensions}\label{sec_extension}

\textbf{Graph isomorphism networks.} This architecture concatenates every layer's embedding together for more expressiveness \cite{xu2018powerful}.
A classification layer is used after the layers.
Let $V^{(i)}$ denote a $d_i$ by $k$ matrix (recall $k$ is the output dimension).
Denote the set of these matrices by $\cV$.
We average the loss of all of the classification layers.
Let $\hat\cL_{_{GIN}}(f)$ denote the average loss of $f$ over $N$ independent samples of $\cD$.
Let ${\cL}_{_{GIN}}(f)$ denote the expected loss of $f$ over a random sample of $\cD$.
See also equation \eqref{eq_loss_gin} in Appendix \ref{app_gin} for their precise definitions.

Next, we state a generalization bound for graph isomorphism networks.
Let $f$ be any $l$-layer MPNN with weights defined in a hypothesis space $\cH$: the parameters of $f$ reside within the constraints from equation \eqref{eq_cH};
further, for every $i$ from $1$ up to $l$, the spectral norm of $V^{(i)}$ is less than $s_l$. %
Building on Lemma \ref{lemma_trace_hess}, we show a bound that scales with the spectral norm generalization of the averaged graph diffusion matrices.
Let $P_{_{GIN}}$ denote the average of $l-1$ matrices: $P_{_G}, P_{_G}^2, \dots, P_{_G}^{l-1}$.
We state the result below.

\begin{corollary}\label{thm_gin}
    Suppose the nonlinear activation mappings and the loss function satisfy the conditions stated in Theorem \ref{thm_mpgnn}.    
    With probability at least $1- \delta$ for any $\delta \ge 0$, and any $\epsilon$ close to zero, any $f$ in $\cH$ satisfies:
    {\small\begin{align}
         {L}_{_{GIN}}(f) \le (1+\epsilon)\hat\cL_{_{GIN}}(f) \label{eq_gin_result} 
         + {\sum_{i=1}^{l} \sqrt {\frac { {BC  d_i \cdot \left({\max\limits_{(X, G, y) \sim \cD} \bignorms{X}^2 \bignorms{ P_{_{GIN}}}^2}\right) \Bigbrace{r_i^2 \prod\limits_{j=1}^{l} s_j^2}}} {N}}} + \bigo{\frac{\log(\delta^{-1})}{N^{3/4}}}, 
    \end{align}}%
    where $B$ is an upper bound on the value of the loss function $\ell$ across the data distribution $\cD$, $C$ is a fixed constant that only depends on the Lipschitz-continuity of the activation mappings and the loss.
\end{corollary}

The proof can be found in Appendix \ref{app_gin}.
In particular, we apply the trace norm bound over the model output of every layer.
The classification layer, which only uses a linear transformation, can also be incorporated. 

\medskip
\noindent\textbf{Fine-tuning Graph Neural Networks.} We note that all of our bounds can be applied to the fine-tuning setting, where a graph neural network is initialized with pretrained weights and then fine-tuned on the target task.
The results can be extended to this setting by setting the norm bounds within equation \eqref{eq_cH} as the distance between the pretrained and fine-tuned model.

\section{An Algorithm for Fine-tuning Graph Neural Networks}\label{sec_alg}

A common practice to apply deep networks on graphs is to fit pretrained GNNs, which can be fine-tuned on a target task.
Typically, only a small amount of data is available for fine-tuning.
Thus, the fine-tuned model may overfit the training data, incurring a large generalization gap. %
A central insight from our analysis is that maintaining a small perturbed loss ensures lower generalization gaps.
Motivated by this observation, we design an algorithm to minimize the perturbed risk of GNN.

Let $f$ denote a GNN and $\tilde \ell(f)$ be the perturbed loss of $f$, with noise injected inside $f$'s weight matrices.
Recall from step \eqref{eq_perturb} that $\tilde \ell(f)$ is equal to $\ell(f)$ plus several expansion terms.
In particular, minimizing the expectation of $\tilde \ell(f)$ is equivalent to minimizing $\hat\cL(f)$ plus the trace of the Hessian matrix.
To estimate this expectation, we sample several noise perturbations independently.
Because Taylor's expansion of $\tilde \ell(f)$ also involves the gradient, we cancel this out by computing the perturbed loss with the negated perturbation. 
Algorithm \ref{alg_nso} describes the complete procedure.

We evaluate the above algorithm for fine-tuning pretrained GNNs.
Empirical results reveal that this algorithm achieves better test performance compared with existing regularization methods for five graph classification tasks.
The code repository for reproducing the experiments is at \url{https://github.com/VirtuosoResearch/Generalization-in-graph-neural-networks}.

\subsection{Experimental setup}

We focus on graph classification tasks, including five datasets from the MoleculeNet benchmark \cite{wu2018moleculenet}.
Each dataset aims to predict whether a molecule has a certain chemical property given its graph representation. 
We use pretrained GINs from \citet{hu2019strategies} and fine-tune the model on each downstream task. Following their experimental setup, we use the scaffold split for the dataset, and the model architecture is fixed for all five datasets. Each model has $5$ layers; each layer has $300$ hidden units and uses average pooling in the readout layer.
We set the parameters, such as the learning rate and the number of epochs following their setup.

\begin{algorithm}[t!]
	\caption{Algorithms for fine-tuning graph neural networks}\label{alg_nso}
	\begin{small}
		\textbf{Input}: A training dataset $\{(X_i, G_i, y_i)\}_{i=1}^{N}$ with node feature $X_i$, graph $G_i$, and graph-level label $y_i$, for $i = 1, \dots, N$.
  
        \vspace{0.02in}%
		\textbf{Require}: Number of perturbations $m$, noise variance $\sigma^2$, learning rate $\eta$, and number of epochs $T$.

        \vspace{0.02in}%
		\textbf{Output}: A trained model $f^{(T)}$.
  
		\begin{algorithmic}[1]
			\STATE At $t = 0$, initialize the parameters of $f^{(0)}$ with pretrained GNN weight matrices.
            \vspace{0.02in}
			\FOR{$1 \le t \le T$}
                \FOR{$1 \le i \le m$}
                    \vspace{0.02in}
                    \STATE Add perturbation $\cE_i$ drawn from a normal distribution with mean zero and variance $\sigma^2$.
                    \vspace{0.02in}
                    \STATE Let $\tilde \cL_i(f^{(t-1)})$ be the training loss of the model $f^{(t-1)}$ with weight matrix perturbed by $\cE_i$.
                    \vspace{0.02in}                    
                    \STATE Let $\tilde \cL_i^{'}(f^{(t-1)})$ be the training loss of the model $f^{(t-1)}$ with weight matrix perturbed by $-\cE_i$.
                    \vspace{0.02in}
                \ENDFOR
			    \STATE Use stochastic gradient descent to update $f^{(t)}$ as $f^{(t-1)} -  \frac{\eta}{2m} \sum_{i=1}^m 
			    \big(\nabla \tilde{\cL}_i\bigbrace{f^{(t-1)}}
			    + \nabla \tilde{\cL}_i^{'}\bigbrace{f^{(t-1)}}\big )$.
                \vspace{0.05in}
			\ENDFOR
		\end{algorithmic}
	\end{small}
\end{algorithm}

We compare our algorithm with previous regularization methods that serve as benchmark approaches for improving generalization.
This includes early stopping, weight decay, dropout, weight averaging \cite{izmailov2018averaging}, and distance-based regularization \cite{gouk2020distance}.
For implementing our algorithm, we set the number of perturbations as $10$ and choose the noise standard deviation $\sigma$ with a grid search in $\{0.01, 0.02, 0.05, 0.1, 0.2, 0.5\}$.

\subsection{Experimental results}\label{sec_empirical}

Table \ref{tab_molecule_pred} reports the test ROC-AUC performance averaged over multiple binary prediction tasks in each dataset. 
Comparing the average ranks of methods across datasets, our algorithm outperforms baselines on all five molecular property prediction datasets. 
The results support our theoretical analysis that the noise stability property of GNN is a strong measure of empirical generalization performance. 
Next, we provide details insights from applying our algorithm.

First, we hypothesize that our algorithm is particularly effective when the empirical generalization gap is large.
To test the hypothesis, we vary the size of the training set in the BACE dataset; we compare the performance of our algorithm with early stopping until epoch $100$. 
We plot the generalization gap between the training and test losses during training, shown in Figure \ref{fig_exp_N1}-\ref{fig_exp_N2}. 
As the trend shows, our algorithm consistently reduces the generalization gap, particularly when the training set size $N$ is $600$.  

Second, we hypothesize that our algorithm helps reduce the trace of the Hessian matrix (associated with the loss).
We validate this by plotting the trace of the Hessian as the number of epochs progresses during training, again using the BACE dataset as an example.
Specifically, we average the trace over the training dataset.
Figure \ref{fig_exp_tr} shows the averaged trace values during the fine-tuning process.
The results confirm that noise stability optimization reduces the trace of the Hessian matrix (more significantly than early stopping).
We note that noise stability optimization also reduces the largest eigenvalue of the Hessian matrix, along with reducing the trace.
This can be seen in Figure \ref{fig_exp_eig}.

Lastly, we study the number of perturbations used in our algorithm. While more perturbations would lead to a better estimation of the noisy stability, we observe that using $10$ perturbations is sufficient for getting the most gain.
We also validate that using negated perturbations consistently performs better than not using them across five datasets.
This is because the negated perturbation cancels out the first-order term in Taylor's expansion.
In our ablation study, we find that adding the negated perturbation performs better than not using it by 1\% on average over the five datasets.

\begin{remark}\normalfont %
We note that noise stability optimization is closely related to sharpness-aware minimization (SAM) \cite{foret2020sharpness}.
Noise stability optimization differs in two aspects compared with SAM.
First, SAM requires solving constrained minimax optimization, which may not even be differentiable \cite{daskalakis2021complexity}.
Our objective remains the same after perturbation.
Second, SAM reduces the largest eigenvalue of the Hessian matrix, which can be seen from Taylor's expansion of $\tilde \ell(f)$.
We reduce the trace of the Hessian matrix, thus reducing the largest eigenvalue.
There is another related work that regularizes noise stability in NLP \cite{hua2021noise}.
Their approach adds noise perturbation to the input and regularizes the loss change in the output. Our approach adds perturbation to weight matrices.
\end{remark}

\begin{table*}[!t]
\centering
\begin{footnotesize}
\caption{Test ROC-AUC (\%) score for five molecular property prediction datasets with different regularization methods. The reported results are averaged over five random seeds.}\label{tab_molecule_pred}
\begin{tabular}{@{}lcccccc@{}}
\toprule
Dataset & SIDER & ClinTox & BACE & BBBP & Tox21 \\ 
\# Molecule Graphs & 1,427 & 1,478 & 1,513 & 2,039 & 7,831 \\
\# Binary Prediction Tasks & 27 & 2 & 1 & 1 & 12 \\ \midrule
Early Stopping       & 61.06$\pm$1.48 & 68.25$\pm$2.63 & 82.86$\pm$0.95 & 67.80$\pm$1.05 & 77.52$\pm$0.23 \\
Weight Decay         & 61.30$\pm$0.21 & 67.43$\pm$2.88 & 83.72$\pm$0.99 & 67.98$\pm$2.41 & 78.23$\pm$0.35 \\
Dropout              & 63.90$\pm$0.90 & 73.70$\pm$2.80 & 84.50$\pm$0.70 & 68.07$\pm$1.30 & 78.30$\pm$0.30 \\
Weight Averaging     & 63.67$\pm$0.34 & 78.78$\pm$1.49 & 83.93$\pm$0.36	& 70.26$\pm$0.24 & 77.59$\pm$0.11 \\
Distance-based Reg. & 64.36$\pm$0.48 & 76.68$\pm$1.19 & 84.65$\pm$0.48 & 70.37$\pm$0.44 & 78.62$\pm$0.24 \\
\midrule
\textbf{Ours (Alg. \ref{alg_nso})} & \textbf{65.13$\pm$0.18} & \textbf{80.18$\pm$0.82} & \textbf{85.07$\pm$0.43} & \textbf{71.22$\pm$0.36} & \textbf{79.31$\pm$0.24} \\ \bottomrule
\end{tabular}
\end{footnotesize}
\end{table*}

\begin{figure*}[!t]
    \centering
	\caption{In Figures \ref{fig_exp_N1} and \ref{fig_exp_N2}, we show that our algorithm is particularly effective at reducing the generalization gap for small training dataset sizes $N$.
	In Figures \ref{fig_exp_tr} and \ref{fig_exp_eig}, we find that both the trace and the largest eigenvalue of the loss Hessian matrix decreased during training.}
	\begin{subfigure}[b]{0.24\textwidth}
		\centering
		\includegraphics[width=0.975\textwidth]{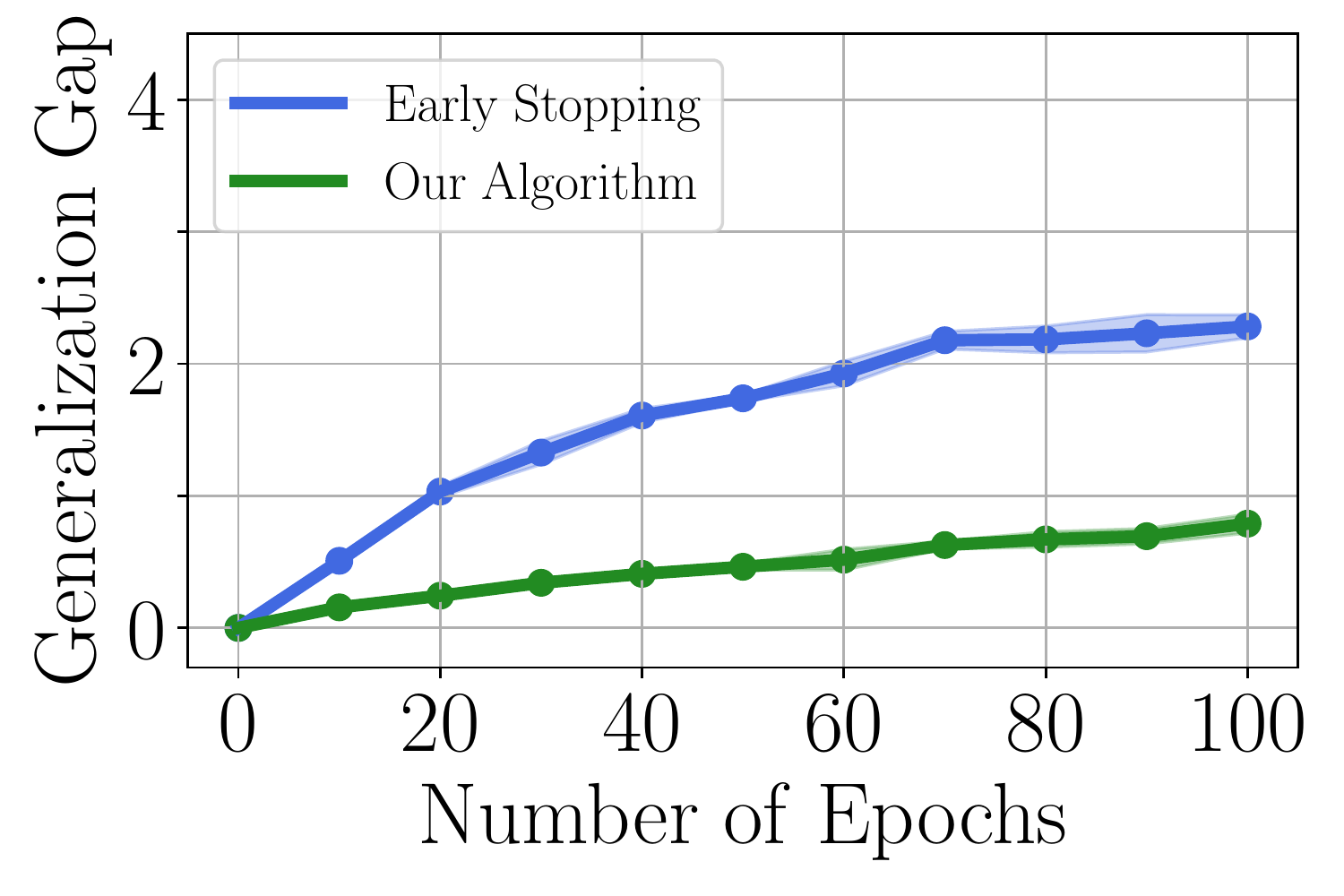}
		\caption{$N=1200$}\label{fig_exp_N1}
	\end{subfigure}\hfill
	\begin{subfigure}[b]{0.24\textwidth}
		\centering
		\includegraphics[width=0.975\textwidth]{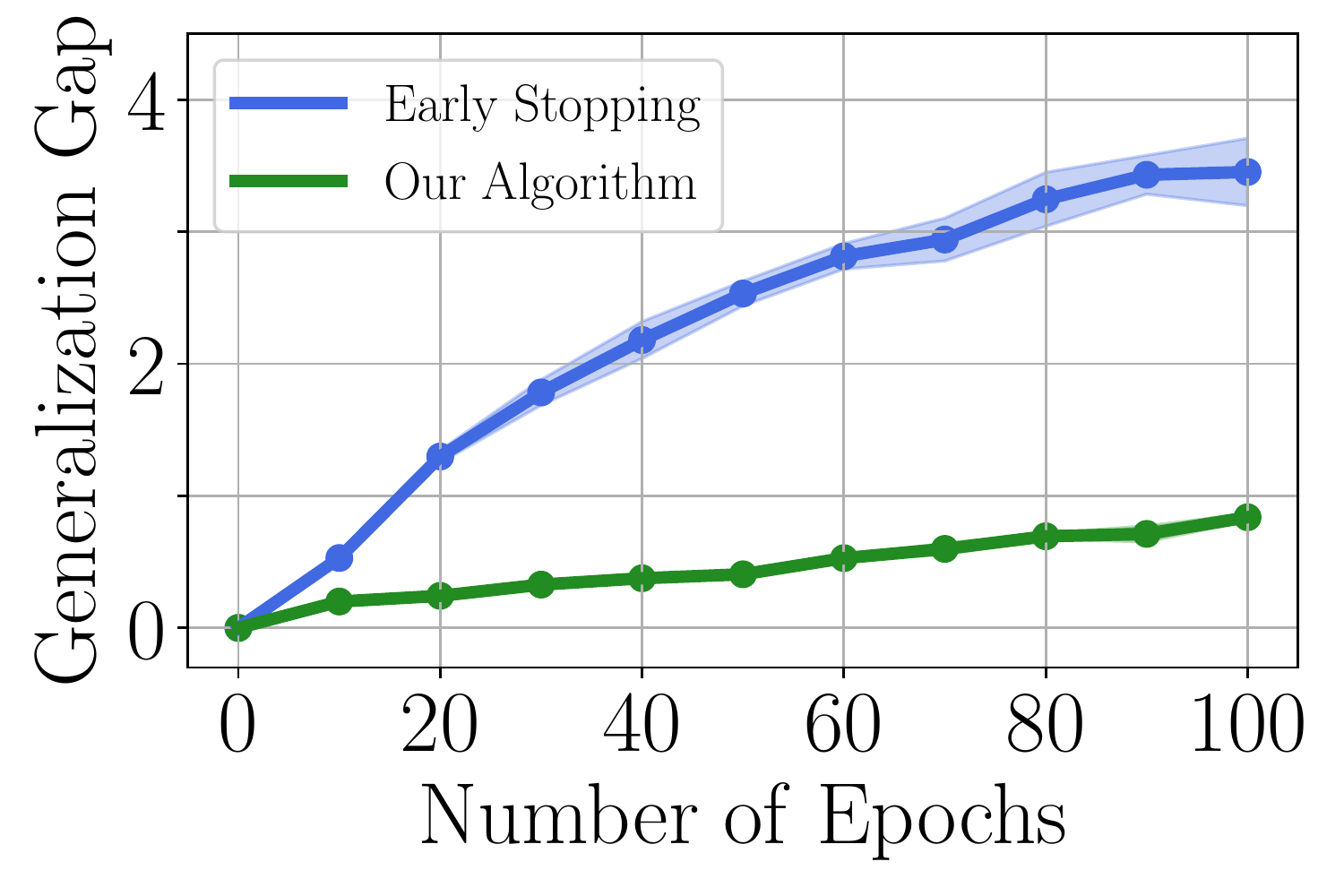}
		\caption{$N=600$}\label{fig_exp_N2}
	\end{subfigure}\hfill
	\begin{subfigure}[b]{0.24\textwidth}
		\centering
		\includegraphics[width=0.975\textwidth]{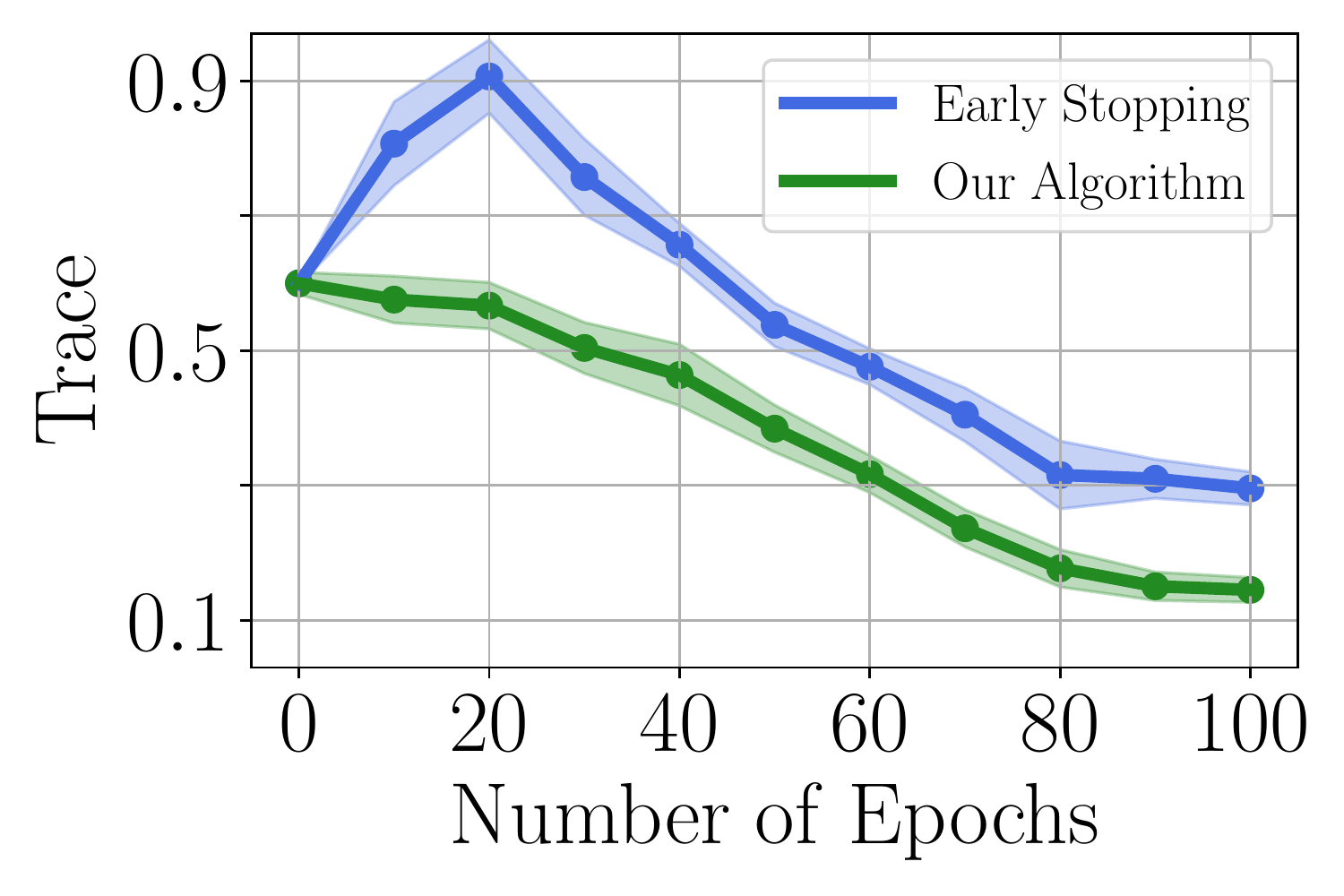}
		\caption{Trace}\label{fig_exp_tr}
	\end{subfigure}\hfill
	\begin{subfigure}[b]{0.24\textwidth}
		\centering
		\includegraphics[width=0.975\textwidth]{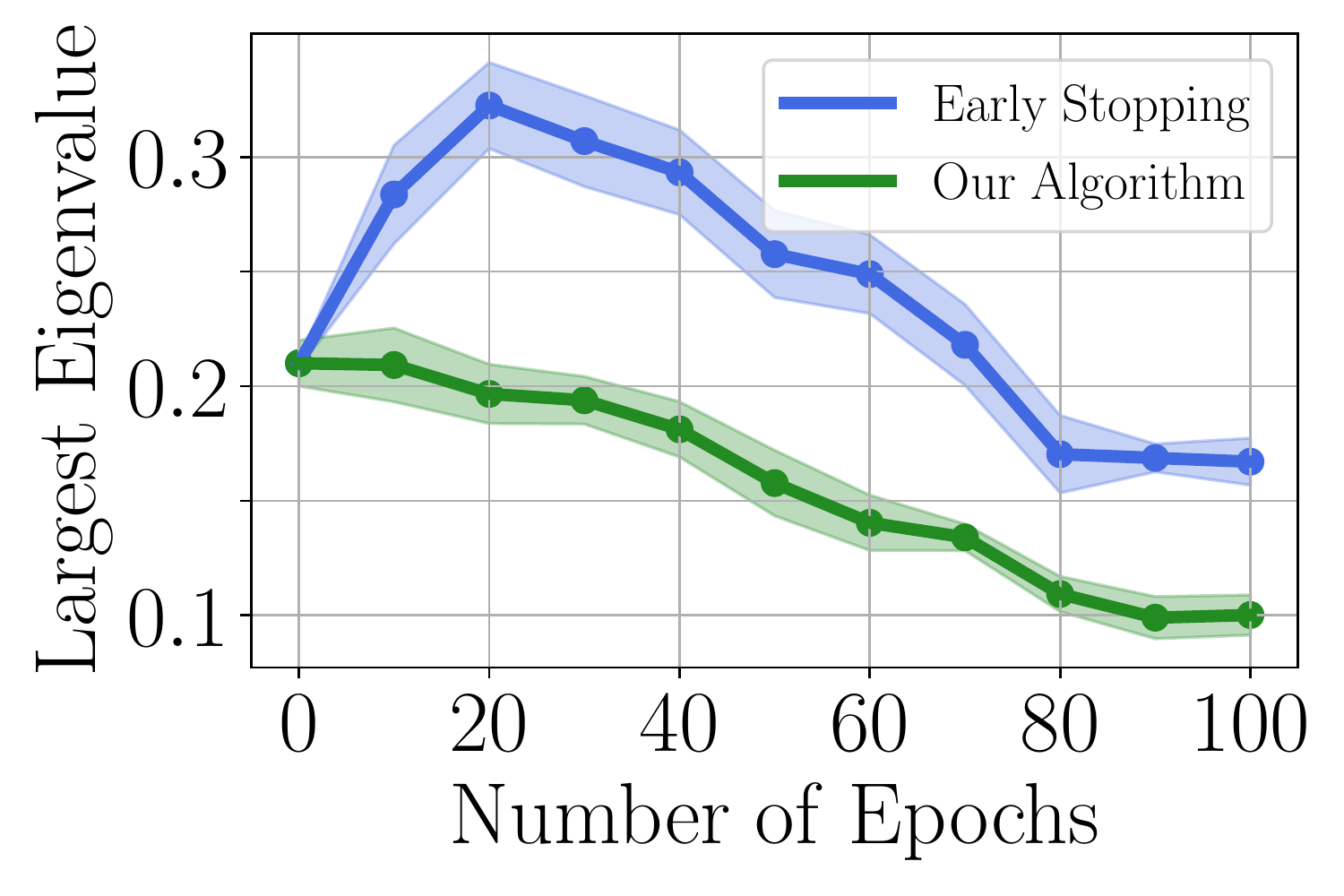}
		\caption{Largest Eigenvalue}\label{fig_exp_eig}
	\end{subfigure}%
    \label{fig_ablate_sample_size} 
\end{figure*}
\section{Conclusion}\label{sec_conclude}

This work develops generalization bounds for graph neural networks with a sharp dependence on the graph diffusion matrix. The results are achieved within a unified setting that significantly extends prior works. In particular, we answer an open question mentioned in \citet{liao2020pac}: a refined PAC-Bayesian analysis can improve the generalization bounds for message-passing neural networks. These bounds are obtained by analyzing the trace of the Hessian matrix with the Lipschitz-continuity of the activation functions. Empirical findings suggest that the Hessian-based bound matches observed gaps on real-world graphs. Thus, our work also develops a practical tool to measure the generalization performance of graph neural networks. The algorithmic results with noise stability optimization further demonstrate the practical implication of our findings.

Our work opens up many interesting questions for future work. Could the new tools we have developed be used to study generalization in graph attention networks \cite{velivckovic2017graph}? Could Hessians be used for measuring out-of-distribution generalization gaps of graph neural networks?
We remark that in a subsequent paper \cite{ju2023noise}, we provide a rigorous analysis of the convergence rates of our algorithm presented in the experiments, showing matching upper and lower bounds given Lipschitz-continuous conditions of the gradients.

\section*{Acknowledgement}

Thanks to Renjie Liao, Haoyu He, and the anonymous referees for providing constructive feedback on our work.
Thanks to Yang Yuan for the helpful discussions.
H. Z. is grateful to Stefanie Jegelka for an invitation to present this paper at her lab meeting.
H. J. and D. L. acknowledge financial support from the startup fund of Khoury College of Computer Sciences, Northeastern University.

\begin{refcontext}[sorting=nyt]
\printbibliography

@inproceedings{yanardag2015deep,
  title={Deep graph kernels},
  author={Yanardag, Pinar and Vishwanathan, S.},
  booktitle={KDD},
  year={2015}
}

@article{wu2018moleculenet,
  title={MoleculeNet: a benchmark for molecular machine learning},
  author={Wu, Zhenqin and Ramsundar, Bharath and Feinberg, Evan N and Gomes, Joseph and Geniesse, Caleb and Pappu, Aneesh S and Leswing, Karl and Pande, Vijay},
  journal={Chemical science},
  year={2018}
}

@article{hu2019strategies,
  title={Strategies for pre-training graph neural networks},
  author={Hu, Weihua and Liu, Bowen and Gomes, Joseph and Zitnik, Marinka and Liang, Percy and Pande, Vijay and Leskovec, Jure},
  journal={ICLR},
  year={2020}
}

@inproceedings{arora2018stronger,
  title={Stronger generalization bounds for deep nets via a compression approach},
  author={Arora, Sanjeev and Ge, Rong and Neyshabur, Behnam and Zhang, Yi},
  booktitle={ICML},
  year={2018},
}

@article{neyshabur2017pac,
  title={A pac-bayesian approach to spectrally-normalized margin bounds for neural networks},
  author={Neyshabur, Behnam and Bhojanapalli, Srinadh and Srebro, Nathan},
  journal={ICLR},
  year={2018}
}

@article{mcallester2013pac,
  title={A PAC-Bayesian tutorial with a dropout bound},
  author={McAllester, David},
  journal={arXiv preprint arXiv:1307.2118},
  year={2013}
}

@article{zhang2016understanding,
  title={Understanding deep learning requires rethinking generalization},
  author={Zhang, Chiyuan and Bengio, Samy and Hardt, Moritz and Recht, Benjamin and Vinyals, Oriol},
  journal={ICLR},
  year={2017}
}

@article{gouk2020distance,
	title={Distance-Based Regularisation of Deep Networks for Fine-Tuning},
	author={Gouk, Henry and Hospedales, Timothy M and Pontil, Massimiliano},
	journal={ICLR},
	year={2021}
}

@article{dziugaite2017computing,
	title={Computing nonvacuous generalization bounds for deep (stochastic) neural networks with many more parameters than training data},
	author={Dziugaite, Gintare Karolina and Roy, Daniel M},
	journal={UAI},
	year={2017}
}

@article{jiang2019fantastic,
	title={Fantastic generalization measures and where to find them},
	author={Jiang, Yiding and Neyshabur, Behnam and Mobahi, Hossein and Krishnan, Dilip and Bengio, Samy},
	journal={ICLR},
	year={2020}
}

@article{bartlett2017spectrally,
	title={Spectrally-normalized margin bounds for neural networks},
	author={Bartlett, Peter and Foster, Dylan J and Telgarsky, Matus},
	journal={NeurIPS},
	year={2017}
}

@inproceedings{guedj2019primer,
  title={A Primer on PAC-Bayesian Learning},
  author={Guedj, Benjamin},
  booktitle={Proceedings of the French Mathematical Society},
  year={2019}
}

@article{foret2020sharpness,
  title={Sharpness-aware minimization for efficiently improving generalization},
  author={Foret, Pierre and Kleiner, Ariel and Mobahi, Hossein and Neyshabur, Behnam},
  journal={ICLR},
  year={2021}
}

@article{long2020generalization,
  title={Generalization bounds for deep convolutional neural networks},
  author={Long, Philip M and Sedghi, Hanie},
  journal={ICLR},
  year={2020}
}

@article{jin2019short,
  title={A short note on concentration inequalities for random vectors with subgaussian norm},
  author={Jin, Chi and Netrapalli, Praneeth and Ge, Rong and Kakade, Sham M and Jordan, Michael I},
  journal={arXiv preprint arXiv:1902.03736},
  year={2019}
}

@inproceedings{golowich2018size,
  title={Size-independent sample complexity of neural networks},
  author={Golowich, Noah and Rakhlin, Alexander and Shamir, Ohad},
  booktitle={COLT},
  year={2018}
}

@article{bartlett2019nearly,
  title={Nearly-tight VC-dimension and pseudodimension bounds for piecewise linear neural networks},
  author={Bartlett, Peter and Harvey, Nick and Liaw, Christopher and Mehrabian, Abbas},
  journal={JMLR},
  year={2019}
}

@article{vishwanathan2010graph,
  title={Graph kernels},
  author={Vishwanathan, S. and Schraudolph, Nicol and Kondor, Risi and Borgwardt, Karsten},
  journal={JMLR},
  year={2010}
}

@article{li2021improved,
  title={Improved regularization and robustness for fine-tuning in neural networks},
  author={Li, Dongyue and Zhang, Hongyang R},
  journal={NeurIPS},
  year={2021}
}

@inproceedings{goel2015note,
  title={A note on modeling retweet cascades on Twitter},
  author={Goel, Ashish and Munagala, Kamesh and Sharma, Aneesh and Zhang, Hongyang},
  booktitle={Workshop on Algorithms and Models for the Web-Graph, 2015},
  year={2015}
}

@inproceedings{zhang2019pruning,
  title={Pruning based distance sketches with provable guarantees on random graphs},
  author={Zhang, Hongyang and Yu, Huacheng and Goel, Ashish},
  booktitle={WWW},
  year={2019}
}

@inproceedings{goel2014connectivity,
  title={Connectivity in random forests and credit networks},
  author={Goel, Ashish and Khanna, Sanjeev and Raghvendra, Sharath and Zhang, Hongyang},
  booktitle={SODA},
  year={2014}
}

@article{ju2023noise,
  title={Noise Stability Optimization for Flat Minima with Tight Rates},
  author={Ju, Haotian and Li, Dongyue and Zhang, Hongyang R},
  journal={arXiv preprint arXiv:2306.08553},
  year={2023}
}

@article{kipf2016semi,
  title={Semi-supervised classification with graph convolutional networks},
  author={Kipf, Thomas N and Welling, Max},
  journal={ICLR},
  year={2017}
}

@inproceedings{garg2020generalization,
  title={Generalization and representational limits of graph neural networks},
  author={Garg, Vikas and Jegelka, Stefanie and Jaakkola, Tommi},
  booktitle={ICML},
  year={2020}
}

@article{liao2020pac,
  title={A PAC-Bayesian Approach to Generalization Bounds for Graph Neural Networks},
  author={Liao, Renjie and Urtasun, Raquel and Zemel, Richard},
  journal={ICLR},
  year={2021}
}

@article{ju2022robust,
  title={Robust Fine-Tuning of Deep Neural Networks with Hessian-based Generalization Guarantees},
  author={Ju, Haotian and Li, Dongyue and Zhang, Hongyang R},
  journal={ICML},
  year={2022}
}

@article{xu2018powerful,
  title={How powerful are graph neural networks?},
  author={Xu, Keyulu and Hu, Weihua and Leskovec, Jure and Jegelka, Stefanie},
  journal={ICLR},
  year={2019}
}

@article{jegelka2022theory,
  title={Theory of Graph Neural Networks: Representation and Learning},
  author={Jegelka, Stefanie},
  journal={arXiv preprint arXiv:2204.07697},
  year={2022}
}

@inproceedings{dai2016discriminative,
  title={Discriminative embeddings of latent variable models for structured data},
  author={Dai, Hanjun and Dai, Bo and Song, Le},
  booktitle={ICML},
  year={2016}
}

@inproceedings{jin2018junction,
  title={Junction tree variational autoencoder for molecular graph generation},
  author={Jin, Wengong and Barzilay, Regina and Jaakkola, Tommi},
  booktitle={ICML},
  year={2018}
}

@article{hamilton2017inductive,
  title={Inductive representation learning on large graphs},
  author={Hamilton, Will and Ying, Zhitao and Leskovec, Jure},
  journal={NeurIPS},
  year={2017}
}

@article{chami2020machine,
  title={Machine learning on graphs: A model and comprehensive taxonomy},
  author={Chami, Ines and Abu-El-Haija, Sami and Perozzi, Bryan and R{\'e}, Christopher and Murphy, Kevin},
  journal={arXiv preprint arXiv:2005.03675},
  pages={1},
  year={2020}
}

@article{du2019graph,
  title={Graph neural tangent kernel: Fusing graph neural networks with graph kernels},
  author={Du, Simon S and Hou, Kangcheng and Salakhutdinov, Russ R and Poczos, Barnabas and Wang, Ruosong and Xu, Keyulu},
  journal={NeurIPS},
  year={2019}
}

@article{xu2019can,
  title={What can neural networks reason about?},
  author={Xu, Keyulu and Li, Jingling and Zhang, Mozhi and Du, Simon S and Kawarabayashi, Ken-ichi and Jegelka, Stefanie},
  journal={ICLR},
  year={2020}
}

@article{xu2020neural,
  title={How neural networks extrapolate: From feedforward to graph neural networks},
  author={Xu, Keyulu and Zhang, Mozhi and Li, Jingling and Du, Simon S and Kawarabayashi, Ken-ichi and Jegelka, Stefanie},
  journal={ICLR},
  year={2021}
}

@article{scarselli2008graph,
  title={The graph neural network model},
  author={Scarselli, Franco and Gori, Marco and Tsoi, Ah Chung and Hagenbuchner, Markus and Monfardini, Gabriele},
  journal={IEEE transactions on neural networks},
  year={2008}
}

@inproceedings{yehudai2021local,
  title={From local structures to size generalization in graph neural networks},
  author={Yehudai, Gilad and Fetaya, Ethan and Meirom, Eli and Chechik, Gal and Maron, Haggai},
  booktitle={ICML},
  year={2021}
}

@article{scarselli2018vapnik,
  title={The vapnik--chervonenkis dimension of graph and recursive neural networks},
  author={Scarselli, Franco and Tsoi, Ah Chung and Hagenbuchner, Markus},
  journal={Neural Networks},
  year={2018}
}

@inproceedings{morris2019weisfeiler,
  title={Weisfeiler and leman go neural: Higher-order graph neural networks},
  author={Morris, Christopher and Ritzert, Martin and Fey, Matthias and Hamilton, William L and Lenssen, Jan Eric and Rattan, Gaurav and Grohe, Martin},
  booktitle={AAAI},
  year={2019}
}

@article{hamilton2017representation,
  title={Representation learning on graphs: Methods and applications},
  author={Hamilton, Will and Ying, Rex and Leskovec, Jure},
  journal={arXiv preprint arXiv:1709.05584},
  year={2017}
}

@article{chen2020can,
  title={Can graph neural networks count substructures?},
  author={Chen, Zhengdao and Chen, Lei and Villar, Soledad and Bruna, Joan},
  journal={NeurIPS},
  year={2020}
}

@article{sato2019approximation,
  title={Approximation ratios of graph neural networks for combinatorial problems},
  author={Sato, Ryoma and Yamada, Makoto and Kashima, Hisashi},
  journal={NeurIPS},
  year={2019}
}

@article{azizian2020expressive,
  title={Expressive power of invariant and equivariant graph neural networks},
  author={Azizian, Waiss and Lelarge, Marc},
  journal={ICLR},
  year={2021}
}

@inproceedings{verma2019stability,
  title={Stability and generalization of graph convolutional neural networks},
  author={Verma, Saurabh and Zhang, Zhi-Li},
  booktitle={KDD},
  year={2019}
}

@inproceedings{hardt2016train,
  title={Train faster, generalize better: Stability of stochastic gradient descent},
  author={Hardt, Moritz and Recht, Benjamin and Singer, Yoram},
  booktitle={ICML},
  year={2016}
}

@article{hua2021noise,
  title={Noise stability regularization for improving BERT fine-tuning},
  author={Hua, Hang and Li, Xingjian and Dou, Dejing and Xu, Cheng-Zhong and Luo, Jiebo},
  journal={ACL},
  year={2021}
}

@article{hardt2021patterns,
  title={Patterns, predictions, and actions: A story about machine learning},
  author={Hardt, Moritz and Recht, Benjamin},
  journal={arXiv preprint arXiv:2102.05242},
  year={2021}
}

@inproceedings{gilmer2017neural,
  title={Neural message passing for quantum chemistry},
  author={Gilmer, Justin and Schoenholz, Samuel S and Riley, Patrick F and Vinyals, Oriol and Dahl, George E},
  booktitle={ICML},
  year={2017}
}

@article{errica2019fair,
  title={A fair comparison of graph neural networks for graph classification},
  author={Errica, Federico and Podda, Marco and Bacciu, Davide and Micheli, Alessio},
  journal={ICLR},
  year={2020}
}

@article{arora2021technical,
  title={Technical perspective: Why don't today's deep nets overfit to their training data?},
  author={Arora, Sanjeev},
  journal={Communications of the ACM},
  year={2021}
}

@article{neyshabur2018towards,
  title={Towards understanding the role of over-parametrization in generalization of neural networks},
  author={Neyshabur, Behnam and Li, Zhiyuan and Bhojanapalli, Srinadh and LeCun, Yann and Srebro, Nathan},
  journal={ICLR},
  year={2019}
}

@article{selsam2018learning,
  title={Learning a SAT solver from single-bit supervision},
  author={Selsam, Daniel and Lamm, Matthew and B{\"u}nz, Benedikt and Liang, Percy and de Moura, Leonardo and Dill, David L},
  journal={ICLR},
  year={2019}
}

@article{wu2020comprehensive,
  title={A comprehensive survey on graph neural networks},
  author={Wu, Zonghan and Pan, Shirui and Chen, Fengwen and Long, Guodong and Zhang, Chengqi and Philip, Yu},
  journal={IEEE transactions on neural networks and learning systems},
  year={2020}
}

@article{izmailov2018averaging,
  title={Averaging weights leads to wider optima and better generalization},
  author={Izmailov, Pavel and Podoprikhin, Dmitrii and Garipov, Timur and Vetrov, Dmitry and Wilson, Andrew Gordon},
  journal={UAI},
  year={2018}
}

@inproceedings{daskalakis2021complexity,
  title={The complexity of constrained min-max optimization},
  author={Daskalakis, Constantinos and Skoulakis, Stratis and Zampetakis, Manolis},
  booktitle={STOC},
  year={2021}
}

@inproceedings{zhang2018end,
  title={An end-to-end deep learning architecture for graph classification},
  author={Zhang, Muhan and Cui, Zhicheng and Neumann, Marion and Chen, Yixin},
  booktitle={AAAI},
  year={2018}
}

@article{ying2018hierarchical,
  title={Hierarchical graph representation learning with differentiable pooling},
  author={Ying, Zhitao and You, Jiaxuan and Morris, Christopher and Ren, Xiang and Hamilton, Will and Leskovec, Jure},
  journal={NeurIPS},
  year={2018}
}

@article{velivckovic2017graph,
  title={Graph attention networks},
  author={Veli{\v{c}}kovi{\'c}, Petar and Cucurull, Guillem and Casanova, Arantxa and Romero, Adriana and Lio, Pietro and Bengio, Yoshua},
  journal={ICLR},
  year={2018}
}

@book{mohri2018foundations,
  title={Foundations of machine learning},
  author={Mohri, Mehryar and Rostamizadeh, Afshin and Talwalkar, Ameet},
  year={2018},
  publisher={MIT press}
}

@article{bruna2013spectral,
  title={Spectral networks and locally connected networks on graphs},
  author={Bruna, Joan and Zaremba, Wojciech and Szlam, Arthur and LeCun, Yann},
  journal={ICLR},
  year={2014}
}

@article{esser2021learning,
  title={Learning theory can (sometimes) explain generalisation in graph neural networks},
  author={Esser, Pascal and Chennuru Vankadara, Leena and Ghoshdastidar, Debarghya},
  journal={NeurIPS},
  year={2021}
}
\end{refcontext}

\appendix
\noindent\textbf{Organization.} %
In Appendix \ref{app_proof}, we state the complete proofs for our results.
In Appendix \ref{sec_add_exp_setup}, we describe extra experiment details to complement our results.
\section{Proofs}\label{app_proof}

This section provides the complete proofs for our results in Section \ref{sec_theory}. First, we state several notations and facts needed in the proofs. Then we provide the proof of the Hessian-based generalization bound for MPNN, as stated in Lemma \ref{lemma_gen_error}.
After that, in Appendix \ref{app_mpgnn}, we provide the proof of Theorem \ref{thm_mpgnn}, a key step of which is the proof of Lemma \ref{lemma_trace_hess}.
Next, in Appendix \ref{app_lb}, we state the proof of the lower bound.
Lastly, in Appendix \ref{app_gin}, we will provide proof for the case of graph isomorphism networks.

First, we state several facts about graphs and provide a short proof of them.

\begin{fact}\label{fact_graph}
Let $G = (V, E)$ be an undirected graph.  Let $d_{G}$ be the maximum degree of $G$. %
\begin{enumerate}[leftmargin=0.25in]
    \item[a)] Let $A$ be the adjacency matrix of $G$. Then, the adjacency matrix satisfies: $\sqrt{d_G} \le \bignorms{A} \le d_G$. %
    \item[b)] The symmetric and degree-normalized adjacency matrix satisfies $\bignorms{D^{-1/2} A D^{-1/2}} \le 1$.
\end{enumerate}
\end{fact}
\begin{proof}
    Based on the definition of the spectral norm, we get
    \begin{align*}
        \bignorms{A} = \max_{\bignorm{x} = 1} x^{\top} A x = \max_{\bignorm{x} = 1} \sum_{(i,j)\in E} x_ix_j \le \max_{\bignorm{x} = 1} \sum_{(i,j)\in E} \frac{1}{2}(x_i^2 + x_j^2) \le d_G \sum_{i\in V} x_i^2 = d_G.
    \end{align*}
    Assume that node $i$ has the maximum degree $d_G$. Denote edges set $E_i =  \{(i,i_k)\}_{k=1}^{d_G}\subseteq E$. Let $x_i = \frac{1}{\sqrt{2}}$, $x_{i_k} = \frac{1}{\sqrt{2d_G}}$ for all $k = 1,\dots,d_G$. The rest entries of $x$ are equal to zero. Thus, $x$ is a normalized vector. Next, we have
    \begin{align*}
        \bignorms{A} = \max_{\bignorm{x} = 1} \sum_{(i,j)\in E} x_ix_j \ge \max_{\bignorm{x} = 1} 2 \sum_{(i,j)\in E_i} x_ix_j = 2 d_G \cdot \frac{1}{\sqrt{2}} \frac{1}{\sqrt{2d_G}} = \sqrt{d_G}.
    \end{align*}
    An example in which $\bignorms{P_G}$ gets close to $\sqrt {d_G}$ is the star graph.
    An example in which $\bignorms{P_G}$ gets close to $d_G$ is the complete graph.

    \smallskip
    Next, we focus on case b).
    From the definition of the spectral norm, we know
    \begin{align*}
        \bignorms{D^{-1/2} A D^{-1/2}} &= \max_{\bignorm{x} = 1} x^{\top} \bigbrace{D^{-1/2} A D^{-1/2}} x = \max_{\bignorm{x} = 1} \sum_{(i,j)\in E} \frac{x_ix_j}{\sqrt{d_id_j}} \\
        &\le \max_{\bignorm{x} = 1} \sum_{(i,j)\in E} \frac{x_i^2}{2d_i} + \frac{x_j^2}{2d_j} = \sum_{i\in V} x_i^2 = 1.
    \end{align*}
    During the middle of the above step, we used the Cauchy-Schwartz inequality. The proof of this result is now completed.
\end{proof}

\noindent\textbf{Notations:}
For two matrices $X$ and $Y$ that are both of dimension $d_1$ by $d_2$, the Hadamard product of $X$ and $Y$, denoted as $X \odot Y$, is equal to the entrywise product of $X$ and $Y$.

\subsection{Proof of our PAC-Bayesian bound (Lemma \ref{lemma_gen_error})}\label{proof_pac_bayes}

To be precise, we will restate the conditions required in Theorem \ref{thm_mpgnn} separately below. The conditions are exactly the same as stated in Section \ref{sec_theory}.

\begin{assumption}\label{ass_1}
    Assume that all the activation functions $\phi_i(\cdot),\rho_i(\cdot),\psi_i(\cdot)$ for any $1\leq i\leq l-1$ and the loss function $\ell(x, y)$ over $x$ are twice-differentiable and $\kappa_0$-Lipschitz. Their first-order derivatives are $\kappa_1$-Lipschitz and their second-order derivatives are $\kappa_2$-Lipschitz.
\end{assumption}

Based on the above assumption, we provide the precise statement for Taylor's expansion, used in equation \eqref{eq_perturb}.

\begin{proposition}\label{prop_taylor}
    In the setting of Theorem \ref{thm_mpgnn}, suppose each parameter in layer $i$ is perturbed by an independent noise drawn from $\cN(0,\sigma_i^2)$. 
    Let $\tilde{\ell}(f(X,G),y)$ be the perturbed loss function with noise perturbation injection vector $\cE$ on all parameters $\cW$ and $\cU$. 
    There exist some fixed value $C_1$ that do not grow with $N$ and $1/\delta$ such that %
    \begin{small}
        \begin{align*}
             \bigabs{\tilde \ell(f(X, G), y) - \ell(f(X, G), y) - \frac 1 2 \sum_{i=1}^l\sigma_{i}^2 \bigtr{\bH^{{(i)}}[\ell (f(X,G),y)]} } 
            \le  C_1 \sum_{i=1}^l \sigma_{i}^3.
        \end{align*}
    \end{small}%
\end{proposition}

\begin{proof}
    By Taylor's expansion, the following identity holds
    \begin{align*}
        \tilde \ell(f(X, G), y) - \ell(f(X, G), y)
        = \exarg{\cE}{\cE^{\top} \nabla \ell(f) + \frac 1 2 {\cE}^{\top} \bH[\ell(f)] {\cE} + R(\ell(f),\cE)}.
    \end{align*}
    where $R(\ell(f),\cE)$ is the rest of the first-order and the second-order terms.
    Since each entry in $\cE$ follows the normal distribution, we have
    $\exarg{\cE}{\cE^{\top} \nabla \ell(f)} = 0$. The Hessian term turns to
    \begin{align*}
        {\cE}^{\top} \bH[\ell(f)] {\cE} = \sum_{i=1}^l \sigma_i^2 \bigtr{\bH^{(i)}[\ell (f(X,G),y)]}.
    \end{align*}
    Since the readout layer is linear, by Proposition \ref{prop_taylor}, there exists a fixed constant $\bar C$ that does not grow with $N$ and $\delta^{-1}$ such that
     $   |R(\ell(f),\cE)| \le \bar C \bignorm{\cE}^3$.
    Based on \citet[Lemma 2]{jin2019short}, for any $x$ drawn from a normal distribution $\cN(0, \sigma^2)$, we have ${\ex{x^3}} \le 6\sigma^3$. Hence, we get
       $\ex{R(\ell(f),\cE)} \le C_1 \sum_{i=1}^l {\sigma_{i}^3}$,
    where $C_1 = \order{h^2 \bar C}$ is a fixed constant. Thus, we have finished the proof.
\end{proof}

Next, we state a Lipschitz-continuity upper bound of the network output at each layer.
This will be needed in the $\epsilon$-covering argument later in the proof of Theorem \ref{thm_mpgnn}.
To simplify the notation, we will abbreviate explicit constants that do not grow with $N$ and $1/\delta$ in the notation $\lesssim$; more specifically, we use $A(n) \lesssim B(n)$ to indicate that there exists a function $c$ that does not depend on $N$ and $1/\delta$ such that $A(n) \le c \cdot B(n)$ for large enough values of $n$.

\begin{proposition}\label{claim_hess}
    In the setting of Theorem \ref{thm_mpgnn}, for any $j=1,\dots,l-1$, the change in the Hessian of output of the $j$ layer network $H^{(j)}$ with respect to $W_i$ and $U_i$ under perturbation on $W$ and $U$ can be bounded as follows:
    \begin{align}
        \bignormFro{\bH_{\cW}^{(i)}[\tilde H^{(j)}] - \bH_{\cW}^{(i)}[H^{(j)}]}  &\lesssim \sum_{t=1}^j\Bigbrace{\bignorms{\Delta U^{(t)}} + \bignorms{\Delta W^{(t)}}}. \label{eq_claim_hess_1} \\
        \bignormFro{\bH_{\cU}^{(i)}[\tilde H^{(j)}] - \bH_{\cU}^{(i)}[H^{(j)}]} &\lesssim \sum_{t=1}^j\Bigbrace{\bignorms{\Delta U^{(t)}} + \bignorms{\Delta W^{(t)}}}. \label{eq_claim_hess_2}
    \end{align}
    Above, the notation $\bH_{\cW}^{(i)}[\tilde H^{(j)}]$ is the perturbation of the Hessian matrix of $H^{(j)}$ by $\Delta \cW$ and $\Delta \cU$, specific to the variables of $\cW$; likewise, $\bH_{\cU}^{(i)}[\tilde H^{(j)}]$ is the perturbation of the Hessian matrix specific to the variables of $\cU$.
\end{proposition}

The proof of Proposition \ref{claim_hess} will be deferred until Appendix \ref{proof_prop_lip}. Based on Propositions \ref{prop_taylor} and \ref{claim_hess}, now we are ready to present the proof of Lemma \ref{lemma_gen_error}.

\begin{proof}[Proof of Lemma \ref{lemma_gen_error}]
    First, we separate the gap of $\cL(f)$ and $\frac{1}{\beta}\hat{\cL}(f)$ into three parts:
    \begin{align*}
         &\cL(f) - \frac{1}{\beta} \hat{\cL}(f) 
        =  \underbrace{\exarg{(X,G,y) \sim \cD}{\ell(f(X,G),y)} - \exarg{(X,G,y) \sim \cD}{\tilde{\ell}(f(X,G),y)}}_{E_1} + \exarg{(X,G,y) \sim \cD}{\tilde{\ell}(f(X,G),y)} \\
        & - \frac{1}{\beta} \Bigbrace{\frac{1}{N} \sum_{i=1}^N \tilde{\ell}(f(X_i,G_i),y_i)} + \underbrace{\frac{1}{\beta} \Bigbrace{\frac{1}{N} \sum_{i=1}^N \tilde{\ell}(f(X_i,G_i),y_i)} - \frac{1}{\beta} \Bigbrace{\frac{1}{N} \sum_{i=1}^N \ell(f(X_i,G_i),y_i)}}_{E_2}.
    \end{align*}
    for any $\beta \in (0,1)$.
    Above, $\tilde{\ell}(f(X,G),y)$ is the perturbed loss from $\ell(f(X, G), y)$ with noise injections $\cE$ added to all the parameters in $\cW$ and $\cU$. By Taylor's expansion from Proposition \ref{prop_taylor}, we can bound the difference between $\tilde{\ell}(f(X,G), y)$ and $\ell(f(X, G)$ with the trace of the Hessian.
    Therefore
    \begin{align*}
        \cL(f) - \frac{1}{\beta} \hat{\cL}(f) \le \,& -\exarg{(X,G,y) \sim \cD}{\frac 1 2 \sum_{i=1}^l \sigma_i^2 \bigtr{\bH^{(i)}[\ell(f(X,G),y)]}}  + \sum_{i=1}^l C_1 \sigma_i^3 \tag{by Prop. \ref{prop_taylor} for $E_1$}\\
        & + \Bigg( \exarg{(X,G,y) \sim \cD}{\tilde{\ell}(f(X,G),y)} - \frac{1}{\beta} \Bigbrace{\frac{1}{N} \sum_{i=1}^N \tilde{\ell}(f(X_i,G_i),y_i)} \Bigg) \\
        & + \frac{1}{2\beta}\sum_{i=1}^l \sigma_i^2  \Bigbrace{\frac{1}{N} \sum_{j=1}^N \bigtr{\bH^{(i)}[\ell(f(X_j,G_j),y_j)]}} + \frac{1}{\beta}\sum_{i=1}^l C_1 \sigma_{i}^3. \tag{by Prop. \ref{prop_taylor} for $E_2$}
    \end{align*}%
    By rearranging the above equation, we get the following:
    {\small\begin{align*}
         \cL(f) - \frac{1}{\beta} \hat{\cL}(f) 
        &\le  \frac 1 2 \sum_{i=1}^l \sigma_i^2 \Bigg( \underbrace{\frac{1}{N} \sum_{j=1}^N \bigtr{\bH^{(i)}[\ell(f(X_j,G_j),y_j)]} - \exarg{(X,G,y) \sim \cD}{\bigtr{\bH^{(i)}[\ell(f(X,G),y)]}}}_{E_3} \Bigg) \\
        & + \frac 1 2 \Bigbrace{\frac{1}{\beta} - 1} \underbrace{\sum_{i=1}^l \frac{\sigma_i^2}{N} { \sum_{j=1}^N \bigtr{\bH^{(i)}[\ell(f(X_j,G_j),y_j)]}}}_{E_4} \\
        & + \Bigbrace{1 + \frac{1}{\beta}} C_1 \sum_{i=1}^l \sigma_i^3 + \underbrace{\exarg{(X,G,y) \sim \cD}{\tilde{\ell}(f(X,G),y)} - {\frac{1}{\beta N} \sum_{i=1}^N \tilde{\ell}(f(X_i,G_i),y_i)}}_{E_5}.
    \end{align*}}%
    Based on Proposition \ref{claim_hess}, the Hessian operator $\bH^{(i)}$ is Lipschitz-continuous for some parameter that does not depend on $N$ and $1/\delta$, for any $i = 1,2\dots,l$.
    Therefore, from \citet[Lemma 2.4]{ju2022robust}, there exist some fixed values $C_2$, $C_3$ that do not grow with $N$ and $1/\delta$, such that with probability at least $1 - \delta$ over the randomness of the training set. Therefore, the matrix inside the trace of $E_3$ satisfies
    \begin{align}
        \bignormFro{\frac 1 N \sum_{j=1}^N\bH^{(i)}[\ell(f(X_j,G_j), y_j)] - \exarg{(X,G,y)\sim\cD}{\bH^{(i)}[\ell(f(X,G), y)]}} \le \frac{C_2\sqrt{\log (C_3 N/\delta)}}{\sqrt N}, \label{eq_trace_bound}
    \end{align}
    for any $i = 1,\dots,l.$
    Thus, by the Cauchy-Schwartz inequality, $E_3$ is less than $\sqrt{2h^2}$ times the RHS of equation \eqref{eq_trace_bound}.
    Suppose the loss function $\ell(f(X,G),y)$ lies in a bounded range $[0, B]$ given any $(X,G,y) \sim \cD$. By the PAC-Bayes bound of \citet[Theorem 2]{mcallester2013pac} (see also \citet{guedj2019primer}), we choose $\cU$ as a prior distribution and $\cW + \cU$ as a posterior distribution. For any $\beta \in (0,1)$ and $\delta \in [0,1)$, with probability at least $1 - \delta$, $E_5$ satisfies:
    \begin{align}
        \exarg{(X,G,y) \sim \cD}{\tilde{\ell}(f(X,G),y)} - {\frac{1}{\beta N} \sum_{i=1}^N \tilde{\ell}(f(X_i,G_i),y_i)}
        \le& \frac{B}{2\beta(1-\beta)N}\Bigbrace{\sum_{i=1}^l {\frac{\bignormFro{W^{(i)}}^2 + \bignormFro{U^{(i)}}^2}{2\sigma_{i}^2}} + \log \frac{1}{\delta}} \nonumber \\
        \le& \frac{B}{2\beta(1-\beta)N} \Bigbrace{\sum_{i=1}^l \frac{s_i^2 r_i^2}{\sigma_i^2} + \log\frac 1 {\delta}}. \label{eq_pac_bayes}
    \end{align}
    The above is because ${\cW}$ and ${\cU}$ are inside the hypothesis set $\cH$.
    For any $i = 1,\dots,l$, let
    \[ \alpha_i = \max\limits_{(X, G, y)\sim\cD} \bigtr{\bH^{(i)}[\ell(f(X,G),y)]}. \]
    Lastly, we use $\sigma_i^2 \alpha_i$ above to upper bound $E_4$.
    Combined with equations \eqref{eq_trace_bound} and \eqref{eq_pac_bayes}, with probability at least $1 - 2\delta$, we get
    \begin{align*}
        \cL(f) - \frac{1}{\beta} \hat{\cL}(f) \le \,&  
           \frac{C_2\sqrt{2 h^2 \log (C_3 N/\delta)}}{\sqrt N}  \sum_{i=1}^l \sigma_i^2 + \Bigbrace{1 + \frac{1}{\beta}} C_1 \sum_{i=1}^l \sigma_i^3 \\
          & + {\frac 1 2 \Bigbrace{\frac{1}{\beta} - 1} \sum_{i=1}^l \alpha_i \sigma_i^2 
          + \frac{B}{2\beta(1-\beta)N}\Bigg( \sum_{i=1}^l {\frac{s_i^2 r_i^2}{\sigma_{i}^2}} + \log \frac{1}{\delta} \Bigg)}.
    \end{align*}
    Next, we will select $\sigma_i$ to minimize the last line above.
    One can verify that this is achieved when
    \begin{align*}
        \sigma_i^2 = \frac{s_i r_i}{1-\beta}\sqrt{\frac{B}{\alpha_i N}}, \text{ for every } i = 1,2,\dots,l.
    \end{align*}
    With this setting of the noise variance, the gap between $\cL(f)$ and $\hat{\cL}(f)/ {\beta}$ becomes:
    \begin{align*}
        & \cL(f) - \frac{1}{\beta} \hat{\cL}(f) \\
        \le& \frac{1}{\beta} \sum_{i=1}^l \sqrt{\frac{B \alpha_i s_i^2 r_i^2}{N}} 
        + \frac{C_2\sqrt{2 h^2 \log (C_3 N/\delta)}}{\sqrt N} \sum_{i=1}^L \sigma_{i}^2 + \Bigbrace{1 + \frac{1}{\beta}} C_1\sum_{i=1}^l \sigma_i^3 + \frac{C}{2\beta(1-\beta)N}\log \frac{1}{\delta}.
    \end{align*}
    Let $\beta$ be a fixed value close to $1$ and independent of $N$ and $\delta^{-1}$; let $\epsilon = (1 - \beta)/\beta$. We get
    \begin{align*}
        &\cL(f) \le  \, (1 + \epsilon) \hat{\cL}(f) + (1 + \epsilon) \sum_{i=1}^l \sqrt{\frac{B \alpha_i r_i^2 s_i^2}{N}} + \xi, \text{ where } \\
        &\xi = \frac{C_2\sqrt{2 h^2 \log (C_3 N/\delta)}}{\sqrt N}  \sum_{i=1}^L \sigma_{i}^2 + \Bigbrace{1 + \frac{1}{\beta}} C_1 \sum_{i=1}^l {\sigma_{i}^3} + \frac{C}{2\beta(1-\beta)N}\log \frac{1}{\delta}.
    \end{align*}
    Notice that $\xi$ is of order $\order{N^{-3/4}\ + \log(\delta^{-1}) N^{-1}} \le \order{\log(\delta^{-1}) / N^{3/4}}$.
    Therefore, we have finished the proof of equation \eqref{eq_main_1}.
\end{proof}

\subsubsection{Proof of Proposition \ref{claim_hess}}\label{proof_prop_lip}

For any $j = 1,2,\dots,l$, let $\tilde H^{(j)}$ be the perturbed network output after layer $j$, with perturbations given by $\Delta \cW$ and $\Delta \cU$.
We show the following Lipschitz-continuity property for $H^{(j)}$.

\begin{claim}\label{claim_perturb}
    Suppose that Assumption \ref{ass_1} holds. For any $j=1,\dots,l-1$, the change in the output of the $j$ layer network $H^{(j)}$ with perturbation added to $\cW$ and $\cU$ can be bounded as follows:
    \begin{align}\label{eq_claim_perturb}
        \bignormFro{\tilde H^{(j)} - H^{(j)}}\lesssim \sum_{t=1}^j\Bigbrace{\bignorms{\Delta U^{(t)}}+\bignorms{\Delta W^{(t)}}}.
    \end{align}
\end{claim}

\begin{proof}
    We will prove using induction with respect to $j$. If $j = 1$, we have
    \begin{align*}
        &\bignormFro{\phi_1\Bigbrace{X \bigbrace{U^{(1)} + \Delta U^{(1)}} + \rho_1\bigbrace{P_{_G} \psi_1\bigbrace{X}} \bigbrace{W^{(1)} + \Delta W^{(1)}}} - \phi_1\Bigbrace{X U^{(1)} + \rho_1\bigbrace{P_{_G} \psi_1\bigbrace{X}} W^{(1)}}} \\
        \leq & \kappa_0\bignormFro{X \Delta U^{(1)} + \rho_1\bigbrace{P_{_G} \psi_1\bigbrace{X}} \Delta W^{(1)}} \lesssim \bignorms{\Delta U^{(1)}}+\bignorms{\Delta W^{(1)}}.
    \end{align*}
    Hence, we know that equation \eqref{eq_claim_perturb} will be correct when $j = 1$. Assuming that equation \eqref{eq_claim_perturb} is correct for any $j \geq 1$, the perturbation of layer $j+1$'s network output $H^{(j+1)}$ is less than
    {\begin{align*} 
        & \bignormFro{\tilde H^{(j+1)} - H^{(j+1)}} \\ 
        \le & \, \kappa_0\bignormFro{X \Delta U^{(j+1)} + \rho_{j+1}\bigbrace{P_{_G} \psi_{j+1}\bigbrace{\tilde H^{(j)}}} \bigbrace{W^{(j+1)} + \Delta W^{(j+1)}} - \rho_{j+1}\bigbrace{P_{_G} \psi_{j+1}\bigbrace{H^{(j)}}} W^{(j+1)}} \\
        \lesssim & \bignorms{\Delta U^{(j+1)}}+\bignorms{\Delta W^{(j+1)}} + \bignormFro{\tilde H^{(j)} - H^{(j)}}.
    \end{align*}}%
    Thus, we have finished the proof of the induction step.
\end{proof}

Next, for any $i$ and $j$, let $\frac{\partial \tilde H^{(j)}}{\partial W^{(i)}}$ be the perturbation of the partial derivative of $H^{(j)}$ with perturbations given by $\Delta \cW$ and $\Delta \cU$.

\begin{claim}\label{claim_grad}
    Suppose that Assumption \ref{ass_1} holds. For any $j=1,\dots,l-1$, the change in the Jacobian of the $j$-th layer's output $H^{(j)}$ with respect to $W^{(i)}$ and $U^{(i)}$ satisfies:
    \begin{align}
        \bignormFro{\frac{{{\partial \tilde H}}^{(j)}}{\partial W^{(i)}} - \frac{\partial H^{(j)}}{\partial W^{(i)}}}  &\lesssim \sum_{t=1}^j\Bigbrace{\bignorms{\Delta U^{(t)}} + \bignorms{\Delta W^{(t)}}}. \label{eq_claim_grad_1} \\
        \bignormFro{\frac{{\partial\tilde H}^{(j)}}{\partial  U^{(i)}}  - \frac{\partial H^{(j)}}{\partial U^{(i)}}} &\lesssim \sum_{t=1}^j\Bigbrace{\bignorms{\Delta U^{(t)}} + \bignorms{\Delta W^{(t)}}}. \label{eq_claim_grad_2}
    \end{align}
\end{claim}

\begin{proof}
    We will consider a fixed $i = 1,\dots,l-1$ and take induction over $j = i, \dots, l-1$.
    We focus on the proof of equation \eqref{eq_claim_grad_1}, while the proof of equation \eqref{eq_claim_grad_2} will be similar.
    To simplify the derivation, we use two notations for brevity.
    Let
    \begin{align*}
        F_j = P_{_G} \psi_{j}\bigbrace{H^{(j-1)}} W^{(j)} \text{ and }
        E_j = X U^{(j)} + \rho_{j}\bigbrace{F_j}.
    \end{align*}
    First, we consider the base case when $j = i$.
    By the chain rule, we have:
    \begin{align*}
        \bignormFro{\frac{\partial \tilde H^{(i)}}{\partial W^{(i)}}  - \frac{\partial H^{(i)}}{\partial W^{(i)}}} 
        =&  \bignormFro{\phi'_i\bigbrace{\tilde E_i} 
        \odot \frac{\partial \tilde E_i} {\partial W^{(i)}} - \phi'_i\bigbrace{E_i} \odot \frac{\partial E_i} {\partial W^{(i)}} } \\
        \lesssim&  \bignormFro{\phi'_i\bigbrace{\tilde E_i}  - \phi'_i\bigbrace{E_i} } + \bignormFro{\frac{\partial\tilde E_i} {\partial W^{(i)}} - \frac{\partial E_i} {\partial W^{(i)}} }. 
    \end{align*}
    From Claim \ref{claim_perturb}, we know
    \begin{align*}
        \bignormFro{\phi'_i\bigbrace{\tilde E_i}  - \phi'_i\bigbrace{E_i}} \leq \kappa_1 \bignormFro{\tilde E_i - E_i} \lesssim \bignorms{\Delta W^{(i)}} + \bignorms{\Delta U^{(i)}}.
    \end{align*}
    By the chain rule again, we get:
    \begin{align*}
        \bignormFro{\frac{\partial \tilde E_i} {\partial W^{(i)}} - \frac{\partial E_i} {\partial W^{(i)}} } 
        \lesssim \bignormFro{\rho'_i\bigbrace{\tilde F_i} - \rho'_i\bigbrace{F_i}} + \bignormFro{\frac{\partial \tilde F_i} {\partial W^{(i)}} - \frac{\partial F_i} {\partial W^{(i)}}} 
        \lesssim {\bignorms{\Delta W^{(i)}} + \bignorms{\Delta U^{(i)}}} \tag{by Claim \ref{claim_perturb} again}.
    \end{align*}
    Hence, we know that equation \eqref{eq_claim_grad_1} will be correct when $j=i$. Assuming that equation \eqref{eq_claim_grad_1} will be correct for any $j$ up to $j \ge i$, we have
    \begin{align*}
        & \bignormFro{\frac{\partial \tilde H^{(j+1)}}{\partial W^{(i)}} - \frac{\partial H^{(j+1)}}{\partial W^{(i)}}} \\
        \lesssim & \bignormFro{\phi'_{j+1}\bigbrace{\tilde E_{j+1}} - \phi'_{j+1}\bigbrace{E_{j+1}} } + \bignormFro{\frac{\partial \tilde E_{j+1}} {\partial W^{(i)}} - \frac{\partial E_{j+1}} {\partial W^{(i)}} } \\
        \lesssim & \sum_{t=1}^{j+1}\Bigbrace{\bignorms{\Delta U^{(t)}} + \bignorms{\Delta W^{(t)}}} + \bignormFro{\rho'_{j+1}\bigbrace{\tilde F_{j+1}} - \rho'_{j+1}\bigbrace{F_{j+1}}} + \bignormFro{\frac{\partial \tilde F_{j+1}} {\partial W^{(i)}} - \frac{\partial F_{j+1}} {W^{(i)}}} \\
        \lesssim & \sum_{t=1}^{j+1}\Bigbrace{\bignorms{\Delta U^{(t)}}+\bignorms{\Delta W^{(t)}}} + \bignormFro{\psi'_{j+1}\bigbrace{\tilde H^{(j)}} - \psi'_{j+1}\bigbrace{H^{(j)}}} +  \bignormFro{\frac{\partial \tilde H^{(j)}} {\partial W^{(i)}} - \frac{\partial H^{(j)}} {\partial W^{(i)}}} \\
        \lesssim & \sum_{t=1}^{j+1}\Bigbrace{\bignorms{\Delta U^{(t)}}+\bignorms{\Delta W^{(t)}}}. \tag{by Claim \ref{claim_perturb} and the induction step}
    \end{align*}
    The above steps all use Claim \ref{claim_perturb}.
    The last step additionally uses the induction hypothesis.
    From repeatedly applying the above beginning with $j = i$ along with the base case of equation \eqref{eq_claim_grad_1}, we conclude that equation \eqref{eq_claim_grad_1} holds.

    \smallskip
    Next, we consider the base case for equation \eqref{eq_claim_grad_2}. For the base case $j=i$, from the chain rule, by Claim \ref{claim_perturb}, we get:
    \begin{align*}
        \bignormFro{\frac{\partial \tilde H^{(i)}}{\partial U^{(i)}} - \frac{\partial H^{(i)}}{\partial U^{(i)}}} 
        \lesssim \bignormFro{\phi'_i\bigbrace{\tilde E_i} - \phi'_i\bigbrace{E_i} } + \bignormFro{\frac{\partial \tilde E_i} {\partial U^{(i)}} - \frac{\partial E_i} {\partial U^{(i)}} }
        \lesssim {\bignorms{\Delta W^{(i)}} + \bignorms{\Delta U^{(i)}}}.
    \end{align*}
    Hence, we know that equation \eqref{eq_claim_grad_2} will be correct when $j=i$. Assuming that equation \eqref{eq_claim_grad_2} will be correct for any $j$ up to $j \ge i$, we have
    \begin{align*}
        & \bignormFro{\frac{\partial \tilde H^{(j+1)}}{\partial U^{(i)}} - \frac{\partial H^{(j+1)}}{\partial U^{(i)}}} 
        \lesssim  \bignormFro{\phi'_{j+1}\bigbrace{\tilde E_{j+1}} - \phi'_{j+1}\bigbrace{E_{j+1}}} + \bignormFro{\frac{\partial \tilde E_{j+1}} {\partial U^{(i)}} - \frac{\partial E_{j+1}} {\partial U^{(i)}} } \\
        &\lesssim  \sum_{t=1}^{j+1}\Bigbrace{\bignorms{\Delta U^{(t)}} + \bignorms{\Delta W^{(t)}}} + \bignormFro{\rho'_{j+1}\bigbrace{\tilde F_{j+1}} - \rho'_{j+1}\bigbrace{F_{j+1}}} + \bignormFro{\frac{\partial \tilde F_{j+1}} {\partial U^{(i)}} - \frac{\partial F_{j+1}} {\partial U^{(i)}}}  \\
        &\lesssim  \sum_{t=1}^{j+1}\Bigbrace{\bignorms{\Delta U^{(t)}}+\bignorms{\Delta W^{(t)}}} + \bignormFro{\psi'_{j+1}\bigbrace{\tilde H^{(j)}} - \psi'_{j+1}\bigbrace{H^{(j)}}} +  \bignormFro{\frac{\partial \tilde H^{(j)}} {\partial U^{(i)}} - \frac{\partial H^{(j)}} {\partial U^{(i)}}} \\
        &\lesssim  \sum_{t=1}^{j+1}\Bigbrace{\bignorms{\Delta U^{(t)}}+\bignorms{\Delta W^{(t)}}}. \tag{by Claim \ref{claim_perturb} and the induction step}
    \end{align*}
    The second and third steps are based on Claim \ref{claim_perturb}.
    From repeatedly applying the above beginning with $j = i$ along with the base case of equation \eqref{eq_claim_grad_2}, we conclude that equation \eqref{eq_claim_grad_2} holds.
    The proof of claim \ref{claim_grad} is complete.
\end{proof}

\begin{proof}[Proof of Proposition \ref{claim_hess}]
    We will consider a fixed $i = 1,\dots,l-1$ and take induction over $j = i, \dots, l-1$.
    We focus on the proof of equation \eqref{eq_claim_hess_1}, while the proof of equation \eqref{eq_claim_hess_2} will be similar.
    To simplify the derivation, we use two notations for brevity.
    Let
    \begin{align*}
        F_j = P_{_G} \psi_{j}\bigbrace{H^{(j-1)}} W^{(j)} \text{ and }
        E_j = X U^{(j)} + \rho_{j}\bigbrace{F_j}.
    \end{align*}
    First, we consider the base case when $j = i$.
    By the chain rule, we have:
    We use the chain rule to get:
    \begin{align*}
        \frac {\partial^2 H^{(i)}} {\partial \bigbrace{W^{(i)}_{p,q}}^2}
        = \phi''_{i}(E_i) \odot \frac{\partial E_i}{\partial W_{p,q}^{(i)}} \odot \frac{\partial E_i}{\partial W_{p,q}^{(i)}}
        + \phi'_{i}(E_i) \odot \rho''_{i}(F_i) \odot {\frac{\partial F_i}{\partial W_{p,q}^{(i)}}} \odot {\frac{\partial F_i}{\partial W^{(i)}_{p,q}}}.
    \end{align*}
    Hence, the Frobenius norm of the Hessian of $H^{(i)}$ with respect to $W_i$ under perturbation on $W$ and $U$ turns to
    \begin{align*}
        \bignormFro{\bH_{\cW}^{(i)} [\tilde H^{(i)}] - \bH_{\cW}^{(i)} [H^{(i)}]} 
        & \lesssim \bignormFro{\phi''_i\bigbrace{\tilde E_i} - \phi''_i\bigbrace{E_i} } 
        + \bignormFro{\frac{\partial \tilde E_i} {\partial W^{(i)}} - \frac{\partial E_i} {\partial W^{(i)}} } + \bignormFro{\phi'_i\bigbrace{\tilde E_i} - \phi'_i\bigbrace{E_i} } \\
        & + \bignormFro{\rho''_i\bigbrace{\tilde F_i} - \rho''_i\bigbrace{F_i} } + \bignormFro{\frac{\partial \tilde F_i} {\partial W^{(i)}} - \frac{\partial F_i} {\partial W^{(i)}} }. 
    \end{align*}
    From Claim \ref{claim_perturb}, we know
    \begin{align*}
        \bignormFro{\phi''_i\bigbrace{\tilde E_i} - \phi''_i\bigbrace{E_i}} &\le \kappa_2 \bignormFro{\tilde E_i - E_i} \lesssim \bignorms{\Delta W^{(i)}} + \bignorms{\Delta U^{(i)}}, \\
        \bignormFro{\phi'_i\bigbrace{\tilde E_i} - \phi'_i\bigbrace{E_i} } &\le \kappa_1 \bignormFro{\tilde E_i - E_i} \lesssim \bignorms{\Delta W^{(i)}} + \bignorms{\Delta U^{(i)}},  \\
        \bignormFro{\rho''_i\bigbrace{\tilde F_i} - \rho''_i\bigbrace{F_i} } &\le \kappa_2 \bignormFro{\tilde F_i - F_i} \lesssim \bignorms{\Delta W^{(i)}} + \bignorms{\Delta U^{(i)}}. 
    \end{align*}
    From Claim \ref{claim_grad}, we have
    \begin{align*}
        \bignormFro{\frac{\partial \tilde E_i} {\partial W^{(i)}} - \frac{\partial E_i} {\partial W^{(i)}} } & \lesssim \bignorms{\Delta W^{(i)}} + \bignorms{\Delta U^{(i)}}, \\
        \bignormFro{\frac{\partial \tilde F_i} {\partial W^{(i)}} - \frac{\partial F_i} {\partial W^{(i)}} } & \lesssim \bignorms{\Delta W^{(i)}} + \bignorms{\Delta U^{(i)}}. 
    \end{align*}
    Hence, we know that equation \eqref{eq_claim_hess_1} will be correct when $j=i$. Assuming that equation \eqref{eq_claim_hess_1} will be correct for any $j$ up to $j \ge i$, we can get the following steps, by taking another derivative of the first-order derivative, we can get the following steps:
    \begin{small}
        \begin{align*}
            \frac {\partial^2 H^{(j+1)}} {\partial \bigbrace{W^{(i)}_{p,q}}^2}
            =& \, \phi''_{j+1}(E_{j+1}) \odot \frac{\partial E_{j+1}}{\partial W_{p,q}^{(i)}} \odot \frac{\partial E_{j+1}}{\partial W_{p,q}^{(i)}}
            + \phi'_{j+1}(E_{j+1}) \odot \rho''_{j+1}(F_{j+1}) \odot \frac{\partial F_{j+1}}{\partial W_{p,q}^{(i)}} \odot \frac{\partial F_{j+1}}{\partial W_{p,q}^{(i)}} \\
            & + \phi'_{j+1}(E_{j+1}) \odot \rho'_{j+1}(F_{j+1}) \odot {P_{_G} \Bigbrace{\psi''_{j+1}(H^{(j)}) \odot \frac{\partial H^{(j)}}{\partial W_{p,q}^{(i)}} \odot \frac{\partial H^{(j)}}{\partial W_{p,q}^{(i)}} + \psi'_{j+1}(H^{(j)}) \odot \frac{\partial^2 H^{(j)}}{\partial\bigbrace{W_{p,q}^{(i)}}^2}}W^{(j+1)}}. 
        \end{align*} 
    \end{small}%
    Thus, the Frobenius norm of the Hessian of $H^{(j+1)}$ with respect to $W^{(i)}$ satisfies:
    \begin{small}
        \begin{align*}
            & \bignormFro{\bH_{\cW}^{(i)} [\tilde H^{(j+1)}] - \bH_{\cW}^{(i)} [H^{(j+1)}]} 
            \lesssim \underbrace{\bignormFro{\phi''_{j+1}\bigbrace{\tilde E_{j+1}} - \phi''_{j+1}\bigbrace{E_{j+1}} }}_{A_1} 
            + \underbrace{\bignormFro{\frac{\partial \tilde E_{j+1}} {\partial W^{(i)}} - \frac{\partial E_{j+1}} {\partial W^{(i)}} }}_{B_1} \\
            + & \underbrace{\bignormFro{\phi'_{j+1}\bigbrace{\tilde E_{j+1}} - \phi'_{j+1}\bigbrace{E_{j+1}} }}_{A_2} 
            + \underbrace{\bignormFro{\rho''_{j+1}\bigbrace{\tilde F_{j+1}} - \rho''_{j+1}\bigbrace{F_{j+1}} }}_{A_3} \\
            +& \underbrace{\bignormFro{\frac{\partial \tilde F_{j+1}} {\partial W^{(i)}} - \frac{\partial F_{j+1}} {\partial W^{(i)}} }}_{B_2} 
            + \underbrace{\bignormFro{\rho'_{j+1}\bigbrace{\tilde F_{j+1}} - \rho'_{j+1}\bigbrace{F_{j+1}} }}_{A_4} \\
            + & \underbrace{\bignormFro{\psi''_{j+1}\bigbrace{\tilde H^{(j)}} - \psi''_{j+1}\bigbrace{H^{(j)}} }}_{A_5} 
            + \underbrace{\bignormFro{ \frac{\partial \tilde H^{(j)}} {\partial W^{(i)}} - \frac{\partial H^{(j)}} {\partial W^{(i)}} }}_{B_3} \\ 
            + & \underbrace{\bignormFro{\psi'_{j+1}\bigbrace{\tilde H^{(j)}} - \psi'_{j+1}\bigbrace{H^{(j)}}}}_{A_6} 
            + \underbrace{\bignormFro{\bH_{\cW}^{(i)} [\tilde H^{(j)}] - \bH_{\cW}^{(i)} [H^{(j)}]}}_{C_1}.
        \end{align*}
    \end{small}%
    Similarly, by Claim \ref{claim_perturb}, we get 
    \begin{align*}
        A_i \lesssim \sum_{t=1}^{j+1}\Bigbrace{\bignorms{\Delta W^{(t)}} + \bignorms{\Delta U^{(t)}}}, \text{ for } 1\le i\le 6.
    \end{align*}
    By Claim \ref{claim_grad}, we get
    \begin{align*}
        B_i \lesssim \sum_{t=1}^{j+1}\Bigbrace{\bignorms{\Delta W^{(t)}} + \bignorms{\Delta U^{(t)}}}, \text{ for } 1\le i\le 3.
    \end{align*}
    By the induction hypothesis, $C_1$ is also less than the above quantity. %
    From repeatedly applying the above beginning with $j = i$ along with the base case of equation \eqref{eq_claim_hess_1}, we conclude that equation \eqref{eq_claim_hess_1} holds.

    \smallskip
    Next, we consider the base case for equation \eqref{eq_claim_hess_2}. For the base case $j=i$, from the chain rule, we get:
    \begin{align*}
        \bignormFro{\bH_{\cU}^{(i)} [\tilde H^{(i)}] - \bH_{\cU}^{(i)} [H^{(i)}]}
        \lesssim & \bignormFro{\phi''_i\bigbrace{\tilde E_i} - \phi''_i\bigbrace{E_i} } + \bignormFro{\frac{\partial \tilde E_i} {\partial U^{(i)}} - \frac{\partial E_i} {\partial U^{(i)}} } \\
        \lesssim & \kappa_2 \bignormFro{\tilde E_i - E_i} + {\bignorms{\Delta W^{(i)}} +  \bignorms{\Delta U^{(i)}}} \tag{by Claim \ref{claim_grad}} \\
        \lesssim & {\bignorms{\Delta W^{(i)}} +  \bignorms{\Delta U^{(i)}}} \tag{by Claim \ref{claim_perturb}}.
    \end{align*}
    Hence, we know that equation \eqref{eq_claim_hess_2} will be correct when $j=i$. Assuming that equation \eqref{eq_claim_hess_2} will be correct for any $j$ up to $j \ge i$, we obtain the induction step
    similar to the proof of equation \eqref{eq_claim_hess_1}, by Claim \ref{claim_perturb}, Claim \ref{claim_grad}, and the induction hypothesis, we conclude that equation \eqref{eq_claim_hess_2} holds.
\end{proof}

\subsection{Proof for message passing graph neural networks}\label{app_mpgnn}

Next, we present proof for message-passing graph neural networks.
First, in Appendix \ref{proof_trace}, we derive the trace bound, which separates the trace of the Hessian matrix into each entry of the weight matrices.
Then in Appendix \ref{proof_first} and \ref{proof_second}, we provide bounds on the first-order and second-order derivatives of the Hessian matrix.
Last, in Appendix \ref{proof_theorem}, building on these results, we finish the proof of Theorem \ref{thm_mpgnn}.

\subsubsection{Proof of Lemma \ref{lemma_trace_hess}}\label{proof_trace}

\begin{proof}[Proof of Lemma \ref{lemma_trace_hess}]
    Notice that $f(X, G) =  H^{(l)}$.
    Recall that in each layer for $1\le i\le l-1$, there are two weight matrices, a $d_{i-1}$ by $d_i$ matrix denoted as $W^{(i)}$, and a $d_0$ by $U^{(i)}$ matrix denoted as $U^{(i)}$.
    To deal with the trace of the Hessian $\bH^{(i)}$, we first notice that there are two parts in the trace:
    \begin{align*}
        \bigabs{\tr\big[\bH^{(i)}[\ell(H^{(l)}, y)]\big]} 
        \le \underbrace{\bigabs{\sum_{p=1}^{d_{i-1}}\sum_{q=1}^{d_i}\frac{\partial^2 \ell(H^{(l)}, y)}{\partial \big(W^{(i)}_{p,q}\big)^2}}}_{T_1}
        + \underbrace{\bigabs{\sum_{p=1}^{d_0}\sum_{q=1}^{d_i} \frac{\partial^2 \ell(H^{(l)}, y)}{\partial\bigbrace{U_{p,q}^{(i)}}^2}}}_{T_2}.
    \end{align*}
    We can inspect $T_1$ and $T_2$ in the above step separately.
    First, we expand out the second-order derivatives in $T_1$.
    This will involve two terms by the chain rule.
    \begin{align}
        T_1 =\,& \bigabs{\sum_{p=1}^{d_{i-1}}\sum_{q=1}^{d_i} \Bigg\langle\frac{\partial\ell(H^{(l)},y)}{\partial H^{(l)}},\frac{\partial^2 H^{(l)}}{\partial \big(W^{(i)}_{p,q}\big)^2}\Bigg\rangle} 
        + \bigabs{ \sum_{p=1}^{d_{i-1}}\sum_{q=1}^{d_i} \Bigg\langle  \frac{\partial^2 \ell(H^{(l)}, y)}{\partial \big(H^{(l)}\big)^2} \frac{\partial H^{(l)}}{\partial W^{(i)}_{p,q}}, \frac{\partial H^{(l)}}{\partial W^{(i)}_{p,q}} \Bigg\rangle} \nonumber \\
        \leq\,& \sum_{p=1}^{d_{i-1}}\sum_{q=1}^{d_i}\bignorm{\frac{\partial\ell(H^{(l)}, y)}{\partial H^{(l)}}}\bignorm{\frac{\partial^2 H^{(l)}}{\partial \big(W^{(i)}_{p,q}\big)^2}} + \sum_{p=1}^{d_{i-1}}\sum_{q=1}^{d_i} \bignorms{\frac{\partial^2 \ell(H^{(l)}, y)}{\partial \big(H^{(l)}\big)^2}}\bignorm{\frac{\partial H^{(l)}}{\partial W^{(i)}_{p,q}}}^2 \nonumber \\
        \leq\,& \kappa_0 \sqrt{k} \sum_{p=1}^{d_{i-1}}\sum_{q=1}^{d_i} \bignorm{\frac{\partial^2 H^{(l)}}{\partial \big(W^{(i)}_{p,q}\big)^2}} + \kappa_1  k \sum_{p=1}^{d_{i-1}}\sum_{q=1}^{d_i} \bignorm{\frac{\partial H^{(l)}}{\partial W^{(i)}_{p,q}}}^2. \label{eq_lemma_loss}
    \end{align}
    The last step is because $\ell(\cdot)$ is $\kappa_0$-Lipschitz continuous and $\ell'(\cdot)$ is $\kappa_1$-Lipschitz continuous, under Assumption \ref{ass_1}.
    Thus, the Euclidean norm of ${\frac{\partial \ell(H^{(l)}, y)}{\partial H^{(l)}}} $ is at most $\kappa_0 \sqrt k$, since $H^{(l)}$ is a $k$-dimensional vector.
    Recall from step \eqref{eq_mpgnn_readout} that $H^{(l)} = \frac{1}{n}\bm{1}_n^{\top} H^{(l-1)}W^{(l)}$. Hence, we have
    \begin{align}
        \bignorm{\frac{\partial H^{(l)}}{\partial W^{(i)}_{p,q}}} 
        &= \bignorm{\frac{1}{n} \bm{1}_n^{\top} \frac{\partial H^{(l-1)}}{\partial W^{(i)}_{p,q}}W^{(l)}} \nonumber \\
        &\leq \bignorm{\frac{1}{n} \bm{1}_n^{\top}} \bignorms{\frac{\partial H^{(l-1)}}{\partial W^{(i)}_{p,q}}W^{(l)}}\leq \frac{1}{\sqrt{n}}\bignorms{\frac{\partial H^{(l-1)}}{\partial W^{(i)}_{p,q}}} \bignorms{W^{(l)}}. \label{eq_lemma_readout_1}
    \end{align}
    In a similar vein, the Euclidean norm of ${\frac{\partial^2 \ell(H^{(l)}, y)}{\partial (H^{(l)})^2}}$ is at most $\kappa_1 k$, since the second-order derivatives become a $k$ by $k$ matrix. Then, we get
    \begin{align}
        \bignorm{\frac{\partial^2 H^{(l)}}{\partial \bigbrace{W^{(i)}_{p,q}}^2}} &= \bignorm{\frac{1}{n} \bm{1}_n^{\top} \frac{\partial^2 H^{(l-1)}}{\partial \bigbrace{W^{(i)}_{p,q}}^2}W^{(l)}} \nonumber \\
        &\leq \bignorm{\frac{1}{n} \bm{1}_n^{\top}} \bignorms{\frac{\partial^2 H^{(l-1)}}{\partial \bigbrace{W^{(i)}_{p,q}}^2}W^{(l)}}\leq \frac{1}{\sqrt{n}}\bignorms{\frac{\partial^2 H^{(l-1)}}{\partial \bigbrace{W^{(i)}_{p,q}}^2}} \bignorms{W^{(l)}}. \label{eq_lemma_readout_2}
    \end{align}
    After substituting equations \eqref{eq_lemma_readout_1} and \eqref{eq_lemma_readout_2} into equation \eqref{eq_lemma_loss}, we get: 
    \begin{align*}
        T_1 &\leq \frac{\kappa_0 \sqrt{k}}{\sqrt{n}} \bignorms{W^{(l)}} \sum_{p=1}^{d_{i-1}}\sum_{q=1}^{d_i} \bignorms{\frac{\partial^2 H^{(l-1)}}{\partial \big(W^{(i)}_{p,q}\big)^2}} + \frac {\kappa_1 k} n \bignorms{W^{(l)}}^2 \sum_{p=1}^{d_{i-1}}\sum_{q=1}^{d_i} \bignorms{\frac{\partial H^{(l-1)}}{\partial W^{(i)}_{p,q}}}^2. \\
        &\leq \frac{\kappa_0 \sqrt{k}}{\sqrt{n}} \bignorms{W^{(l)}} \sum_{p=1}^{d_{i-1}}\sum_{q=1}^{d_i} \bignormFro{\frac{\partial^2 H^{(l-1)}}{\partial \big(W^{(i)}_{p,q}\big)^2}} +  \frac {\kappa_1 k} n \bignorms{W^{(l)}}^2 \sum_{p=1}^{d_{i-1}}\sum_{q=1}^{d_i}\bignormFro{\frac{\partial H^{(l-1)}}{\partial W_{p,q}^{(i)}}}^2.
    \end{align*}
    The proof for the case of $T_2$ concerning $U^{(i)}$ follows the same steps as above.
    Without belaboring all the details, one can get that
    \begin{align}
        T_2 \le  \frac{\kappa_0 \sqrt k}{\sqrt n} \bignorms{W^{(l)}} \sum_{p=1}^{d_{0}}\sum_{q=1}^{d_i} \bignormFro{\frac{\partial^2 H^{(l-1)}}{\partial\bigbrace{U_{p,q}^{(i)}}^2}}
        + \frac{\kappa_1 k}{n} \bignorms{W^{(l)}}^2 \sum_{p=1}^{d_0}\sum_{q=1}^{d_{i}}\bignormFro{\frac{\partial H^{(l-1)}}{\partial U_{p,q}^{(i)}}}^2.
    \end{align}
    This completes the proof of Lemma \ref{lemma_trace_hess}.
\end{proof}

\subsubsection{Dealing with first-order derivatives}\label{proof_first}

Based on Lemma \ref{lemma_trace_hess}, the analysis involves two parts,
one on the first-order derivatives of $H^{(j)}$ for all layers $j$,
and the other on the second-order derivatives of $H^{(j)}$ for all layers $j$.

\begin{proposition}\label{prop_first_mpgnn}
    In the setting of Theorem \ref{thm_mpgnn}, 
    the first-order derivative of $H^{(j)}$ with respect to $W^{(i)}$ and $U^{(i)}$ satisfies the following, for any $i = 1,\dots,l-1$ and $j\ge i$:
    \begin{align}
        \bignormFro{\frac{\partial H^{(j)}} {\partial W^{(i)}}}
        \le\,& \kappa_0^{3(j-i+1)} \sqrt{d_i} \bignorms{P_{_G}}^{j-i+1} \bignormFro{H^{(i-1)}} \prod_{t=i+1}^j \bignorms{W^{(t)}}, \label{eq_deri_W} \\
        \bignormFro{\frac{\partial H^{(j)}} {\partial U^{(i)}}} \le\,& \kappa_0^{3(j-i)+1} \sqrt{d_i} \bignorms{P_{_G}}^{j-i+1} \bignormFro{X} \prod_{t=i+1}^j \bignorms{W^{(t)}}. \label{eq_deri_U}
    \end{align}
\end{proposition}

\begin{proof}
    We will consider a fixed $i = 1,\dots,l-1$ and take induction over $j = i, \dots, l-1$.
    We focus on the proof of equation \eqref{eq_deri_W}, while the proof of equation \eqref{eq_deri_U} will be similar.
    First, we consider the base case when $j = i$.
    Let $W^{(i)}_{p,q}$ be the $(p,q)$-th entry of $W^{(i)}$, for any valid indices $p$ and $q$.
    Recall that $\phi_i(\cdot)$ is $\kappa_0$-Lipschitz continuous from Assumption \ref{ass_1}, for any $i = 1,\dots, l-1$.
    Therefore,
    \begin{align}
        \bignormMax{\phi'_i(x)} \le \kappa_0,~~
        \bignormMax{\psi'_i(x)} \le \kappa_0,~~\text{and}~~
        \bignormMax{\rho'_i(x)} \le \kappa_0. \label{eq_deri}
    \end{align} 
    For each $(p, q)$-entry of $W^{(i)}$, by the chain rule, we have:
    \begin{align}
        \bignormFro{\frac {\partial H^{(i)}} {\partial W^{(i)}_{p,q}}}
        &= \bignormFro{\phi'_i\bigbrace{X U^{(i)} + \rho_i\bigbrace{P_{_G} \psi_i(H^{(i-1)}) W^{(i)}}}
        \odot \frac{\partial \bigbrace{X U^{(i)} + \rho_i\bigbrace{P_{_G} \psi_i(H^{(i-1)}) W^{(i)}}}} {\partial W^{(i)}_{p,q}}} \label{eq_W_prime} \\
        &\le \kappa_0\bignormFro{\frac{\partial \rho_i\bigbrace{P_{_G} \psi_i(H^{(i-1)}) W^{(i)}}} {\partial W^{(i)}_{p,q}}} \tag{by equation \eqref{eq_deri}} \\
        &= \kappa_0\bignormFro{\rho'_i\bigbrace{P_{_G} \psi_i(H^{(i-1)}) W^{(i)}} \odot \frac{\partial\bigbrace{P_{_G} \psi_i(H^{(i-1)}) W^{(i)}}} {W^{(i)}_{p,q}}} \nonumber \\
        &\le \kappa_0^2\bignormFro{\frac{\partial \bigbrace{P_{_G} \psi_i(H^{(i-1)}) W^{(i)}}} {\partial W^{(i)}_{p,q}}}. \tag{again by equation \eqref{eq_deri}}
    \end{align}
    Notice that only the $q$-th column of the derivative $P_{_G} \psi_i(H^{(i-1)}) W^{(i)}$ is nonzero, which is equal to the $p$'th column of $P_{_G} \psi_i(H^{(i-1)})$. Thus, the Jacobian of $H^{(i)}$ over $W^{(i)}$ satisfies:
    \begin{align}
        \bignormFro{\frac{\partial H^{(i)}} {\partial W^{(i)}}}
        =& \sqrt{\sum_{p=1}^{d_{i-1}} \sum_{q=1}^{d_i} \bignormFro{\frac {\partial H^{(i)}} {\partial W_{p,q}^{(i)}}}^2} \nonumber \\
        \le& \kappa_0^2 \sqrt{\sum_{p=1}^{d_{i-1}} \sum_{q=1}^{d_i} \bignormFro{\frac {\partial\bigbrace{P_{_G} \psi_i(H^{(i-1)}) W^{(i)}}} {W_{p,q}^{(i)}}}^2}
        = \kappa_0^2 \sqrt{d_i} \bignormFro{P_{_G} \psi_i(H^{(i-1)})}. \label{eq_mpgnn_base}
    \end{align}
    Therefore, the above equation \eqref{eq_mpgnn_base} implies that equation \eqref{eq_deri_W} holds in the base case.
    Next, we consider the induction step from layer $j$ to layer $j+1$. The derivative of $H^{(j+1)}$ with respect to $W^{(i)}_{p,q}$ satisfies:
    \begin{align*}
        & \bignormFro{\frac {\partial H^{(j+1)}} {\partial W^{(i)}_{p,q}}} \\
        =& \bignormFro{\phi'_{j+1}\bigbrace{X U^{(j+1)} + \rho_{j+1}\bigbrace{P_{_G} \psi_{j+1}(H^{(j)}) W^{(j+1)}}}
        \odot \frac{\partial \bigbrace{X U^{(j+1)} + \rho_{j+1}\bigbrace{P_{_G} \psi_{j+1}(H^{(j)}) W^{(j+1)}}}} {\partial W^{(i)}_{p,q}}} \\
        \le& \kappa_0\bignormFro{\frac{\partial\rho_{j+1}\bigbrace{P_{_G} \psi_{j+1}(H^{(j)}) W^{(j+1)}}} {\partial W^{(i)}_{p,q}}} \tag{by equation \eqref{eq_deri}} \\
        \le& \kappa_0\bignormFro{\rho'_{j+1}\bigbrace{P_{_G} \psi_{j+1}(H^{(j)}) W^{(j+1)}} \odot \frac{\partial\bigbrace{P_{_G} \psi_{j+1}(H^{(j)}) W^{(j+1)}}} {\partial W^{(i)}_{p,q}}} \\
        \le& \kappa_0^2\bignormFro{P_{_G} \frac{\partial \psi_{j+1}(H^{(j)}) } {\partial W^{(i)}_{p,q}}W^{(j+1)}} \tag{again by equation \eqref{eq_deri}}
    \end{align*}
    By applying equation \eqref{eq_deri} w.r.t. $\psi'_{j+1}$, The above is less than:
    \begin{align*} 
         \kappa_0^2\bignorms{P_{_G}}\bignormFro{\psi_{j+1}'(H^{(j)}) \odot \frac{\partial H^{(j)} } {\partial W^{(i)}_{p,q}}}\bignorms{W^{(j+1)}} 
        \le \kappa_0^3\bignorms{P_{_G}}\bignorms{W^{(j+1)}}\bignormFro{\frac{\partial H^{(j)} } {\partial W^{(i)}_{p,q}}}. 
    \end{align*}
    Hence, the Jacobian of $H^{(j+1)}$ with respect to $W^{(i)}$ satisfies:
    \begin{align*}
        \bignormFro{\frac {\partial H^{(j+1)}} {\partial W^{(i)}}} \le \kappa_0^3\bignorms{P_{_G}}\bignorms{W^{(j+1)}}\bignormFro{\frac{\partial H^{(j)} } {\partial W^{(i)}}}.
    \end{align*} 
    From repeatedly applying the above beginning with $j = i$ along with the base case of equation \eqref{eq_mpgnn_base}, we conclude that equation \eqref{eq_deri_W} holds.

    \smallskip
    Next, we consider the base case for equation \eqref{eq_deri_U}. For each $(p, q)$-th entry of $U^{(i)}$, from the chain rule we get:
    \begin{align*}
        \bignormFro{\frac{\partial H^{(i)}} {\partial U^{(i)}_{p, q}}}
        &= \bignormFro{\phi'_i\Bigbrace{X U^{(i)} + \rho_i\bigbrace{P_{_G} \psi_i(H^{(i-1)}) W^{(i)}}} \odot \frac{\partial\Bigbrace{X U^{(i)} + \rho_i\bigbrace{P_{_G} \psi_i(H^{(i-1)}) W^{(i)}}}} {\partial U^{(i)}_{p,q}}} \\
        &\le \kappa_0 \bignormFro{\frac {\partial (X U^{(i)})} {\partial U^{(i)}_{p, q}}}. \tag{by equation \eqref{eq_deri}}
    \end{align*}
    Therefore, by summing over $p=1,\dots,d_0$ and $q=1,\dots,d_i$, we get:
    \begin{align}
        \bignormFro{\frac {\partial H^{(i)}} {\partial U^{(i)}}}
        &= \sqrt{\sum_{p=1}^{d_0} \sum_{q=1}^{d_i} \bignormFro{\frac {\partial H^{(i)}} {\partial U^{(i)}_{p, q}}}^2} \nonumber \\
        &\le \kappa_0 \sqrt{\sum_{p=1}^{d_0} \sum_{q=1}^{d_i} \bignormFro{\frac {\partial(X U^{(i)})} {\partial U^{(i)}_{p,q}}}^2}
        = \kappa_0 \sqrt{d_i} \bignormFro{X}. \label{eq_base_U}
    \end{align}
    Going from layer $i$ to layer $j+1$, the derivative of $H^{(j+1)}$ with respect to $U^{(i)}_{p,q}$ satisfies:
    \begin{align*}
        & \bignormFro{\frac{\partial H^{(j+1)}} {\partial U^{(i)}_{p, q}}} \\
        &= \bignormFro{\phi'_{j+1}\bigbrace{X U^{(j+1)} + \rho_{j+1}\bigbrace{P_{_G} \psi_{j+1}(H^{(j)}) W^{(j+1)}}} \odot \frac{\partial\bigbrace{X U^{(j+1)} + \rho_{j+1}\bigbrace{P_{_G} \psi_{j+1}(H^{(j)}) W^{(j+1)}}}} {\partial U^{(i)}_{p,q}}} \\
        &\le \kappa_0 \bignormFro{\frac {\partial \rho_{j+1}\bigbrace{P_{_G} \psi_{j+1}(H^{(j)}) W^{(j+1)}}} {\partial U^{(i)}_{p, q}}} \tag{by equation \eqref{eq_deri} w.r.t. $\phi'_{j+1}$}\\
        &\le \kappa_0^3\bignorms{P_{_G}}\bignorms{W^{(j+1)}}\bignormFro{\frac{\partial H^{(j)} } {\partial U^{(i)}_{p,q}}}. \tag{by equation \eqref{eq_deri} w.r.t. $\rho'_{j+1}, \psi'_{j+1}$}
    \end{align*}
    Hence, the Jacobian of $H^{(j+1)}$ with respect to $U^{(i)}$ satisfies:
    \begin{align*}
        \bignormFro{\frac {\partial H^{(j+1)}} {\partial U^{(i)}}} \le \kappa_0^3\bignorms{P_{_G}}\bignorms{W^{(j+1)}}\bignormFro{\frac{\partial H^{(j)} } {\partial U^{(i)}}}.
    \end{align*}
    By repeatedly applying the above step beginning with the base case of equation \eqref{eq_base_U}, we have proved that equation \eqref{eq_deri_U} holds.
    The proof of Proposition \ref{prop_first_mpgnn} is complete.
\end{proof}

\subsubsection{Deal with second-order derivatives}\label{proof_second}

In the second part towards showing Theorem \ref{thm_mpgnn} for MPNNs, we look at second-order derivatives of the embeddings.
This will appear later when we deal with the trace of the Hessian.
A fact that we will use throughout the proof is
\begin{align}
    \bignormMax{\phi''_i(x)} \le \kappa_1,~~
    \bignormMax{\psi''_i(x)} \le \kappa_1,~~\text{and}~~
    \bignormMax{\rho''_i(x)} \le \kappa_1, \label{eq_deri_sec}
\end{align}
for any $x$ and $i = 1,\dots,l-1$.
This is because $\phi'_i, \psi_i', $ and $\rho_i'$ are all $\kappa_1$-Lipschitz continuous from Assumption \ref{ass_1}.

\begin{proposition}\label{prop_second_mpgnn}
    In the setting of Theorem \ref{thm_mpgnn}, 
    the second-order derivative of $H^{(l)}$ with respect to $W^{(i)}$ and $U^{(i)}$ satisfies the following, for any $i = 1,\dots,l-1$ and any $j= i,\dots,l-1$:
    {\begin{align}
        \sum_{p=1}^{d_{i-1}}\sum_{q=1}^{d_i} \bignormFro{\frac{\partial^2 H^{(j)}}{\bigbrace{W_{p,q}^{(i)}}^2}}
        &\le C_{i,j} \kappa_1 d_i \max(\bignorms{P_{_G}}^{j-i+2}, \bignorms{P_{_G}}^{2(j-i+1)})
        \bignormFro{H^{(i-1)}}^2
        \prod_{t=i+1}^{j} s_t^2, \label{eq_second_W} \\
        \sum_{p=1}^{d_{i-1}}\sum_{q=1}^{d_i} \bignormFro{\frac{\partial^2 H^{(j)}}{\bigbrace{U_{p,q}^{(i)}}^2}}
        &\le \hat{C}_{i,j} \kappa_1 d_i \max(\bignorms{P_{_G}}^{j-i}, \bignorms{P_{_G}}^{2(j-i)}) \bignormFro{X}^2 \prod_{t=i+1}^j s_t^2, \label{eq_second_U}
    \end{align}}%
    where $C_{i,j}$
    $$ C_{i,j} = \left\{
        \begin{aligned}
            &\kappa_0^{3(j-i+1)}\frac{\kappa_0^{3(j-i)+2} - 1}{\kappa_0 - 1}, & \kappa_0 \neq 1, \\
            &3(j - i) + 2, & \kappa_0 = 1,
        \end{aligned}
    \right.
    $$
    and $\hat{C}_{i,j}$
    $$ \hat{C}_{i,j} = \left\{
        \begin{aligned}
            &\kappa_0^{3(j-i)}\frac{\kappa_0^{3(j-i)+1} - 1}{\kappa_0 - 1}, & \kappa_0 \neq 1, \\
            &3(j - i) + 1, & \kappa_0 = 1.
        \end{aligned}
    \right.
    $$
    are fixed constants that depend on the Lipschitz-continuity of the activation mappings.
\end{proposition}

\begin{proof} 
    First, we will consider equation \eqref{eq_second_W}.
    To simplify the derivation, we introduce two notations for brevity.
    Let
    \begin{align*}
        F_j = P_{_G} \psi_{j}\bigbrace{H^{(j-1)}} W^{(j)} \text{ and }
        E_j = X U^{(j)} + \rho_{j}\bigbrace{F_j}.
    \end{align*}
    In the base case when $j = i$, from the first-order derivative in equation \eqref{eq_W_prime}, we use the chain rule to get:
    \begin{align}
        \frac {\partial^2 H^{(i)}} {\partial \bigbrace{W^{(i)}_{p,q}}^2}
        =& \phi''_{i}(E_i) \odot \frac{\partial E_i}{\partial W_{p,q}^{(i)}} \odot \frac{\partial E_i}{\partial W_{p,q}^{(i)}}
        + \phi'_{i}(E_i) \odot \rho''_{i}(F_i) \odot {\frac{\partial F_i}{\partial W_{p,q}^{(i)}}} \odot {\frac{\partial F_i}{\partial W^{(i)}_{p,q}}}. \label{eq_Hi_chain}
    \end{align}
    Using equation \eqref{eq_deri_sec}, the maximum entries of $\phi''_{i}(\cdot), \rho''_{i}(\cdot)$ are at most $\kappa_1$.
    Using equation \eqref{eq_deri}, the maximum entry of
    $\phi'_{i}(\cdot)$ is at most $\kappa_0$.
    Notice that the derivative of $E_i$ can be reduced to the derivative of $F_i$ as follows:
    \begin{align}
        \bignormFro{\frac{\partial E_i}{\partial W_{p,q}^{(i)}}}^2
        = \bignormFro{\rho'_{i}(F_i) \odot \frac{\partial F_1}{\partial W_{p,q}^{(i)}}}^2 
        \le \kappa_0^2 \bignormFro{\frac{\partial F_i}{\partial W_{p,q}^{(i)}}}^2. \label{eq_second_p2}
    \end{align}
    Therefore, based on the conditions for first- and second-order derivatives (cf. \eqref{eq_deri} and \eqref{eq_deri_sec}), the Frobenius norm of the above equation \eqref{eq_Hi_chain} is at most:
    \begin{align*}
        \bignormFro{\frac{\partial^2 H^{(i)}} {\partial\bigbrace{W_{p,q}^{(i)}}^2}}
        \le \kappa_1 \bignormFro{\frac{\partial E_i} {\partial W_{p,q}^{(i)}}}^2 + \kappa_0\kappa_1 \bignormFro{\frac{\partial F_i}{\partial W_{p,q}^{(i)}}}^2
        \le (\kappa_0 + 1)\kappa_0\kappa_1 \bignormFro{\frac{\partial F_i}{\partial W_{p,q}^{(i)}}}^2.
    \end{align*}
    Notice that the derivative of $F_i$ with respect to $W^{(i)}_{p,q}$ is nonzero only in the $q$-th column of $F_i$, and is equal to the $p$-th column of $P_{_G} g_{i}(H^{(i-1)})$. Therefore, by summing over $p = 1,\dots,d_{i-1}$ and $q = 1,\dots,d_i$, we get:
    \begin{align*}
        \sum_{p=1}^{d_{i-1}}\sum_{q=1}^{d_i} \bignormFro{\frac{\partial F_{i}}{\partial\bigbrace{W_{p,q}^{(i)}}^2}}^2
        \le d_i \bignormFro{P_{_G} \psi_i(H^{(i-1)})}^2.
    \end{align*}
    Therefore, we have derived the base case when $j = i$ as:
    \begin{align}
        \bignormFro{\frac{\partial^2 H^{(i)}}{\partial\bigbrace{W_{p,q}^{(i)}}^2}}
        \le (\kappa_0 + 1) \kappa_0^3 \kappa_1 d_i \bignorms{P_{_G}}^2 \bignormFro{H^{(i-1)}}^2. \label{eq_second_base_W}
    \end{align}
    Next, we consider the induction step from layer $j$ to layer $j+1$. This step is similar to the base case but also differs since $H^{(j)}$ is now dependent on $W^{(i)}$.
    Recall that the second-order derivatives satisfy equation \eqref{eq_deri_sec}.
    Based on the Lipschitz-continuity conditions, the Frobenius norm of the second-order derivatives satisfies:
    \begin{align}
        & \bignormFro{\frac{\partial^2 H^{(j+1)}}{\partial\bigbrace{W_{p,q}^{(i)}}^2}} \\
        &\le \kappa_1 \bignormFro{\frac{\partial E_{j+1}}{\partial W_{p,q}^{(i)}}}^2
        + \kappa_0\kappa_1 \bignormFro{\frac{\partial F_{j+1}}{\partial W_{p,q}^{(i)}}}^2
        + \kappa_0^2 \bignorms{P_{_G}} \bignorms{W^{(j+1)}} \Bigbrace{\kappa_1 \bignormFro{\frac{\partial H^{(j)}}{\partial W_{p,q}^{(i)}}}^2 + \kappa_0 \bignormFro{\frac {\partial^2 H^{(j)}} {\partial \bigbrace{W_{p,q}^{(i)}}^2}}} \nonumber \\
        &\le (\kappa_0 + 1)\kappa_0\kappa_1 \bignormFro{\frac {\partial F_{j+1}} {\partial W_{p,q}^{(i)}}}^2
        + \kappa_0^2\bignorms{P_{_G}} \bignorms{W^{(j+1)}} \Bigbrace{\kappa_1\bignormFro{\frac{H^{(j)}}{\partial W_{p,q}^{(i)}}}^2
        + \kappa_0\bignormFro{\frac {\partial^2 H^{(j)}} {\partial\bigbrace{W_{p,q}^{(i)}}^2}}}. \label{eq_second_p1}
    \end{align}
    The last step follows similarly as equation \eqref{eq_second_p2}.
    For the derivative of $F_{j+1}$, using the chain rule, we get:
    \begin{align*}
        \bignormFro{\frac{\partial F_{j+1}}{\partial W_{p,q}^{(i)}}}^2
        &= \bignormFro{P_{_G} \frac {\partial \psi_{j+1}(H^{(j)})} {\partial W_{p,q}^{(i)}} W^{(j+1)}}^2 \\
        &\le \bignorms{P_{_G}}^2 \bignorms{W^{(j+1)}}^2 \bignormFro{\frac{\partial \psi_{j+1}(H^{(j)})}{\partial W_{p,q}^{(i)}}}^2 \\
        &\le \bignorms{P_{_G}}^2 \bignorms{W^{(j+1)}}^2 \bignormFro{\psi'_{j+1}(H^{(j)}) \odot \frac{\partial H^{(j)}}{\partial W_{p,q}^{(i)}}}^2 \\
        &\le \kappa_0^2 \bignorms{P_{_G}}^2 \bignorms{W^{(j+1)}}^2 \bignormFro{\frac {\partial H^{(j)}} {\partial W_{p,q}^{(i)}}}^2.
    \end{align*}
    Therefore, combining the above with equations \eqref{eq_second_p1} together, we get the following result:
    \begin{align*}
        \bignormFro{\frac{\partial^2 H^{(j+1)}}{\partial\bigbrace{W_{p,q}^{(i)}}^2}}
        \le\,& \Bigbrace{(\kappa_0+1)\kappa_0^3\kappa_1 \bignorms{P_{_G}}^2 \bignorms{W^{(j+1)}}^2 + \kappa_0^2 \kappa_1 \bignorms{P_{_G}} \bignorms{W^{(j+1)}}} \bignormFro{\frac {\partial H^{(j)}} {\partial W_{p,q}^{(i)}}}^2 \nonumber \\
        &+ \kappa_0^3 \bignorms{P_{_G}} \bignorms{W^{(j+1)}} \bignormFro{\frac {\partial^2 H^{(j)}} {\partial \bigbrace{W^{(i)}_{p,q}}^2}} \nonumber \\
        \le& \max\Bigbrace{\bignorms{P_{_G}}, \bignorms{P_{_G}}^2}s_{j+1}^2
        \Bigbrace{(\kappa_0^2+\kappa_0+1)\kappa_0^2\kappa_1\bignormFro{\frac {\partial H^{(j)}} {\partial W_{p,q}^{(i)}}}^2
        + \kappa_0^3 \bignormFro{\frac {\partial^2 H^{(j)}} {\partial\bigbrace{W_{p,q}^{(i)}}^2}}}. %
    \end{align*}
    Based on equation \eqref{eq_deri_W} of Proposition \eqref{prop_first_mpgnn}, the first-order derivative of $H^{(j)}$ satisfies:
    \begin{align}
        \sum_{p=1}^{d_{i-1}} \sum_{q=1}^{d_i} \bignormFro{\frac{\partial H^{(j)}} {\partial W_{p,q}^{(i)}}}^2 \le \kappa_0^{6(j-i+1)}{d_i} \bignorms{P_{_G}}^{2(j-i+1)} \bignorm{H^{(i-1)}}^2 \prod_{t=i+1}^j s_t^2. \label{eq_mpgnn_second_base}
    \end{align}
    Applying equation \eqref{eq_mpgnn_second_base} to
    the above (and summing over $p = 1,\dots, d_{i-1}$ and $q = 1,\dots,d_i$) forms the induction step for showing equation \eqref{eq_second_W}:
    {\small\begin{align*}
         \sum_{p=1}^{d_{i-1}} \sum_{q=1}^{d_i} \bignormFro{\frac {\partial^2 H^{(j+1)}} {\partial \bigbrace{W_{p,q}^{(i)}}^2}}
        \le\,& \frac{\kappa_0^3 - 1}{\kappa_0 - 1} \kappa_0^{6(j-i+1)+2} \kappa_1 d_i \max\Bigbrace{\bignorms{P_{_G}}^{2(j-i)+3}, \bignorms{P_{_G}}^{2(j-i)+4}} \bignormFro{H^{(i-1)}}^2 \prod_{t=i+1}^{j+1} s_t^2 \\
        &+ \kappa_0^3\max\bigbrace{\bignorms{P_{_G}}, \bignorms{P_{_G}}^2} s_{j+1}^2
        \sum_{p=1}^{d_{i-1}}\sum_{q=1}^{d_i}\bignormFro{\frac{\partial^2 H^{(j)}}{\partial\bigbrace{W_{p,q}^{(i)}}^2}}.
    \end{align*}}%
    By repeatedly applying the induction step along with the base case in equation \eqref{eq_second_base_W}, we have shown that equation \eqref{eq_second_W} holds:
    \begin{align}
        \sum_{p=1}^{d_{i-1}}\sum_{q=1}^{d_i} \bignormFro{\frac{\partial^2 H^{(j)}}{\bigbrace{W_{p,q}^{(i)}}^2}}
        \le C_{i,j} \kappa_1 d_i \max(\bignorms{P_{_G}}^{j-i+2}, \bignorms{P_{_G}}^{2(j-i+1)})
        \bignormFro{H^{(i-1)}}^2
        \prod_{t=i+1}^{j} s_t^2, \label{eq_conclude_W}
    \end{align}
    where $C_{i,j}$ satisfies the following equation:
    $$ C_{i,j} = \left\{
        \begin{aligned}
            &\kappa_0^{3(j-i+1)}\frac{\kappa_0^{3(j-i)+2} - 1}{\kappa_0 - 1}, & \kappa_0 \neq 1, \\
            &3(j - i) + 2, & \kappa_0 = 1.
        \end{aligned}
    \right.
    $$
    \smallskip
    In the second part of the proof, we consider equation \eqref{eq_second_U} similar to the first part.
    However, the analysis will be significantly simpler.
    We first consider the base case. Similar to equation \eqref{eq_Hi_chain}, the second-order derivative of $H^{(i)}$ over $W^{(i)}_{p,q}$ satisfies, for any $p = 1,\dots, d_{0}$ and $q = 1,\dots, d_i$:
    \begin{align*}
        \bignormFro{\frac{\partial^2 H^{(i)}}{\partial\bigbrace{U_{p,q}^{(i)}}^2}}
        = \bignorm{\phi_i''(E_i) \odot \frac{\partial E_i}{\partial U^{(i)}_{p,q}} \odot \frac{\partial E_i}{\partial U^{(i)}_{p,q}}}
        \le \kappa_1 \bignormFro{\frac{\partial(X U^{(i)})} {\partial U^{(i)}_{p,q}}}^2.
    \end{align*}
    Therefore, by summing up the above over all $p$ and $q$, we get the base case result:
    \begin{align}
        \sum_{p=1}^{d_{i-1}}\sum_{q=1}^{d_i} \bignormFro{\frac{\partial^2 H^{(i)}}{\partial\bigbrace{U_{p,q}^{(i)}}^2}}
        \le \kappa_1 d_i \bignormFro{X}^2.
    \end{align}
    Next, we consider the induction step from layer $j$ to layer $j+1$.
    This step follows the same analysis until equation \eqref{eq_conclude_W}, from which we can similarly derive that:
    \begin{align}
        \sum_{p=1}^{d_{i-1}}\sum_{q=1}^{d_i} \bignormFro{\frac{\partial^2 H^{(j)}}{\bigbrace{U_{p,q}^{(i)}}^2}}
        \le \hat{C}_{i,j} \kappa_1 d_i \max(\bignorms{P_{_G}}^{j-i}, \bignorms{P_{_G}}^{2(j-i)}) \bignormFro{X}^2 \prod_{t=i+1}^j s_t^2.
    \end{align}
    where $\hat{C}_{i,j}$ satisfies the following equation:
    $$ \hat{C}_{i,j} = \left\{
        \begin{aligned}
            &\kappa_0^{3(j-i)}\frac{\kappa_0^{3(j-i)+1} - 1}{\kappa_0 - 1}, & \kappa_0 \neq 1, \\
            &3(j - i) + 1, & \kappa_0 = 1.
        \end{aligned}
    \right.
    $$
\end{proof}

\subsubsection{Proof of Theorem \ref{thm_mpgnn}}\label{proof_theorem}

Based on Propositions \ref{prop_first_mpgnn} and \ref{prop_second_mpgnn}, we are ready to present the proof of Theorem \ref{thm_mpgnn} for message passing GNNs.
First, we will apply the bounds on the derivatives back in Lemma \ref{lemma_trace_hess}.
After getting the trace of the Hessians, we then use the PAC-Bayes bound from Lemma \ref{lemma_gen_error} to complete the proof.

\begin{proof}[Proof of Theorem \ref{thm_mpgnn}]
    By applying equations \eqref{eq_deri_W} and \eqref{eq_second_W}  into Lemma \ref{lemma_trace_hess}'s result, we get that the trace of $\bH^{(l)}$ with respect to $W^{(i)}$ is less than:
    {\begin{align} 
         &\left( \frac{\kappa_0 \sqrt{k}}{\sqrt n} C_{i,l-1} \kappa_1 d_i \max\big(\bignorms{P_{_G}}^{l-i+1}, \bignorms{P_{_G}}^{2(l-i)}\big)  
         + \frac{\kappa_1 k}{n} \kappa_0^{6(l-i)} \kappa_1 d_i \bignorms{P_{_G}}^{2(l-i)} \right) \bignormFro{H^{(i-1)}}^2 \prod_{t=i+1}^l s_t^2 \nonumber \\
        &\le (\kappa_0 C_{i,l-1} + \kappa_0^{6(l-i)}) \sqrt{\frac{{k}}{{n}}} \kappa_1 d_i \max\Bigbrace{\bignorms{P_{_G}}^{l-i+1}, \bignorms{P_{_G}}^{2(l-i)}}
        \bignormFro{H^{(i-1)}}^2 \prod_{t=i+1}^l s_t^2,\label{eq_mpgnn_pr2}
    \end{align}}%
    for any $i=1,2,\cdots,l-1$. Here we have
    $$ \kappa_0 C_{i,l-1} + \kappa_0^{6(l-i)} = \left\{
        \begin{aligned}
            &\kappa_0^{3(l-i)+1}\frac{\kappa_0^{3(l-i)} - 1}{\kappa_0 - 1}, & \kappa_0 \neq 1, \\
            &3(l - i) - 1, & \kappa_0 = 1.
        \end{aligned}
    \right.
    $$
    It remains to consider the Frobenius norm of $H^{(i-1)}$.
    Notice that this satisfies the following:
    \begin{align*}
        &\bignormFro{H^{(i-1)}}
        \le \kappa_0 \bignormFro{X U^{(i-1)} + \rho_{i-1}(P_{_G} \psi_{i-1}(H^{(i-2)}))W^{(i-1)}} \\
        &\le \kappa_0 \bignorms{U^{(i-1)}}\bignormFro{X}
        + \kappa_0^3 \bignorms{P_{_G}} \bignorms{W^{(i-1)}} \bignormFro{H^{(i-2)}}
        \le \kappa_0 s_i \bignormFro{X} + \kappa_0^3 \bignorms{P_{_G}} s_i \bignormFro{H^{(i-2)}}.
    \end{align*}
    By induction over $i$ for the above step, we get that the Frobenius norm of $H^{(i-1)}$ must be less than:
    \begin{align}
        \bigbrace{\kappa_0^{3(i-1)} + \sum_{j=0}^{i-2} \kappa_0^{3j+1}} {\sqrt{k}} \max_{(X, G, y) \sim \cD}\bignorms{X} \max \bigbrace{1,\bignorms{P_{_G}}^{i-1}} \prod_{j=1}^{i-1} s_j. \label{eq_mpgnn_pr1}
    \end{align}
    By applying the above \eqref{eq_mpgnn_pr1} back in \eqref{eq_mpgnn_pr2}, we have shown that the trace of $\bH^{(l)}$ with respect to $W^{(i)}$ is less than:
    \begin{align}\label{eq_mpgnn_pr3}
        C' \max_{(X, G, y) \sim \cD}\bignorms{X}^2 \kappa_1 d_i k \max(1, \norm{P_{_G}}^{2(l-1)}) \prod_{t=1: t\neq i}^l s_t^2,
    \end{align}
    where $C'$ satisfies the following equation:
    $$ C' = \left\{
        \begin{aligned}
            &\frac{(\kappa_0^{3l} - 1)(\kappa_0^{3(l-1)/2} - 1)^2}{(\kappa_0 - 1)^3}, & \kappa_0 \neq 1, \\
            &\frac{4}{9}l^3, & \kappa_0 = 1.
        \end{aligned}
    \right.
    $$
    To be specific, when $\kappa_0 = 1$, $(3(l-i)-1)i^2\leq \frac{4}{9}l^3$ . If $\kappa_0\neq 1$ and $i \geq 2$, we have
    \begin{align*}
        &\Bigbrace{\kappa_0^{3(l-i)+1}\frac{\kappa_0^{3(l-i)} - 1}{\kappa_0 - 1}}\bigbrace{\kappa_0^{3(i-1)} + \sum_{j=0}^{i-2} \kappa_0^{3j+1}}^2 \le \kappa_0^{3(l-i)+3}\frac{\kappa_0^{3(l-i)} - 1}{\kappa_0 - 1}\frac{(\kappa_0^{3(i-1)} - 1)^2}{(\kappa_0 - 1)^2}\\
        &= \frac{\kappa_0^{3l} - \kappa_0^{3(l-i+1)}}{(\kappa_0 - 1)^3}\Bigbrace{(\kappa_0^{3(l-i)} - 1)(\kappa_0^{3(i-1)} - 1)} 
        \le \frac{(\kappa_0^{3l} - 1)(\kappa_0^{3(l-1)/2} - 1)^2}{(\kappa_0 - 1)^3}.
    \end{align*}
    If $\kappa_0\neq 1$ and $i = 1$, we obtain 
    \begin{align*}
        \Bigbrace{\kappa_0^{3(l-i)+1}\frac{\kappa_0^{3(l-i)} - 1}{\kappa_0 - 1}}\Bigbrace{\kappa_0^{3(i-1)} + \sum_{j=0}^{i-2} \kappa_0^{3j+1}}^2 = \kappa_0^{3l - 2} \frac{\kappa_0^{3(l - 1)} - 1}{\kappa_0 - 1} \le \frac{(\kappa_0^{3l} - 1)(\kappa_0^{3(l-1)/2} - 1)^2}{(\kappa_0 - 1)^3}.
    \end{align*}
    
    The above works for the layers from the beginning until layer $l-1$.
    Last, we consider the trace of $\bH^{(l)}$ with respect to $W^{(l)}$ (notice that $\cU$ is not needed in the readout layer).
    Similar to equation \eqref{eq_lemma_loss}, one can prove that the trace of the Hessian with respect to $W^{(l)}$ satisfies:
    \begin{align*}
        \bigabs{\tr\big[\bH^{(l)}[\ell(H^{(l)}, y)]\big]}
        \leq\,& \kappa_0 \sqrt{k} \sum_{p=1}^{d_{l-1}}\sum_{q=1}^{d_l} \bignorm{\frac{\partial^2 H^{(l)}}{\partial \big(W^{(l)}_{p,q}\big)^2}} + \kappa_1 k \sum_{p=1}^{d_{l-1}}\sum_{q=1}^{d_l} \bignorm{\frac{\partial H^{(l)}}{\partial W^{(l)}_{p,q}}}^2 \\
        \leq\,& \kappa_0 \sqrt{k} \sum_{p=1}^{d_{l-1}}\sum_{q=1}^{d_l} \bignorm{\frac{1}{n} \bm{1}_n^{\top} H^{(l-1)}\frac{\partial^2 W^{(l)}}{\partial \big(W^{(l)}_{p,q}\big)^2}} + \kappa_1 k \sum_{p=1}^{d_{l-1}}\sum_{q=1}^{d_l} \bignorm{\frac{1}{n} \bm{1}_n^{\top} H^{(l-1)}\frac{\partial W^{(l)}}{\partial W^{(l)}_{p,q}}}^2 \\
        \leq\,& \kappa_1 k \sum_{p=1}^{d_{l-1}}\sum_{q=1}^{d_l} \bignorm{\frac{1}{n}\bm{1}_n}^2 \bignorms{H^{(l-1)}\frac{\partial W^{(l)}}{\partial W^{(l)}_{p,q}}}^2 \\
        =\,& \kappa_1 \frac{k}{n} {d_l \bignormFro{H^{(l-1)}}^2}
    \end{align*}
    By equation \eqref{eq_mpgnn_pr1}, the above is bounded by
    \begin{align*}
        & \kappa_1  \frac{k}{n} d_l \bigbrace{\kappa_0^{3(l-1)} + \sum_{j=0}^{l-2} \kappa_0^{3j+1}}^2 \max_{(X, G, y) \sim \cD}\bignormFro{X}^2 \max \bigbrace{1,\bignorms{P_{_G}}^{2(l-1)}} \prod_{j=1}^{l-1} s_j^2 \\
        \leq~ & C_l \max_{(X, G, y) \sim \cD}\bignorms{X}^2 \kappa_1 d_l k \max\Bigbrace{1, \bignorms{P_{_G}}^{2(l-1)}} \prod_{t=1: t\neq l}^l s_t^2,
    \end{align*}%
    since $\frac{\bignormFro{X}^2}{n} \le \bignorms{X}^2$,
    where $C_l$ satisfies the following equation:
    $$ C_l = \left\{
        \begin{aligned}
            &\kappa_0^2\frac{(\kappa_0^{3(l-1)} - 1)^2}{(\kappa_0 - 1)^2}, & \kappa_0 \neq 1, \\
            &l^2, & \kappa_0 = 1.
        \end{aligned}
    \right.
    $$
    Finally, let
    \begin{align}
        \tilde C &= \max(C',C_l). \label{eq_const} 
    \end{align}
    From the value of $C'$ above and the value of $C_l$, we get that $\tilde C$ is equal to
    \begin{align}
        \tilde C &= \left\{
        \begin{aligned}
            &\frac{(\kappa_0^{3l} - 1)(\kappa_0^{3(l-1)/2} - 1)^2}{(\kappa_0 - 1)^3}, & \kappa_0 \neq 1, \nonumber \\
            &\frac{1}{2}l^3, & \kappa_0 = 1. \nonumber
        \end{aligned}\right.
    \end{align}
    Similarly by applying equations \eqref{eq_deri_U} and \eqref{eq_second_U} into Lemma \ref{lemma_trace_hess}, the trace of $\bH^{(l)}$ with respect to $U^{(i)}$ is also less than equation \eqref{eq_mpgnn_pr3}.
    Therefore, we have completed the proof for message-passing neural networks. 
\end{proof}

\subsection{Proof of matching lower bound (Theorem \ref{prop_lb})}\label{app_lb}

For simplicity, we will exhibit the instance for a graph ConvNet, that is, we ignore the parameters in $\cU$ and also set the mapping $\rho_t$ and $\psi_t$ as the identity mapping.
Further, we set the mapping $\phi_t(x) = x$ as the identity mapping, too, for simplifying the proof.
In the proof, we show that for an arbitrary configuration of weight matrices $W^{(1)}, W^{(2)}, \dots, W^{(l)}$, there exists a data distribution such that for this particular configuration, the generation gap with respect to the data distribution satisfies the desired equation \eqref{eq_lb}.

\begin{proof}[Proof of Theorem \ref{prop_lb}]
    Recall that the underlying graph for the lower bound instance is a complete graph.
    Next, we will specify the other parts of the data distribution $\cD$.
    Let $Z = \prod_{i=1}^l W^{(i)}$ denote the product of the weight matrices.
    We are going to construct a binary classification problem.
    Thus, the dimension of $Z$ will be equal to $n$ by $2$.
    Let $Z = U D V^{\top}$ be the singular value decomposition of $Z$.
    Let $\lambda_{\max}(Z)$ be the largest singular value of $Z$, with corresponding left and right singular vectors $u_1$ and $v_1$, respectively.
    Within the hypothesis set $\cH$, $\lambda_{\max}(Z)$ can be as large as $\prod_{i=1}^l s_i$.
    Denote a random draw from $\cD$ as $X, G, y$, corresponding to node features, the graph, and the label:
    \begin{enumerate}[leftmargin=15pt]
        \item The feature matrix $X$ is is equal to $\bm{1}_n u_1^{\top}$;
        \item The class label $y$ is drawn uniformly between $+1$ and $-1$;
        \item Lastly, the diffusion matrix $P$ is the adjacency matrix of $G$, which has a value of one in every entry of $P$.
    \end{enumerate}
    Given the example and the weight matrices, we will use the logistic loss to evaluate $f$'s loss.
    Notice that $P = \bm{1}_n \bm{1}_n^{\top}$.
    Thus, one can verify  $\lambda_{\max}(P) = n$.
    Crucially, the network output of our GCN is equal to
    \begin{align*}
        H^{(l)} = \frac 1 n \bm{1}_n^{\top} P^{l-1} X W^{(1)} W^{(2)} \cdots W^{(l)}
        = n^{l-1} \Bigbrace{\frac{\bm{1}_n^{\top} X}{n} Z}
        = n^{l-1} \Bigbrace{u_1^{\top} U D V^{\top}}
        = \bigbrace{n^{l-1} \lambda_{\max}(Z)} v_1^{\top}.
    \end{align*}
    Let us denote $\alpha = n^{l-1} \lambda_{\max}(Z)$---the spectral norms of the diffusion matrix and the layer weight matrices.
    Let $v_{1,1}, v_{1,2}$ be the first and second coordinate of $v_1$, respectively.
    Notice that $y$ is drawn uniformly between $+1$ or $-1$.
    Thus, with probability $1/2$, the loss of this example is $\log(1 + \exp(- \alpha \cdot v_{1,1}))$;
    with probability $1/2$, the loss of this example is $\log(1 + \exp(\alpha\cdot v_{1,2}))$.
    Let $b_i$ be a random variable that indicates the logistic loss of the $i$-th example.
    The generalization gap is equal to
    \begin{align*}
        \epsilon = {\frac 1 N \sum_{i=1}^N b_i} - \frac 1 2\Bigbrace{\log(1 + \exp(- \alpha \cdot v_{1,1})) + \log(1 + \exp(\alpha \cdot v_{1,2}))}.
    \end{align*}
    By the central limit theorem, as $N$ grows to infinity, the generalization gap $\epsilon$ converges to a normal random variable whose mean is zero and variance is equal to
    \begin{align*}
        \frac 1 {4N} \Bigbrace{\log(1 + exp(-\alpha \cdot v_{1,1})) - \log(1 + \exp(\alpha \cdot v_{1,2}))}^2 \gtrsim \frac{\alpha^2}{N},
    \end{align*}
    for large enough values of $N$.
    As a result, with probability at least $0.1$, when $N$ is large enough, the generalization gap $\epsilon$ must be at least 
    \[ \bigo{\sqrt{\frac{\alpha^2}  N}}, \text{ where } \alpha = \bignorms{P_{_G}}^{l-1} \lambda_{\max}\Bigg(\prod_{i=1}^l W^{(i)}\Bigg). \]
    Notice that the product matrix's spectral norm is at most $\prod_{i=1}^l s_i$.
    Thus, we have completed the proof of equation \eqref{eq_lb}.
\end{proof}

\subsection{Proof for graph isomorphism networks (Corollary \ref{thm_gin})}\label{app_gin}

To be precise, we state the loss function for learning graph isomorphism networks as the averaged loss over all the classification layers:
\begin{align}\label{eq_loss_gin}
    \bar{\ell}(f(X, G), y) = \frac 1 {(l-1)} \sum_{i=1}^{l-1}  \ell\Bigbrace{\frac 1 n \bm{1}_n^{\top} H^{(i)} V^{(i)}, y}.
\end{align}
Thus, $\hat \cL_{GIN}(f)$ is equivalent to the empirical average of $\bar\ell$ over $N$ samples from $\cD$.
$\cL_{GIN}(f)$ is then equivalent to the expectation of $\bar\ell$ over a random sample from $\cD$.

\begin{proof}[Proof of Corollary \ref{thm_gin}]
    This result follows the trace guarantee from Lemma \ref{lemma_trace_hess}.
    For any $i=1,\dots,l-1$ and any $j = i,\dots,l-1$, we can derive the following result with similar arguments:
    \begin{align*}
        \bigabs{\tr\Big[\bH_{\cW}^{(i)}\big[\ell\Bigbrace{\frac 1 n \bm{1}_n^{\top} H^{(j)} V^{(j)}, y}\big]\Big]} 
        \leq \frac{\kappa_0 \sqrt{k}}{\sqrt{n}} \bignorms{V^{(j)}} \sum_{p=1}^{d_{i-1}}\sum_{q=1}^{d_i} \bignormFro{\frac{\partial^2 H^{(j)}}{\partial \big(W^{(i)}_{p,q}\big)^2}} + \frac {\kappa_1  k} n \bignorms{V^{(j)}}^2 \bignormFro{\frac{\partial H^{(j)}}{\partial W^{(i)}}}^2.
    \end{align*}
    Next, we repeat the steps in Propositions \ref{prop_first_mpgnn} and \ref{prop_second_mpgnn}, for any $i = 1,\dots,l-1$ and any $j = i,\dots,l-1$:
    \begin{align*}
        & \max_{(X, G, y) \sim \cD} \bigabs{\tr\big[\bH^{(i)}[\ell\bigbrace{\frac 1 n \bm{1}_n^{\top} H^{(j)} V^{(j)}, y}]\big]} \\
        \le&  2\kappa_1\tilde{C} d_i k\max_{(X, G, y) \sim \cD}\bignorms{X}^2 \max \Bigbrace{1, \bignorms{P_{_G}}^{2(j-i+1)}} \bignorms{V^{(j)}}^2 \prod_{t = 1:\, t \neq i}^j s_t^2. %
    \end{align*}
    Based on the above step, the trace of the loss Hessian matrix with respect to $W^{(i)}, U^{(i)}$ satisfies:
    {\begin{align*}
        & \max_{(X, G, y) \sim \cD} \bigabs{\bigtr{\bH^{(i)}\bigbracket{\bar{\ell}(f(X, G), y)}}} \\
        =& \max_{(X, G, y) \sim \cD} \bigabs{\bigtr{\bH^{(i)}\bigbracket{\frac{1}{(l-1)} \sum_{j=1}^{l-1} \ell\bigbrace{\frac 1 n \bm{1}_n^{\top} H^{(j)} V^{(j)}, y}}}} \\
        =& \frac{1}{l-1} \sum_{j=i}^{l-1} \max_{(X, G, y) \sim \cD} \bigabs{\bigtr{\bH^{(i)}\bigbracket{\ell\Bigbrace{\frac 1 n \bm{1}_n^{\top} H^{(j)} V^{(j)}, y}}}} \\
        \le & 2 \kappa_1 \tilde{C} d_i k \max_{j=1}^{l-1}\bignorms{V^{(j)}}^2\Bigbrace{\max_{(X, G, y) \sim \cD}\bignorms{X}^2  \sum_{j=1}^{l-1} \frac{\max \bigbrace{1, \bignorms{P_{_G}}^{2j}}}{{l-1}}} \prod_{t = 1:\, t \neq i}^{l-1} s_t^2.
    \end{align*}}%
    Within the above step, the propagation matrix satisfies:
    \[ \frac 1 {(l-1)} \sum_{j=1}^{l-1} \max\bigbrace{1, \bignorms{P_{_G}}^{2j}} \le \max\fullbrace{1, \bignorms{\frac 1 {l-1}\sum_{j=1}^{l-1} P_{_G}^j }^2}. \]
    Notice that $P_{GIN} = \frac 1 {l-1} \sum_{j=1}^{l-1} P_{_G}^j$.
    Thus, we have completed the generalization error analysis for graph isomorphism networks in equation \eqref{eq_gin_result}.
\end{proof}
\section{Experiment Details}\label{sec_add_exp_setup}

For our result, we measure $B$ as an upper bound on the loss value taken over the entire data distribution.
Across five datasets in our experiments, setting $B = 5.4$ suffices for all the training and testing examples in the datasets.

For comparing generalization bounds, we use two types of model architectures, including GCN \cite{kipf2016semi} and the MPGNN in \citet{liao2020pac}. Following the setup in \citet{liao2020pac}, we apply the same network weights across multiple layers in one model, i.e., $W^{(t)} = W$ and $U^{(t)} = U$ across the first $l-1$ layers. For GCNs, we set $\cU$ as zero, $\rho_t$ and $\psi_t$ as identity mappings, $\phi_t$ as ReLU function. For MPGNNs, we specify $\phi_t$ as ReLU, $\rho_t$ and $\psi_t$ as Tanh function. For both model architectures, we set the width of each layer $d_t = 128$ and vary the network depth $l$ in 2, 4, and 6.
On the three collaboration network datasets, we use one-hot encodings of node degrees as input node features. 
We train the models with Adam optimizer with a learning rate of $0.01$ and set the number of epochs as $50$ and batch size as $128$ on all three datasets.
We compute the generalization bounds following the setup in \citet{liao2020pac}. %
Theorem 3.4 from \citet{liao2020pac}: 
    {\small$$
        \sqrt{\frac{42^2}{\gamma^2 N}
        \Bigg( \max\limits_{(X, G, y) \sim \cD}\bignorms{X}^2 \Bigg)
        \Big(\max \Big(\zeta^{-l+1}, (\lambda\xi)^{\frac{l+1}{l}} \Big) \Big)^2 l^2h\log(4lh) (2s_1^2r_1^2 + s_l^2r_l^2))},
    $$}%
    where $\zeta = \min(s_1, s_l)$, $\lambda = s_1s_l$, $\xi = \frac{(ds_1)^{l-1} - 1}{ds_1-1}$, $d$ is the max degree, $h$ is the max hidden width, and $\gamma$ is the desired margin in the margin loss. Note that $s_i = s_1$ and $r_{i}=r_1$ for $1 \leq i \leq l-1$ since the first $l-1$ layers apply the same weight.  
Proposition 7 from \citet{garg2020generalization}: 
    {\small$$
        48 s_l h Z\sqrt{\frac{3}{\gamma^2 N}\log\big( 24s_l\sqrt{N}
        \max\big( Z, M\sqrt{h}
        \max\big(
        \kappa_2 s_1, \bar{R}s_1 \big) \big) \big)},
    $$}%
    where $M{=}\frac{(ds_1)^{l-1} - 1}{ds_1-1}$, $\bar{R}=d\cdot \min( \kappa_1 \sqrt{h}, \kappa_2 s_1 M)$, $Z = \kappa_2 s_1 + \bar{R}s_1$, $\kappa_1 = \max\limits_{x \in \real^{h}}\norm{\phi(x)}_{\infty}$, and  $\kappa_2 = \max\limits_{_{_{(X, G, y) \sim \cD}}}\bignorms{X}^2$.

\end{document}